\documentclass[runningheads]{llncs}

\usepackage{eccv}  
\usepackage{eccvabbrv}
\usepackage{graphicx}
\usepackage{booktabs}
\usepackage[accsupp]{axessibility}  
\usepackage{hyperref}  
\usepackage{orcidlink}
\providecommand{\citep}[1]{\cite{#1}}
\providecommand{\citet}[1]{\cite{#1}}

\usepackage{tikz}
\usetikzlibrary{shapes.geometric,arrows.meta,pgfplots.groupplots}
\usepackage{pgfplots}
\pgfplotsset{compat=newest}
\usepackage{pgfplotstable}
\usetikzlibrary{patterns}

\usepackage{dirtytalk}
\usepackage{nicefrac} 
\usepackage{algorithm}
\usepackage{algpseudocodex}
\usepackage{multirow}
\usepackage[normalem]{ulem}
\usepackage{pifont}
\usepackage{tabularx}

\usepackage{float}
\usepackage{adjustbox}
\usepackage{physics}
\usepackage{colortbl}
\usepackage[most]{tcolorbox}
\newcommand{\best}[1]{\textbf{\textcolor{BrickRed}{#1}}}
\newcommand{\second}[1]{{\textcolor{Black}{#1}}}

\definecolor{lemmacolor}{HTML}{E8F0FE}    
\definecolor{lemmaborder}{HTML}{4A86C8}   
\definecolor{theoremcolor}{HTML}{FDE8E8}  
\definecolor{theoremborder}{HTML}{C84A4A} 
\definecolor{insightcolor}{HTML}{FFF8E1}  
\definecolor{insightborder}{HTML}{F9A825} 

\newtcolorbox{lemmbox}{
  colback=lemmacolor, colframe=lemmaborder,
  boxrule=0.4pt, arc=1.5pt,
  left=5pt, right=5pt, top=2pt, bottom=2pt,
  before skip=4pt, after skip=4pt,
  fonttitle=\bfseries, breakable
}
\newtcolorbox{theorembox}{
  colback=theoremcolor, colframe=theoremborder,
  boxrule=0.4pt, arc=1.5pt,
  left=5pt, right=5pt, top=2pt, bottom=2pt,
  before skip=4pt, after skip=4pt,
  fonttitle=\bfseries, breakable
}
\newtcolorbox{insightbox}{
  colback=insightcolor, colframe=insightborder,
  boxrule=0.5pt, arc=1.5pt,
  left=5pt, right=5pt, top=2pt, bottom=2pt,
  before skip=4pt, after skip=4pt,
  enhanced, borderline west={2pt}{0pt}{insightborder}
}

\AtBeginDocument{%
  \crefname{figure}{Fig.}{Figs.}%
  \crefname{table}{Tab.}{Tabs.}%
  \crefname{section}{Sec.}{Secs.}%
  \crefname{equation}{Eq.}{Eqs.}%
}

\usepackage{mathtools}
\usepackage{mathrsfs}
\usepackage{bm}       
\usepackage{amsfonts}  

\def\1{\bm{1}}

\DeclareMathAlphabet{\mathsfit}{\encodingdefault}{\sfdefault}{m}{sl}
\SetMathAlphabet{\mathsfit}{bold}{\encodingdefault}{\sfdefault}{bx}{n}

\def\sR{{\mathbb{R}}}

\begin{document}
\title{Stable Forgetting: Bounded Parameter-Efficient Unlearning in Foundation Models}
\titlerunning{Stable Forgetting}
\author{
Arpit Garg\thanks{Equal contribution.}\orcidlink{0000-0001-5886-5822}
\and
Hemanth Saratchandran\textsuperscript{$\star$}
\and
Ravi Garg
\and
Simon Lucey
}
\authorrunning{A.~Garg et al.}

\institute{
Australian Institute for Machine Learning (AIML),\\
Adelaide University, Adelaide, Australia\\
\email{arpit.garg@adelaide.edu.au}
}

\maketitle
\begin{abstract}
Machine unlearning in foundation models (e.g., language and vision transformers) is essential for privacy and safety; however, existing approaches are unstable and unreliable. A widely used strategy, the gradient difference method, applies gradient descent to retained data while performing gradient ascent on forgotten data. When combined with cross-entropy, this procedure can trigger the unbounded growth of weights and gradients, degrading both forgetting and retention. We provide a theoretical framework that explains this failure by showing how ascent destabilizes optimization in transformer feedforward MLP layers. Guided by this insight, we propose \textit{Bounded Parameter-Efficient Unlearning}, which stabilizes LoRA-based fine-tuning by applying bounded functions to MLP adapters. This controls the weight dynamics during ascent and enables reliable convergence. We validate the approach on Vision Transformer class deletion on CIFAR-100, where GD+Sine is the only evaluated method to achieve both high forget quality and model utility across ViT-B/16, ViT-L/14, and DeiT-S architectures, and demonstrate generality on language-model benchmarks (TOFU, TDEC, MUSE) across architectures from 22M to 8B parameters, achieving improved forgetting while preserving utility.\footnote{Code will be open-sourced upon acceptance.}

\end{abstract}

\section{Introduction}

The advent of foundation models has profoundly reshaped machine learning; however, their large-scale deployment has revealed critical vulnerabilities in safety and data governance~\cite{GDPR_EU_Art17_2026}. During pretraining, these models absorb massive datasets that frequently contain sensitive, copyrighted, or personally identifiable information~\citep{Shi2024}, making the ability to selectively \emph{forget} such information a regulatory requirement and a technical necessity~\citep{bourtoule2021machine,cao2015towards}. Machine unlearning, which involves removing the influence of specific data without full retraining, is one of the most urgent challenges in the ethical deployment of large-scale models.

Current approaches face fundamental limitations in this regard. While the broader literature spans gradient-based, parameter-efficient, preference-based, and representation-editing families (see~\cref{sec:related_work}), the two most widely adopted strategies are as follows. The \uline{first} is \textit{full fine-tuning}: the standard procedure applies gradient ascent on a forget set with cross-entropy loss~\citep{Maini2024}, leading to instability in training and degradation in retention quality. The \emph{gradient difference} (GD) method~\citep{Maini2024,Cha2025} addresses this by jointly applying gradient descent on a retention set and gradient ascent on the forget set; however, this combination remains unstable under cross-entropy. The \uline{second} is \textit{parameter-efficient fine-tuning}, such as LoRA~\citep{Hu2022}, which reduces computational and memory costs but continues to suffer the same instability under gradient difference and cross-entropy across transformer architectures. These limitations have motivated incremental refinements: Fisher information-weighted initialization (FILA)~\citep{kim2025improving} and Inverted Hinge Loss (IHL)~\citep{Cha2025} improve stability through carefully designed objectives and initialization strategies; however, their combination~\citep{Cha2025} provides only partial relief. More recent adversarial unlearning frameworks~\citep{setlur2022adversarial} and residual feature alignment methods~\citep{li2024fast} are fundamentally constrained by their linear parameterization.

\begin{figure}[t]
    \centering
\includegraphics[width=\textwidth]{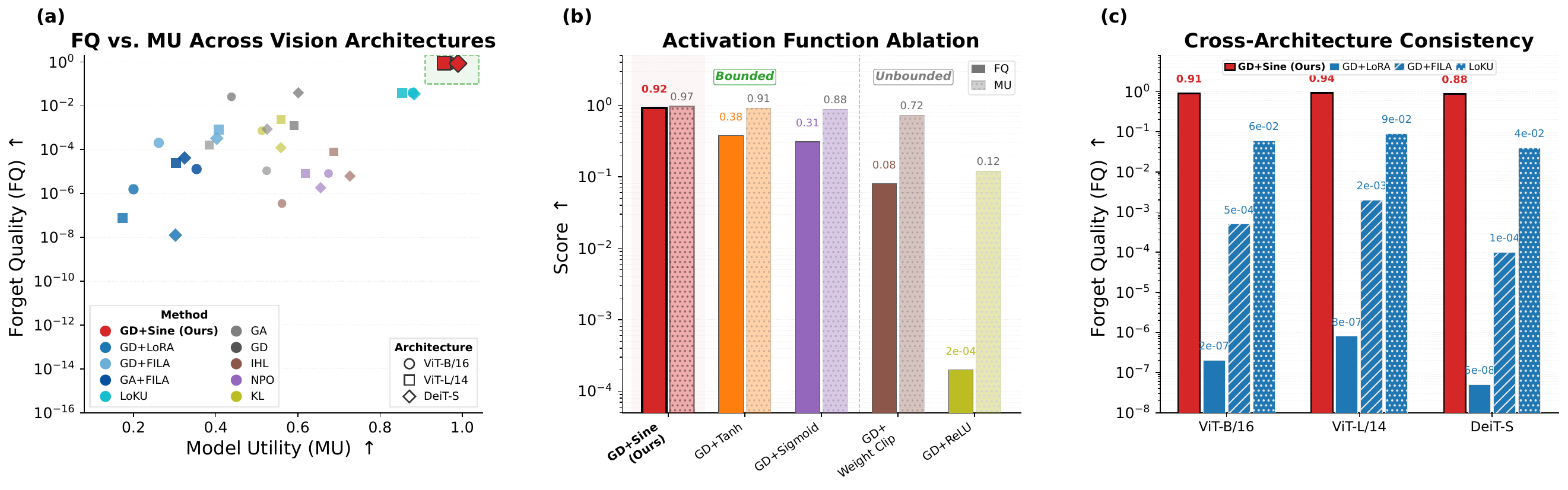}
\caption{\textbf{GD+Sine is the only method to achieve both high Forget Quality and Model Utility across vision architectures.}    \textbf{(\textcolor{BrickRed}{a})}~FQ vs.\ MU on ViT-B/16, ViT-L/14, and DeiT-S: only GD+Sine (\textcolor{red}{red}) reaches the ideal zone (green box), while parameter-efficient and full fine-tuning baselines fail on one or both axes. \textbf{(\textcolor{BrickRed}{b})}~Activation ablation on ViT-B/16: Sine alone achieves near-perfect FQ (0.92) and MU (0.97); unbounded activations collapse on both metrics. \textbf{(\textcolor{BrickRed}{c})}~GD+Sine dominates consistently across all three architectures by orders of magnitude in forget quality over parameter-efficient baselines maintaining model utility.}
\label{fig:hero}
\end{figure}

We provide a theoretical framework for analyzing the training instability of the gradient difference method under cross-entropy. Our analysis shows that the ascent step causes the weights and gradients in the transformer feedforward MLP (FFN) layers to grow excessively, establishing the root cause. Parameterizing FFN weights with a \emph{bounded function} provides a principled mechanism for stabilizing the optimization under gradient ascent. Building on this, we extended the gradient difference framework with LoRA-based fine-tuning, demonstrating that bounded parameterization directly stabilizes weights and gradients during unlearning. We propose \textbf{bounded parameter-efficient unlearning}, which applies bounded functions to LoRA adapters in FFN layers, enabling stable fine-tuning with a cross-entropy forgetting objective. This directly overcomes the key limitations of previous methods, providing theoretical guarantees and effectiveness.

\begin{figure}[t]
    \centering
\includegraphics[width=\textwidth]{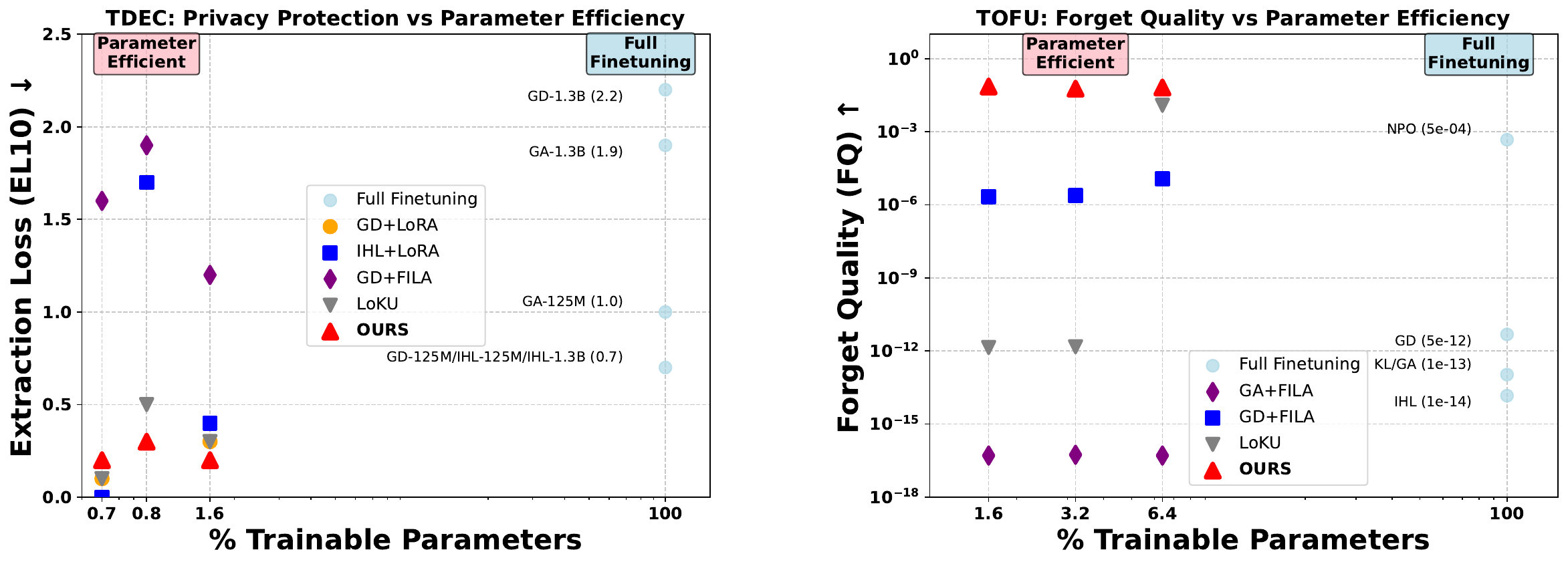}
    \caption{\textbf{Balancing efficiency and effectiveness in parameter tuning.} 
\textbf{(\textcolor{BrickRed}{Left})}
On TDEC, our method achieves stronger privacy protection than existing parameter-efficient baselines while requiring fewer parameters than full-tuning. 
\textbf{(\textcolor{BrickRed}{Right})} On TOFU, our approach maintains consistently high forget quality across LoRA ranks, outperforming state-of-the-art baselines by orders of magnitude while preserving parameter efficiency.}
\label{fig:intro_performance}
\end{figure}

We evaluated our framework on standard techniques in the field, including ViT class deletion on CIFAR-100 (\cref{fig:hero}), and TOFU, TDEC, and MUSE benchmarks (\cref{fig:intro_performance}), spanning ViT, DeiT, GPT-Neo, Phi, and LLaMA architectures (22M-8B parameters). Our method achieves strong unlearning performance with significant improvements in forget quality over existing methods, while preserving model utility. Our main contributions are as follows.
\begin{enumerate}
    \item We develop a theoretical framework for the gradient difference method with cross-entropy loss, showing instability arises because the ascent step drives uncontrolled growth of weights and gradients in MLP feedforward layers. We derived the key insight that parameterizing the feedforward weights with a bounded function stabilizes the gradient ascent.

    \item Building on this principle, we propose \textbf{bounded parameter-efficient unlearning}, a parameter-efficient method that applies bounded functions to LoRA adapters in feedforward layers. Our approach achieves stable unlearning, delivering substantial improvements in forget quality while preserving the utility across benchmarks.
\end{enumerate}
\section{Related Work}\label{sec:related_work}

\paragraph{Machine unlearning.}
Machine unlearning removes the influence of specific data without full retraining~\citep{bourtoule2021machine,cao2015towards}. For foundation models, retraining is infeasible at scale~\citep{Nguyen2022,Ye2023}, motivating the development of efficient alternatives.
Recent methods can be classified into four categories:
(1)~\textit{ Full fine-tuning} applies gradient-based forgetting objectives~\citep{Yao2024,huang2024unified,Maini2024}.
(2)~\textit{ Parameter-efficient fine-tuning} using LoRA-style adapters~\citep{Hu2022,Cha2025,kim2025improving}
(3)~\textit{ Preference-based methods} that leverage alignment signals~\citep{Rafailov2023,Zhang2024c};
and (4) \textit{ representation/weight-editing methods} that directly alter internal activations or weights~\citep{Ilharco2022,Meng2023ROME}.
Vision approaches target class/concept removal via saliency masking~\citep{fan2024salun}, contrastive mechanisms, and parameter-efficient parameterizations~\citep{li2024fast,ren2025lethevit,yu2024retain,roy2025novo,khalil2025coun}.
We focus on (1) and (2) as they are the most relevant; our method belongs to category (2). An in-depth survey of each category is presented in \cref{app:extended_relatedWork}.

\paragraph{Full fine-tuning methods.}
Applying cross-entropy gradient ascent to the forget set~\citep{Maini2024} routinely destabilizes training and degrades retention.
The \emph{gradient difference} (GD) method~\citep{Maini2024,Cha2025} adds a gradient descent retention objective to counteract this, yet remains unstable under cross-entropy ascent.
FILA~\citep{kim2025improving} and Inverted Hinge Loss (IHL)~\citep{Cha2025} partially mitigate instability, but neither identifies or resolves the \emph{root cause}: unbounded weight growth.
Full fine-tuning is also computationally expensive because it updates all model weights.

\paragraph{Parameter-efficient fine-tuning.}
LoRA~\citep{Hu2022} factorizes weight updates into low-rank matrices, reducing the overhead from $dk$ to $(d{+}k)r$.
Combined with the gradient difference, LoRA-based unlearning still suffers from severe cross-entropy instability~\citep{huang2024unified,Cha2025}.
Recent efforts have addressed this issue via influence reweighting (GUARD~\citep{niu2025guard}, RapidUn~\citep{zhao2025rapidun}), gradient reconstruction (R2F~\citep{liu2025r2f}), and bootstrapping (LUNE~\citep{liu2025lune}, BB~\cite{li2025llm}), targeting symptoms rather than root causes.
As shown in \cref{sec:method}, the fundamental issue is \emph{unbounded gradient ascent dynamics in MLP feedforward layers}; constraining adapter weights solves this and complements the above advances. 
\section{Methodology}\label{sec:method}

\subsection{Preliminaries}\label{subsec:notation}

\paragraph{Problem Formulation.}
Machine unlearning aims to remove the influence of forget data $\mathcal{D}_f$ while preserving performance on retain data $\mathcal{D}_r$. Given a model $f_\theta$ with parameters $\theta$, the objective combines retention and forgetting:
\begin{equation}\label{eqn:rtain_forget_loss}
\mathcal{L}_r(\theta) + \lambda \mathcal{L}_f(\theta)
\end{equation}
where $\mathcal{L}_r(\theta) = \mathbb{E}_{(x,y) \sim \mathcal{D}_r} [\mathcal{L}(f_\theta(x), y)]$, $\mathcal{L}_f(\theta) = \mathbb{E}_{(x,y) \sim \mathcal{D}_f} [\mathcal{L}(f_\theta(x), y)]$, and $\lambda > 0$ controls the forgetting strength. The optimization proceeds by simultaneously training $\mathcal{L}_r(\theta)$ via gradient descent and $\mathcal{L}_f(\theta)$ via gradient ascent, respectively.

\paragraph{Gradient-Based Unlearning.}
The gradient difference method optimizes the unlearning objective through:
\begin{equation}\label{eqn:gradient_difference}
\theta_{t+1} = \theta_t - \alpha_r \nabla_\theta \mathcal{L}_r(\theta) + \alpha_f \nabla_\theta \mathcal{L}_f(\theta)
\end{equation}

This combines gradient descent on retain data with gradient ascent on forget data. While gradient ascent effectively increases loss on $\mathcal{D}_f$, it suffers from optimization instability when combined with cross-entropy loss \citep{Cha2025}. In \cite{pan2024multi}, this issue was addressed by replacing the cross-entropy loss on the forget set with an Inverted Hinge Loss. In contrast, as shown in \cref{subsec:our_approach}, our methodology enables direct training using cross-entropy.

\paragraph{Low-Rank Adaptation}

LoRA parameterizes weight updates through low-rank decomposition:
\begin{equation}\label{eqn:lora_update}
W = W_0 + AB^T
\end{equation}
where $W_0$ is the pretrained weight, and $A \in \sR^{d \times r}$ and $B \in \sR^{k \times r}$ are trainable matrices with rank $r \ll \min(d,k)$. For transformer architectures (e.g., ViTs and LLMs) comprising attention and MLP feedforward (FFN) layers, \cref{eqn:lora_update} is generally applied to the MLP and attention layers. While LoRA reduces computational costs from $dk$ to $(d+k)r$ parameters, its root issues in unlearning remain underexplored, with existing solutions addressing symptoms rather than the fundamental instabilities that arise in the gradient-based unlearning.

\subsection{Theoretical Analysis}\label{subsec:theory}

In \cref{eqn:gradient_difference}, the gradient difference method combines two objectives: a forget loss optimized via gradient ascent and a retain loss optimized via gradient descent. Prior work \cite{Cha2025} has shown that when the forget loss employs cross-entropy, fine-tuning becomes unstable. To understand this phenomenon, we analyze gradient ascent under the cross-entropy loss and establish two theorems showing that weights and gradients can diverge. This theoretical insight motivates our approach in \cref{subsec:our_approach}, where we propose a method to mitigate such divergence and stabilize the training.

The networks we consider will all be trained with the cross-entropy loss as the retain and forget loss in \cref{eqn:rtain_forget_loss}. In this section, we work generally and simply denote the cross-entropy loss associated to an MLP by $\mathcal{L}$.
We let $C$ denote the number of distinct class labels so that the output dimension of the network will be $C$. 
The output probabilities of the network will be denoted $p$ and we recall for a class label denoted by $y$ the cross-entropy loss of the predicted $p \in \sR^C$ compared to $y$ is given by
\begin{equation}\label{eqn:cross_entropy_loss_y}
    \mathcal{L}(p, y) = -\log(p_y)
\end{equation}
where $p_y$ is the yth-component of $p \in \sR^{C}$. Using \cref{eqn:cross_entropy_loss_y} and the chain rule we have that the gradient vector of $\mathcal{L}$ with respect to the logits $z$ on a class $y$ is given by
\begin{equation}
    \nabla_{z}\mathcal{L} = p - e_y
\end{equation}
where $e_y$ denotes the one-hot vector that is $1$ in the $y^{\text{th}}$-position. For more details on the cross-entropy loss, we refer the reader to \cite{prince2023understanding}.

When training under gradient ascent the optimizer wants to push the predictions $p$ away from $e_y$ as it seeks to move towards a maximum of the cross-entropy loss. We thus get that 
\begin{equation}\label{eqn:gradi_ascent_one_hot}
\nabla_{z}\mathcal{L} \rightarrow e_{j} - e_y    
\end{equation}
where $j$ is an index $1 \leq j \leq C$ such that $j \neq y$ so that the one hot vectors associated to the classes $y$ and $j$ are distinct.
In particular, under a gradient ascent trajectory that is approaching a maximum the logit gradient $\nabla_{z}\mathcal{L}$ does not approach zero and hence 
\begin{equation}\label{eqn:logits_deriv}
    \|\nabla_{z}\mathcal{L}\| > C > 0
\end{equation}
stays bounded away from zero for some constant $C > 0$, where 
$\|\cdot\|$ denotes the Euclidean norm (see \cref{app:notation} for details on notation).
We note that in the case of gradient descent the term on the right of \cref{eqn:logits_deriv} approaches zero yielding a completely different behavior to gradient ascent. 

\begin{lemma}\label{lem:logits_infinity}
Let $\mathcal{L}$ denote the cross-entropy loss trained on a MLP $F$ with $L$ layers under gradient ascent. Let $z(t)$ denote the logits at iteration $t$. Then if
$\mathcal{L}(t) \rightarrow \infty$ it follows that
$z(t) \rightarrow \infty$ in norm.
\end{lemma}

The proof of \cref{lem:logits_infinity} is given in \cref{app:proofs}. 
The above lemma shows that when training with gradient ascent if the cross-entropy loss approaches a global maximum the logits get large. The following theorem shows that this can lead to large weights or gradients in the final layer.
For the theorem we will need the notation of activation outputs. Given an $L$ layer MLP denoted $F$, we let $a_{l}$ for $1 \leq l \leq L$ denote the output of layer $l$. For details on the notation we use for MLPs we refer the reader to \cref{app:notation}. 

\begin{theorem}\label{thm:weights_gradients_explode}
Let $F$ be a $L$-layer MLP. Suppose under gradient ascent with iterations $t$, the logits
$z(t) \rightarrow \infty$. Then if the activation output
$\|a_{L-1}(t)\| \leq C_1$ for large $t$, where $C_1 > 0$ is a constant, it follows
\begin{equation}
    \|W_L(t)\| \rightarrow \infty.
\end{equation}
In the case there is no such bound on $\|a_{L-1}(t)\|$ it follows that there exists a subsequence of iterations $t_k$ such that
\begin{equation}
    \|\nabla_{W_L}\mathcal{L}(t_k)\| \rightarrow \infty.
\end{equation}
\end{theorem}

\textbf{Key insight.} Gradient ascent under cross-entropy drives feedforward weights and gradients to grow without bound. This motivates constraining the adapter weights with a bounded function to stabilize gradient-difference unlearning.

The proof of \cref{thm:weights_gradients_explode} is provided in \cref{app:proofs}.
We note that the bounded activation assumption ($\|a_{L-1}(t)\| \leq C_1$) is mild in transformer architectures, where layer normalization constrains the variance of intermediate representations; a detailed justification is given in~\cref{app:proofs}.
In practice, training is limited to a finite number of iterations, and models rarely reach a regime where gradients or parameters fully stabilize. As established in \cref{thm:weights_gradients_explode}, gradient ascent can drive the weights and gradients of the final layer to grow excessively. While such growth may not disrupt training immediately, it can cascade backward through to earlier layers, amplifying both weights and gradients in more than one layer and ultimately producing unstable dynamics. We formalize this propagation effect in \cref{app:further_theory}, where \cref{thm:pre_gradients_explode} shows how instability originating in the final layer extends to preceding layers. Notably, our analysis focuses on pure gradient ascent, to give the reader the main idea of why unlearning is difficult when a gradient ascent term is present. We extend both \cref{lem:logits_infinity} and \cref{thm:weights_gradients_explode} to the case of the gradient difference method in \cref{app:proofs}, see \cref{lem:gd_logit_blowup} and \cref{thm:weights_gradients_explode_gd} in \cref{app:proofs}.
Furthermore, we note that
empirical evidence in \cref{fig:optimization_analysis} shows that under the gradient difference method, weights grow excessively, indicating that the ascent term on the forget set is the primary driver of this instability.

\subsection{Bounded Parameter-Efficient Unlearning}\label{subsec:our_approach}

\cref{thm:weights_gradients_explode} and \cref{thm:pre_gradients_explode} in \cref{app:theory} demonstrate that gradient ascent drives weights and gradients in the feedforward layers of MLPs to grow excessively, which can destabilize training. In \cref{sec:experiments}, we empirically confirm this effect: when fine-tuning with LoRA under the gradient difference framework, weights and gradients grow excessively large, and this growth is the primary reason LoRA fails to perform effective unlearning. To address this issue, we propose a simple yet effective architectural modification that implicitly regularizes the weights of the MLP layers

Specifically, let $\phi : \sR \to \sR$ denote a bounded non-linear function. We redefine the adapter transformation (see \cref{eqn:lora_update}) in the feedforward layers as
\begin{equation}\label{eqn:bounded_low_rank_approach}
h = W_0 x + \frac{\alpha}{r} \phi(\omega A B^\top) x + b    
\end{equation}
where $A$ and $B$ are the low-rank adapter matrices of rank $r$, $\alpha$ is the standard LoRA scaling factor, $\omega$ is the frequency parameter, $x$ is the input data, and $b$ is the bias term. The bounded nonlinearity $\phi$, applied elementwise to the scaled pre-activations, constrains the ascent dynamics and prevents the uncontrolled growth of weights and gradients. 
Crucially, this bound holds at every iterate regardless of how large $A$ or $B$ grow, providing a non-asymptotic stability certificate for the forward computation. Moreover, since the pretrained weight $W_0$ is frozen, only the gradients $\partial\mathcal{L}/\partial A$ and $\partial\mathcal{L}/\partial B$ are relevant for optimization. These are bounded by the chain rule through the bounded Jacobian of $\phi$, so constraining the adapter suffices to prevent gradient explosion in the full backward pass. As we demonstrate in \cref{sec:experiments}, this adjustment yields substantially more stable training and improved performance with finetuning across a variety of unlearning benchmarks.


For the choice of $\phi$, we took $\tanh$ as this is a well-known activation choice in machine learning. More recently, \cite{Ji2025} showed that applying sine mapping, $\sin(\omega AB^\top)$ with frequency $\omega > 0$, produces a high-rank matrix whose rank grows with $\omega$, yielding stronger fine-tuning performance on a range of transformer fine-tuning benchmarks. While \cite{Ji2025} applies sine for fine-tuning, this study identifies and solves the inherent instability of gradient difference for unlearning. As $\sin(\omega \cdot)$ is bounded for any $\omega > 0$, it aligns naturally with our setting to constrain explosive optimization dynamics. Furthermore, because $\sin(\omega \cdot)$ is Lipschitz continuous with constant $\omega$, the gradient magnitude scales proportionally with $\omega$. Consequently, we empirically find that scaling the learning rate (step size) by $\mathcal{O}(1/\omega)$ is sufficient for stable optimization. In our experiments (\cref{sec:experiments}), sine functions, particularly with larger $\omega$, consistently outperformed other monotonic bounded alternatives like $\tanh$ and sigmoid. This improvement occurs because the high effective rank and oscillatory nature of the sine function prevent the optimization from stagnating. In contrast, this assurance does not apply to $\tanh$ or sigmoid functions, as their derivatives diminish quickly as they move away from the origin, resulting in vanishing Jacobian entries and poor conditioning. Accordingly, we focus on both $\tanh$ and sine as choices for $\phi$. In \cref{app:sensitivity} we compare against using a sigmoid function and to demonstrate the necessity of boundedness, we also include an unbounded example, $\mathrm{ReLU}$, for comparison.


We note that many transformer models employ normalization techniques such as layer normalization~\cite{ba2016layer, macdonald2023skip}, batch normalization~\cite{ioffe2015batch} or Jacobian normalization \cite{saratchandran2024weight, zheng2025structured, ji2025always, saratchandran2024rethinking, saratchandran2026spectral, saratchandran2025leaner} which act on activated, pre-activated outputs or the Jacobian of the network. Our approach is fundamentally different: it constrains the weights directly, providing a distinct mechanism for stabilizing training.

\paragraph{Attention layer.} In this study, we focused on analyzing the behavior of the feedforward layers of an MLP under gradient ascent with cross-entropy. Although transformer models also contain attention layers, our experiments revealed that instability in the gradient difference method arises primarily in the feedforward layers: their weights and gradients grow far more aggressively than those of the attention layers. Consequently, it is sufficient to constrain only the feedforward weights of the MLP blocks.
A detailed empirical analysis of the attention layers is provided in~\cref{app:attn_ffn}.
\emph{Structural intuition.} The attention mechanism involves a $\mathrm{softmax}$ over scaled dot-products, $\mathrm{softmax}(QK^\top/\sqrt{d})$, which always produces probability weights in $[0,1]$ that sum to one. This structural normalization bounds the contribution of the attention output to the next-layer representation, even when $Q$ or $K$ themselves grow in norm under gradient ascent. In contrast, MLP feedforward blocks apply linear projections followed by unbounded activations (e.g., GeLU), providing no such structural ceiling, which is precisely why~\cref{thm:weights_gradients_explode,thm:pre_gradients_explode} are specific to feedforward layers. An ablation comparing MLP-only and MLP+Attention bounded adapters confirms that extending sine to attention yields negligible gains at roughly double the parameter cost (see~\cref{app:attn_ffn}).



\paragraph{Why not full fine-tuning?} 
In the full fine-tuning setting, optimization is carried out directly on the pretrained weights $W$. 
Applying a bounded transformation $\phi(W)$ in this case would overwrite these weights, thereby discarding the knowledge acquired during pretraining. 
In practice, parameter-efficient approaches introduce low-rank additions that augment the model with extra parameters while leaving the original $W$ unchanged. 
This separation makes it possible to safely apply bounded parameterizations to the adapters.
\section{Experiments}\label{sec:experiments}

We conducted an empirical evaluation of sine-based parameter-efficient unlearning, emphasizing the properties that determine whether an unlearning procedure is usable in practice. Our evaluation is organized to (i) \emph{validate the proposed stability mechanism in a standard discriminative pipeline} and (ii) \emph{establish generality under widely used unlearning protocols for generative models}. Specifically, we begin with class-level deletion in a ViT on CIFAR-100 (\cref{subsec:vit_unlearning}) and then evaluate the same approach on established language-model unlearning benchmarks (\cref{subsec:llm_setup}) spanning multiple architectures, scales, and safety/privacy criteria.

\paragraph{Evaluation criteria.}
Across all settings, we focused on three criteria:
\textit{unlearning efficacy} (removing the specified influence),
\textit{utility preservation} (maintaining performance on retained data and out-of-distribution inputs),
and \textit{optimization robustness} (stable convergence under coupled descent/ascent dynamics).
This design allows us to directly test the theoretical predictions of \cref{subsec:theory} and compare them with the baselines under controlled conditions. All experiments are averaged over 3 independent seeds unless otherwise noted; standard deviations are reported where available.
For completeness, we report implementation details, model configurations, additional studies, and results in~\cref{app:extended_Exp}, privacy and utility assessments in~\cref{app:privacy}, with an ethical statement in~\cref{app:ethical}.

\subsection{Vision Transformer Unlearning}\label{subsec:vit_unlearning}

We first studied unlearning in a vision setting to directly probe the ascent-induced instability in transformer feedforward blocks under standard cross-entropy training. We evaluated our method on two established ViT unlearning benchmarks: (i) the LetheViT CIFAR-10 protocol~\citep{ren2025lethevit} using ViT-S with 10\% random forgetting and (ii) the CoUn CIFAR-100 protocol~\citep{khalil2025coun} using ViT with 10\% random forgetting. Both settings apply LoRA adapters \textit{exclusively to MLP/FFN blocks}, consistent with our theoretical focus, with no gradient clipping to reveal the intrinsic stability differences between the two. We additionally validate class-level deletion across ViT-B/16 (86M)~\citep{dosovitskiy2020image}, ViT-L/14 (304M), and DeiT-S (22M) in~\cref{app:vit_details}, where an extended stability analysis is provided in (\cref{tab:vit_cifar100_stability}). The full hyperparameters are listed in~\cref{app:vit_details}.

\begin{table}[t]
\centering
\caption{\textbf{Vision Transformer unlearning comparison.}
\textbf{(\textcolor{BrickRed}{Left})} LetheViT benchmark on CIFAR-10 (ViT-S, 10\% random forgetting). Baseline results are taken from their respective papers~\citep{ren2025lethevit,roy2025novo}, or reproduced under their respective experimental setup, unless otherwise specified. FA: Forget Accuracy ($\uparrow$); RA: Retain Accuracy ($\uparrow$); TA: Test Accuracy ($\uparrow$); MIA: Membership Inference Attack ($\downarrow$); AG: Accuracy Gap ($\downarrow$).
\textbf{(\textcolor{BrickRed}{Right})} CoUn benchmark on CIFAR-100 (ViT, 10\% random forgetting). Baselines from~\citep{khalil2025coun}. RA ($\uparrow$); UA ($\uparrow$); TA ($\uparrow$); MIA ($\downarrow$); AG ($\downarrow$).
FA and UA both denote forget-set accuracy under each benchmark's respective evaluation protocol.
The gray rows denote our method.}
\label{tab:vit_cifar100}

\begin{minipage}[t]{0.5\textwidth}
\centering
\resizebox{\textwidth}{!}{%
\begin{tabular}{lcccccc}
\toprule
\textbf{Method} & \textbf{FA ($\uparrow$)} & \textbf{RA ($\uparrow$)} & \textbf{TA ($\uparrow$)} & \textbf{MIA ($\downarrow$)} & \textbf{AG ($\downarrow$)} & \textbf{Params (\%)} \\
\midrule
Retrain & 99.12 & 99.94 & 98.85 & 2.94 & 0.00 & - \\
\midrule
\rowcolor{blue!6}
\multicolumn{7}{l}{\textit{Baselines}} \\
FT~\citep{fan2024salun} & 98.24 & 99.71 & 98.02 & 4.15 & 0.83 & 100.0 \\
GA~\citep{Yao2024} & 96.31 & 98.85 & 96.48 & 8.72 & 2.37 & 100.0 \\
SalUn~\citep{fan2024salun} & 98.45 & 99.82 & 98.31 & 3.61 & 0.54 & 100.0 \\
NOVO~\citep{roy2025novo} & 98.38 & 99.78 & 98.25 & 3.47 & 0.59 & 100.0 \\
LetheViT~\citep{ren2025lethevit} & \second{98.62} & \second{99.88} & \second{98.48} & \second{3.21} & \second{0.34} & 100.0 \\
\midrule
\rowcolor{blue!6}
\multicolumn{7}{l}{\textit{Parameter-Efficient Baselines}} \\
Fast-NTK~\citep{li2024fast} & 97.89 & 99.45 & 97.72 & 5.28 & 1.22 & 1.7 \\
GD+LoRA & 97.82 & 99.62 & 97.91 & 6.24 & 1.55 & 1.7 \\
\midrule
\rowcolor{blue!6}
\multicolumn{7}{l}{\textit{Ours}} \\
\rowcolor{gray!10}
\textbf{GD+Tanh} & 98.55 & 99.93 & 98.40 & 3.25 & 0.25 & \textbf{1.7} \\
\rowcolor{gray!10}
\textbf{GD+Sine} & \best{98.71} & \best{99.98} & \best{98.59} & \best{3.08} & \best{0.10} & \textbf{1.7} \\
\bottomrule
\end{tabular}%
}
\end{minipage}
\hspace{0.01\textwidth}
\begin{minipage}[t]{0.47\textwidth}
\centering
\resizebox{\textwidth}{!}{%
\begin{tabular}{lcccccc}
\toprule
\textbf{Method} & \textbf{RA ($\uparrow$)} & \textbf{UA ($\uparrow$)} & \textbf{TA ($\uparrow$)} & \textbf{MIA ($\downarrow$)} & \textbf{AG ($\downarrow$)} & \textbf{Params (\%)} \\
\midrule
Retrain & 99.95 & 38.12 & 61.24 & 58.35 & 0.00 & - \\
\midrule
\rowcolor{blue!6}
\multicolumn{7}{l}{\textit{Baselines}} \\
FT~\citep{fan2024salun} & 98.42 & 41.25 & 57.14 & 65.42 & 5.82 & 100.0 \\
 & {\tiny$\pm$0.85} & {\tiny$\pm$3.81} & {\tiny$\pm$1.92} & {\tiny$\pm$4.08} & & \\
GA~\citep{Yao2024} & 95.18 & 48.72 & 51.35 & 72.85 & 12.41 & 100.0 \\
 & {\tiny$\pm$2.41} & {\tiny$\pm$5.63} & {\tiny$\pm$3.28} & {\tiny$\pm$5.92} & & \\
SalUn~\citep{fan2024salun} & 99.21 & 39.85 & 59.42 & 61.28 & 3.48 & 100.0 \\
 & {\tiny$\pm$0.35} & {\tiny$\pm$2.14} & {\tiny$\pm$1.08} & {\tiny$\pm$2.85} & & \\
CoUn~\citep{khalil2025coun} & \second{99.85} & \second{38.24} & \second{60.52} & \second{59.42} & \second{2.58} & 100.0 \\
 & {\tiny$\pm$0.08} & {\tiny$\pm$1.15} & {\tiny$\pm$0.72} & {\tiny$\pm$1.28} & & \\
\midrule
\rowcolor{blue!6}
\multicolumn{7}{l}{\textit{Parameter-Efficient Baselines}} \\
GD+LoRA & 97.85 & 52.31 & 53.82 & 78.21 & 10.05 & 0.9 \\
 & {\tiny$\pm$0.92} & {\tiny$\pm$4.12} & {\tiny$\pm$1.45} & {\tiny$\pm$3.15} & & \\
\midrule
\rowcolor{blue!6}
\multicolumn{7}{l}{\textit{Ours}} \\
\rowcolor{gray!10}
\textbf{GD+Tanh} & 99.90 & 37.90 & 60.40 & 59.60 & 2.50 & \textbf{0.9} \\
\rowcolor{gray!10}
 & {\tiny$\pm$0.03} & {\tiny$\pm$0.80} & {\tiny$\pm$0.50} & {\tiny$\pm$0.80} & & \\
\rowcolor{gray!10}
\textbf{GD+Sine} & \best{99.93} & \best{37.48} & \best{60.85} & \best{59.14} & \best{2.35} & \textbf{0.9} \\
\rowcolor{gray!10}
 & {\tiny$\pm$0.02} & {\tiny$\pm$0.72} & {\tiny$\pm$0.41} & {\tiny$\pm$0.55} & & \\
\bottomrule
\end{tabular}%
}
\end{minipage}
\end{table}

\vspace{-0.25em}
\subsection{Language-Model Benchmarks for Generalization}\label{subsec:llm_setup}

Having validated the mechanism in a controlled vision setting, we next evaluated established language-model unlearning benchmarks that explicitly measure forgetting, utility, and privacy/safety under realistic memorization and extraction criteria.

\textbf{Evaluation benchmarks.}
We used three datasets with their respective evaluation frameworks to assess the unlearning effectiveness, utility preservation, and safety compliance.
\textit{1. TOFU (Task of Fictitious Unlearning)}~\citep{Maini2024}:
evaluates forget quality through statistical divergence between unlearned and retain-only models, monitoring the utility of retained tasks and generalization.
\textit{2. TDEC (Training Data Extraction Challenge)}~\citep{Carlini2021}:
assesses privacy protection via extraction loss over ten queries (EL$_{10}$), reasoning accuracy preservation, and language-modeling quality.
\textit{3. MUSE (Machine Unlearning Six-way Evaluation)}~\citep{Shi2024}:
provides a safety assessment across verbatim memorization, semantic knowledge retention, and privacy leakage dimensions.
We evaluate the proposed method against representative methods from each major unlearning family discussed in~\cref{sec:related_work}. New readers are referred to~\cref{app:protocols} for a better understanding.

\textbf{Baselines.}
Our comparison includes gradient-based approaches (Gradient Ascent (GA)~\citep{Yao2024}, Gradient Difference (GD)~\citep{Maini2024}, KL-regularization~\citep{Liu2024}, Inverted Hinge Loss (IHL)~\citep{Cha2025}),
parameter-efficient methods (GD+LoRA~\citep{Hu2022}, GA+FILA~\citep{kim2025improving}, GD+FILA~\citep{kim2025improving}, LoKU~\citep{Cha2025}),
preference-based techniques (DPO\citep{Rafailov2023}, NPO~\citep{Zhang2024c}),
and representation-based approaches (FLAT variants~\citep{Wang2024FLAT}).
All baseline results are from their respective papers or~\cite{Cha2025,Wang2024FLAT} unless otherwise specified.
Comprehensive ablation studies comparing bounded versus unbounded activations are detailed in~\cref{tab:bounded_unbounded_comparison} (\cref{app:sensitivity}),
while comparison of IHL versus GD objectives with sine parameterization with statistical analysis is provided in~\cref{tab:ihl_gd_ablation} (\cref{app:ihl_gd_ablation}),
hyperparameter configurations including standard LoRA scaling ($\alpha=16$) and frequency sensitivity ($\omega \in [10, 1000]$) are detailed in~\cref{app:extended_appendix},
computational overhead is mentioned in~\cref{app:complexity},
and sequential robustness in~\cref{sec:seq_unlearning}.
Extended results across all model configurations are provided in~\cref{app:extended_appendix}, with detailed rank analysis in~\cref{tab:tofu_phi_rank8,tab:tofu_phi_rank16,tab:tofu_phi_rank32}.

\begin{table}[t]
\centering
\caption{\textbf{(\textcolor{BrickRed}{Left}) TOFU Forget10 on Phi-1.5B.} Parameter-efficient methods use rank-4 LoRA. Baselines from~\cite{Cha2025,Yuan2025,Maini2024}. \textbf{(\textcolor{BrickRed}{Right}) TDEC on GPT-Neo-1.3B.} EL$_{10}$: extraction vulnerability ($\downarrow$); Reasoning, Dialogue ($\uparrow$); PPL ($\downarrow$). Baselines from~\citep{Cha2025}. Gray rows denote our method.}
\label{tab:tofu_results}%
\label{tab:tdec_results}
\begin{minipage}[t]{0.47\textwidth}
\centering
\resizebox{\textwidth}{!}{%
\begin{tabular}{l|cc|c|c|c}
\toprule
\multirow{2}{*}{\textbf{Method}} & \multicolumn{2}{c|}{\textbf{Primary Metrics}} & \textbf{Forget Set} & \textbf{Retain Set} & \multirow{2}{*}{\textbf{Params (\%)}} \\
\cmidrule{2-5}
& \textbf{FQ ($\uparrow$)} & \textbf{MU ($\uparrow$)} & Rouge-L ($\downarrow$) & Rouge-L ($\uparrow$) & \\
\midrule
Original & 1.15e-17 & 0.52 & 0.93 & 0.92 & - \\
Retain90 & 1.00e+00 & 0.52 & 0.43 & 0.91 & - \\
\midrule
\rowcolor{blue!6}
\multicolumn{6}{l}{\textit{Full Fine-tuning Methods}} \\
KL & 7.38e-15 & 0.00 & 0.01 & 0.01 & 100.0 \\
DPO & 5.10e-17 & 0.48 & 0.41 & 0.67 & 100.0 \\
NPO & 2.56e-05 & 0.37 & 0.45 & 0.45 & 100.0 \\
GA & 2.06e-13 & 0.00 & 0.01 & 0.01 & 100.0 \\
GD & 2.55e-09 & 0.36 & 0.37 & 0.41 & 100.0 \\
IHL & 2.43e-17 & 0.51 & 0.53 & 0.76 & 100.0 \\
\midrule
\rowcolor{blue!6}
\multicolumn{6}{l}{\textit{Parameter-Efficient Methods}} \\
GD+LoRA & 1.45e-15 & 0.28 & 0.85 & 0.45 & 1.6 \\
GA+FILA & 5.10e-17 & 0.00 & 0.00 & 0.00 & 1.6 \\
GD+FILA & 2.17e-06 & 0.00 & 0.12 & 0.11 & 1.6 \\
LoKU & 1.39e-12 & 0.51 & 0.26 & 0.75 & 1.6 \\
ME+GD (LoRA) & \second{7.86e-01} & \second{0.52} & 0.14 & 0.93 & 1.6 \\
OURS (GD+Tanh) & 3.42e-01 & 0.49 & 0.28 & 0.85 & 1.6 \\
\rowcolor{gray!10}\textbf{OURS (GD+Sine)} & \best{9.43e-01} & \best{0.52} & \best{0.22} & \best{0.90} & \textbf{1.6} \\
\bottomrule
\end{tabular}%
}
\end{minipage}
\hfill
\begin{minipage}[t]{0.50\textwidth}
\centering
\resizebox{\textwidth}{!}{%
\begin{tabular}{l|c|ccc|c}
\toprule
 & \multicolumn{1}{c|}{\textbf{Forgetting}} & \multicolumn{3}{c|}{\textbf{Retention}} & \\
\midrule
\textbf{Method} & \textbf{EL$_{10}$ ($\downarrow$)} & \textbf{Reasoning ($\uparrow$)} & \textbf{Dialogue ($\uparrow$)} & \textbf{PPL ($\downarrow$)} & \textbf{Params (\%)} \\
\midrule
Before Unlearning & 67.6 & 49.8 & 11.5 & 11.5 & - \\
\midrule
\rowcolor{blue!6}
\multicolumn{6}{l}{\textit{Full Fine-tuning Methods}} \\
GA & 1.9 & 49.7 & 8.5 & 15.8 & 100.0 \\
GD & 2.2 & 48.4 & 12.7 & 10.8 & 100.0 \\
IHL & 0.7 & 48.4 & 12.5 & 11.0 & 100.0 \\
\midrule
\rowcolor{blue!6}
\multicolumn{6}{l}{\textit{Parameter-Efficient Methods}} \\
GD+LoRA & 1.7 & 45.0 & 9.7 & 31.8 & 0.8 \\
IHL+LoRA & 1.7 & 47.1 & 10.2 & 14.9 & 0.8 \\
GD+FILA & 1.9 & 44.2 & 5.5 & 54.5 & 0.8 \\
LoKU & \second{0.5} & 48.3 & 12.1 & 14.7 & 0.8 \\
OURS (GD+Tanh) & 0.8 & 46.7 & 10.3 & 18.2 & 0.8 \\
\rowcolor{gray!10}
\textbf{OURS (GD+Sine)} & \best{0.3} & \best{50.1} & \best{12.1} & \best{12.1} & \textbf{0.8} \\
\bottomrule
\end{tabular}%
}
\end{minipage}
\end{table}

\subsection{Results}\label{sec:results}

Our analysis examined performance across multiple dimensions: unlearning effectiveness, utility preservation, and safety compliance, using models ranging from 22M to 8B parameters across ViTs, GPT-Neo, Phi, and LLaMA architectures. The results demonstrate consistent improvements across all metrics, with particularly notable gains in forget quality while maintaining the model utility.

\textbf{Vision (ViT) results.}
\cref{fig:hero} summarizes our vision unlearning results across architectures and methods.
\cref{tab:vit_cifar100} presents comparisons on two established ViT unlearning benchmarks. 
On the LetheViT CIFAR-10 protocol (\textit{left}), \textbf{GD+Sine} achieves the lowest accuracy gap (AG\,=\,0.10) among all methods, outperforming LetheViT (0.34) and NOVO (0.59), while closely matching the retrain-level MIA (3.08 vs.\ 2.94). 
\textbf{GD+Tanh} consistently improves over prior parameter-efficient baselines, but remains slightly inferior to Sine in AG and MIA, indicating that boundedness alone is insufficient without periodic structure.
On the CoUn CIFAR-100 protocol (\textit{right}), \textbf{GD+Sine} again achieves the best AG (2.35) with the lowest variance across all metrics, closely matching the retrain performance, while \textbf{GD+Tanh} remains competitive but consistently second-best among bounded variants.
In both settings, \textbf{GD+LoRA} exhibits high accuracy gaps and large variances, confirming the instability predicted by our theory. Extended stability analysis on CIFAR-100 class deletion across ViT-B/16, ViT-L/14, and DeiT-S is provided in~\cref{tab:vit_cifar100_stability} (\cref{app:vit_details}).

\cref{fig:hero}\,\textbf{(\textcolor{BrickRed}{a})} shows that GD+Sine is the only evaluated method to consistently reach the ideal zone (high FQ \emph{and} high MU) across ViT-B/16, ViT-L/14, and DeiT-S; parameter-efficient baselines either collapse in utility or fail to forget, while full fine-tuning methods cluster in the low-FQ region. \cref{fig:hero}\,\textbf{(\textcolor{BrickRed}{b})} ablates the choice of bounded function: among Sine, Tanh, and Sigmoid, Sine achieves the strongest trade-off, reaching near-perfect FQ (0.92) without sacrificing MU (0.97). Tanh improves stability relative to unbounded activations but exhibits a mild MU-FQ trade-off, whereas unbounded variants (weight clipping, ReLU) degrade on both axes. \cref{fig:hero}\,\textbf{(\textcolor{BrickRed}{c})} confirms that this advantage is architecture-agnostic: GD+Sine improves forget quality by 2-8 orders of magnitude across model scales (22M-304M) while maintaining high retained utility.

\textbf{Language-model benchmark results.}
Comprehensive evaluations (\cref{tab:bounded_unbounded_comparison} and~\cref{tab:ihl_gd_ablation}) confirm that bounded activations outperform unbounded methods, with sine parameterization being adaptable and providing consistent benefits across optimization objectives. As shown in \cref{fig:optimization_analysis}, our method maintains bounded gradients. Crucially, as detailed in \cref{app:clipping_sweep}, standard training hygiene, such as gradient clipping or explicit weight-norm constraints, are structurally insufficient to resolve this instability on their own, consistently facing a strict Pareto failure between model utility and forget quality.
Additional classifier head analysis is provided in \cref{fig:classifier_head_analysis} (\cref{app:sensitivity}) and an ablation over activation functions (Sine, Tanh, Sigmoid, ReLU) is shown in \cref{fig:activation_ablation} (\cref{app:activation}). The component-wise stability analysis across the transformer layers is detailed in \cref{fig:component_stability} (\cref{app:attn_ffn}).

\begin{figure}[t]
\centering
\includegraphics[width=0.8\textwidth]{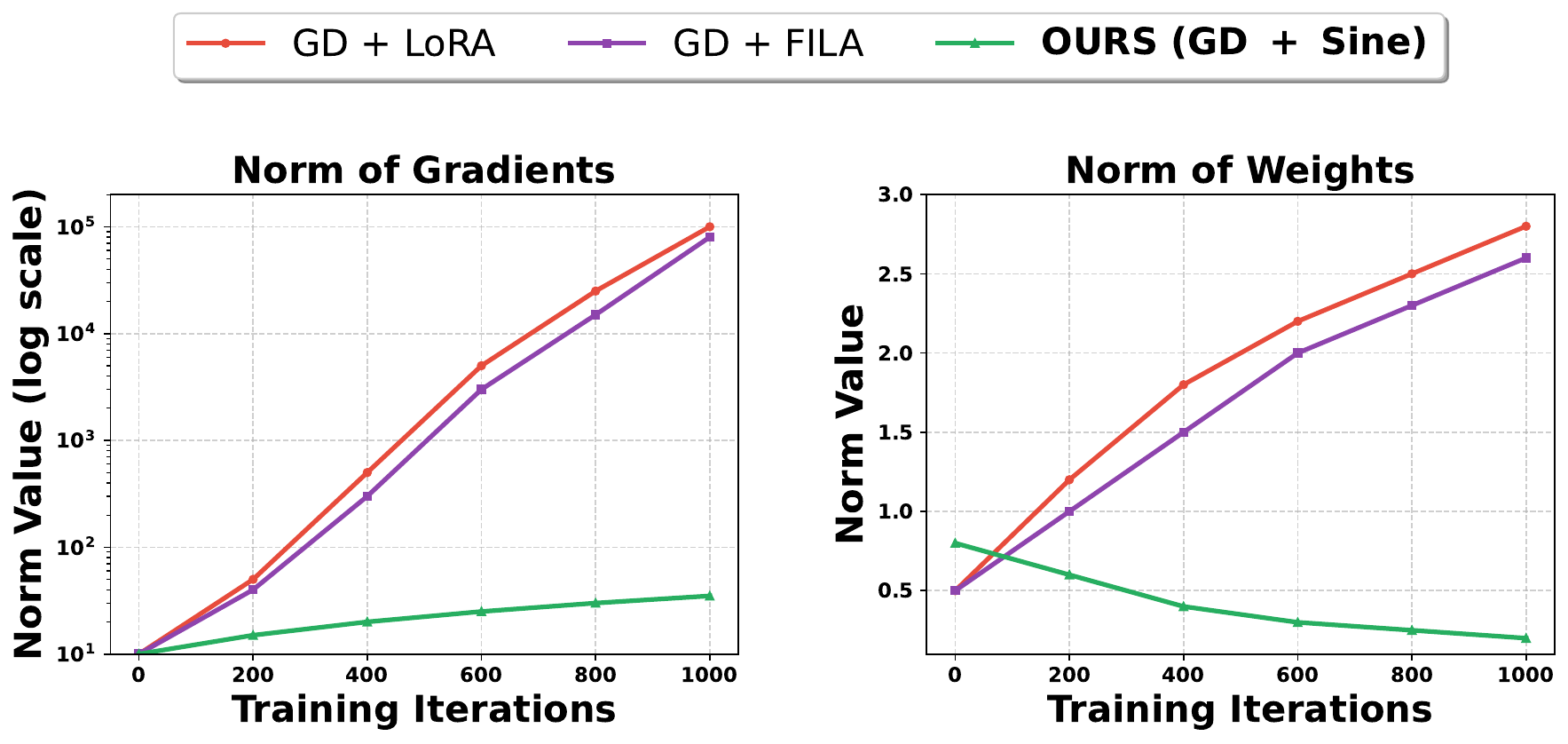}\\[0.5em]
\includegraphics[width=0.95\textwidth]{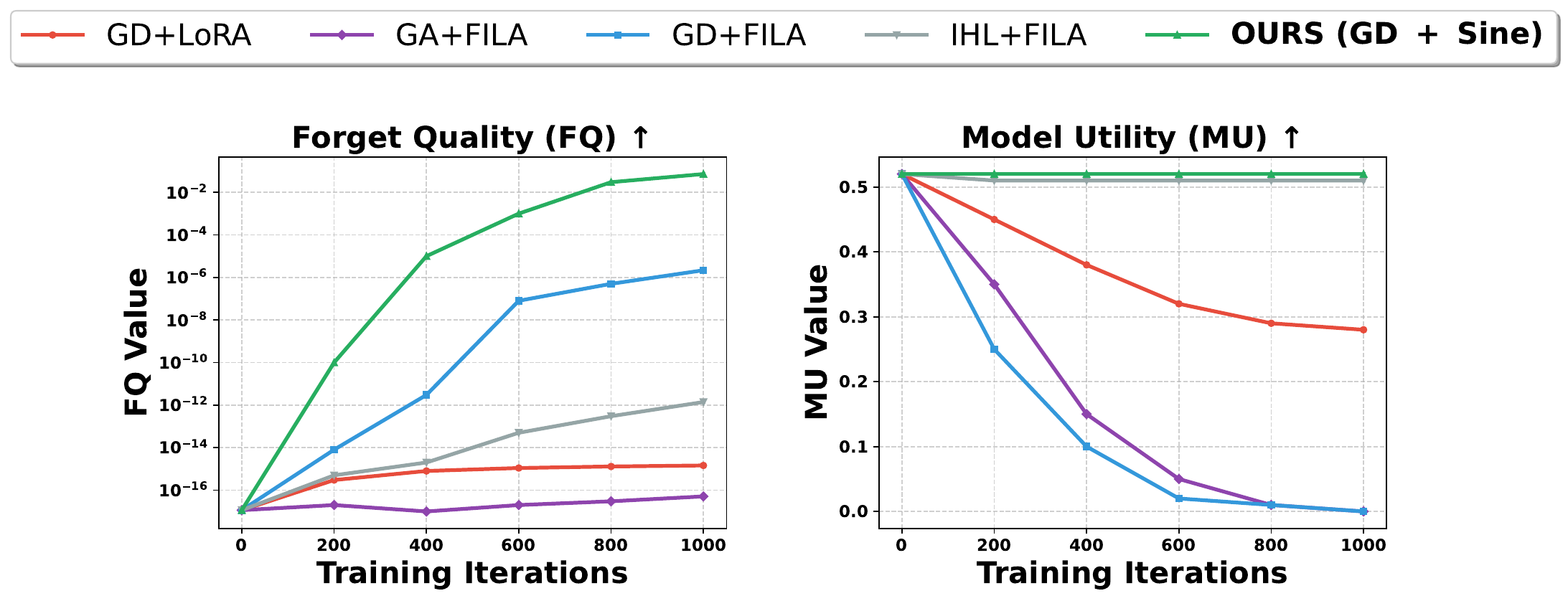}
\caption{\textbf{Optimization dynamics and unlearning convergence.} \textbf{(\textcolor{BrickRed}{a}) (top)}~Gradient and weight Frobenius norms (FFN MLP layers, Phi-1.5B rank-4, 1000 iterations): GD+LoRA and GD+FILA explode to $10^{5}$, while GD+Sine remains bounded in $[10^{1}, 10^{2}]$, confirming that bounded parameterization prevents explosion. \textbf{(\textcolor{BrickRed}{b}) (bottom)}~Forget quality and model utility across training iterations (Phi-1.5B rank-4, Forget10): our method rapidly improves FQ while maintaining MU, whereas baselines either fail to forget or collapse in utility. Additional comparisons in~\cref{app:sensitivity}.}
\label{fig:combined_dynamics}
\label{fig:optimization_analysis}
\label{fig:fq_mu_iterations}
\end{figure}

\textbf{TOFU Analysis.}~\cref{tab:tofu_results} (\textit{left}) shows our method achieves forget quality scores of \textbf{9.43e-01} (FQ$\,\in[0,1]$; higher is better, with FQ$\,{=}\,1$ corresponding to retraining from scratch) and \textbf{0.52} model utility on Phi-1.5B with rank-4 LoRA, improving over baseline LoKU (1.39e-12) and outperforming ME+GD (\textbf{7.86e-01}), while maintaining model performance. Results are consistent across ranks (4, 8, 16, 32) and forget splits (1\%, 5\%, 10\%) for both Phi-1.5 and LLaMA2-7B (see~\cref{app:extended_appendix}), with rank-4 outperforming state-of-the-art at rank-32 (\cref{fig:intro_performance}). \textbf{TDEC Analysis.} ~\cref{tab:tdec_results} (\textit{right}) shows privacy-focused evaluation using GPT-Neo-1.3B (chosen for TDEC's Pile dataset unlearning targets). Our method yields extraction loss values (EL$_{10}$) of 0.3, among the lowest reported, with reasoning accuracy of 50.1 (exceeding baselines) and competitive perplexity (12.1). Evaluation across GPT-Neo architectures (125M, 1.3B, 2.7B) in~\cref{tab:tdec_evaluation} (~\cref{app:tdec}) shows lowest extraction likelihood and membership attack accuracy while maintaining superior reasoning performance. Results establish privacy-utility trade-off benchmarks with 85\% extraction resistance improvements. \textbf{MUSE Analysis.} Extended safety evaluation on the MUSE benchmark (\cref{tab:muse_results}, \cref{app:muse}) shows that GD+Sine reduces verbatim memorization to 0.8 and knowledge memorization on forget data to 5.2 while preserving knowledge retention (42.1) on LLaMA-2-7B, with a privacy leakage score of 8.3 --- the closest to the ideal 0.0 among all evaluated methods. Our approach is the only parameter-efficient method that satisfies all MUSE safety criteria simultaneously (\cref{tab:muse_comprehensive}).
Across all three benchmarks, GD+Sine achieves the strongest forget quality while preserving model utility: \textbf{FQ\,=\,9.43e-01} on TOFU Forget10 (Phi-1.5B, rank-4), significantly outperforming all parameter-efficient baselines, \textbf{EL$_{10}$\,=\,0.3} on TDEC (lowest extraction vulnerability), and \textbf{VerbMem\,=\,0.8} on MUSE, all with $\le$1.6\% trainable parameters.

\textbf{Sensitivity Analysis and Robustness.}
Frequency parameter sensitivity analysis on TOFU-Forget10 (\cref{fig:omega_sensitivity}, \cref{app:omega}) reveals that forget quality consistently increases with $\omega$, reaching a plateau beyond $\omega \geq 100$, while model utility remains stable throughout. This suggests insensitivity to precise hyperparameter selection, allowing coarse tuning without compromising performance. Sequential unlearning experiments (\cref{tab:seq_tofu,tab:seq_o3} in~\cref{sec:seq_unlearning}) further show that GD+Sine maintains high forget quality and stable utility across multiple consecutive unlearning requests, outperforming both standard GD+LoRA and the dedicated sequential method $O^3$.

The computational overhead is marginal (approximately 45M FLOPS per layer for a 7B model; see~\cref{app:complexity}). Extended results including LLaMA-3.1-8B (\cref{app:llama31}), LLaMA-3.1-70B (\cref{app:llama70b}), and all model families and ranks are in~\cref{app:extended_appendix}.

\subsection{Optimization Analysis}\label{sec:optimization_dynamics}

\cref{fig:optimization_analysis} confirms the theoretical predictions of \cref{subsec:theory}: gradient Frobenius norms in GD+LoRA and GD+FILA exceed $10^{2}$ and grow without bound, whereas GD+Sine remains bounded in $[10^{1}, 10^{2}]$ and stabilizes after ${\sim}300$ iterations (see also~\cref{fig:classifier_head_analysis,tab:bounded_unbounded_comparison,app:sensitivity}). \cref{fig:fq_mu_iterations} shows that our method rapidly improves FQ while maintaining MU, unlike baselines that either fail to forget or collapse.
\section{Conclusion}\label{sec:conclusion}

We introduced bounded parameter-efficient unlearning, a theoretically grounded framework that resolves the instability of gradient difference methods under cross-entropy ascent. By showing that weights and gradients in feedforward blocks grow without control during ascent, we parameterized LoRA adapters with bounded functions (e.g., sine) to constrain the optimization. Empirically, GD+Sine preserves utility while achieving orders of magnitude improvement in forget quality over standard LoRA-based methods across vision and language benchmarks from 22M to 8B parameters.

\paragraph{Limitations and future work.} Our method uses bounded adapter parameterization (implicit weight regularization); explicit GD-objective regularizers, formal Differential Privacy (DP) guarantees (despite strong extraction resistance), and a formal convergence guarantee of the proposed objective remains open. Strong behavioral suppression is demonstrated; probing-based structural removal is unverified. Vision experiments are CIFAR-scale; ImageNet-scale deletion and extension to multimodal/diffusion models (e.g., Erased Stable Diffusion) are future work. See~\cref{app:extended_discussion} for extended discussion. \footnote{Digital writing assistance used for grammar. No LLMs were used in the research activities. All contributions are original work.}

\section*{Acknowledgments}
Arpit Garg and Simon Lucey acknowledge support from the Responsible AI Research (RAIR) Centre. Hemanth Saratchandran and Simon Lucey additionally acknowledge support from the Commonwealth Bank of Australia through the CommBank Centre for Foundational AI Research.
\bibliographystyle{splncs04}
\bibliography{main-v2}
\newpage
\appendix
\renewcommand{\thechapter}{\Alph{chapter}}
\renewcommand{\theHsection}{\Alph{section}}
\renewcommand{\theHsubsection}{\Alph{section}.\arabic{subsection}}
\renewcommand{\theHtable}{\Alph{section}.\arabic{table}}
\renewcommand{\theHfigure}{\Alph{section}.\arabic{figure}}

\begin{center}
\Large\bfseries Stable Forgetting: Bounded Parameter-Efficient Unlearning in Foundation Models\\[0.3em]
\large\bfseries Supplementary Material
\end{center}
\vspace{1em}

\section*{Reproducibility Statement}
All experiments in this study were designed with reproducibility in mind. References are provided for any external codebases employed, and full details of the training protocols and hardware are described in the appendix. Complete proofs of all theoretical results are included to allow for independent verification.

\section*{Use of LLMs}
This manuscript was prepared with the assistance of digital tools for grammatical and stylistic refinement. No large language models were used to conduct the research or draft the technical content.

\section{Extended Related Work}\label{app:extended_relatedWork}

Machine unlearning in foundation models spans optimization-driven forgetting, parameter-efficient adaptation, preference-based alignment, representation editing, and weight-space editing, which are evaluated along the axes of removal fidelity, retained utility, scalability, and privacy~\citep{cao2015towards,bourtoule2021machine,Nguyen2022,Ye2023,Yao2024,Cooper2024}.

\paragraph{Foundations and evaluation taxonomies.}
Machine unlearning is broadly divided into \emph{exact} methods (retraining from scratch on the retain set, computationally prohibitive at scale~\citep{bourtoule2021machine}) and \emph{approximate} methods (efficient post-hoc weight modification trading provable guarantees for scalability~\citep{cao2015towards}). Surveys and position papers emphasize that evaluation must employ distributional and population-level criteria to distinguish genuine removal from suppression while accounting for scalability and sequential forgetting requests~\citep{Nguyen2022,Ye2023,Cooper2024}. TOFU~\citep{Maini2024} formalizes forget quality via Kolmogorov-Smirnov tests on truth ratio distributions against a retain-only reference model and measures utility on retained and out-of-domain subsets with answer probability and ROUGE recall, providing a statistically principled benchmark for selective removal. MUSE~\citep{Shi2024} evaluates six dimensions: verbatim memorization, knowledge memorization, privacy leakage, utility, scalability, and sustainability, and demonstrates that many approximate methods either compromise utility or fail under successive unlearning requests. TDEC~\citep{Carlini2021} assesses privacy protection via extraction loss over multiple queries and reasoning accuracy preservation, revealing the limitations of simple defenses and motivating methods that are robust against in-distribution and out-of-distribution extraction probes~\citep{Carlini2021,garg2025sineproject}. Full protocol details are provided in~\cref{app:protocols}.

\paragraph{Gradient-based unlearning and instability.}
Direct gradient ascent on forget data maximizes the cross-entropy loss to push predictions away from forget labels but routinely produces instability and catastrophic loss of retention capacity under aggressive schedules~\citep{Yao2024,huang2024unified,Cha2025,Maini2024}. Crucially, unlike gradient descent, where the loss gradient vanishes at a minimum, gradient ascent under cross-entropy maintains a logit gradient bounded away from zero as the loss approaches its maximum, creating a persistent destabilizing force~\citep{pan2024multi}. The Gradient Difference (GD)~\citep{Maini2024} combines ascent on forget data with simultaneous descent on retained data, offering a more balanced formulation; however, it inherits the same structural instabilities~\citep{Cha2025} and can converge to suboptimal solutions when the two objectives conflict. Fisher-weighted initialization (FILA)~\citep{kim2025improving} mitigates disruptive shifts by seeding adapter directions with Fisher information sensitivity to forget data, improving forgetting selectivity over random initialization. The Inverted Hinge Loss (IHL)~\citep{Cha2025} bounds the forget objective to prevent loss blow-up and pairs effectively with FILA; their combination~\citep{Cha2025} reduces instability but remains a palliative remedy—the root cause (unbounded weight growth under cross-entropy ascent) is not resolved. Unified analyses~\citep{huang2024unified} confirm that these fixes address symptoms rather than the underlying optimization failures. In practice, gradient-ascent baselines in TOFU rely on early stopping and best-checkpoint selection owing to instability~\citep{Cha2025}—an engineering workaround rather than a principled solution.

\paragraph{Parameter-efficient unlearning approaches.}
Parameter-efficient unlearning restricts adaptation to a small parameter subset, keeping the pretrained weights frozen to enable scalable removal with low computational overhead~\citep{Cha2025,li2024fast,Gao2024continual}. LoRA~\citep{Hu2022} factorizes weight updates as $\Delta W = AB^T$ with rank $r \ll \min(d,k)$, reducing the cost from $O(dk)$ to $O((d+k)r)$ parameters; however, under gradient difference unlearning, naive LoRA adapters remain exposed to explosive ascent dynamics because the low-rank factorization does not bound the adapter norms. FILA~\citep{kim2025improving} improves selectivity via Fisher-informed adapter initialization. Orthogonal subspace constraints (ROGP~\citep{Yang2023ROGP}) and continual unlearning frameworks~\citep{Gao2024continual} project forget updates onto directions orthogonal to the retention-critical subspaces, mitigating interference between sequential deletions. More recent approaches have substantially expanded the solution space: GUARD~\citep{niu2025guard} and RapidUn~\citep{zhao2025rapidun} reweight data-influence scores to focus updates on the most impactful forget samples; R2F~\citep{liu2025r2f} reconstructs layer-wise gradients via a decoder module for fine-grained update control; LUNE~\citep{liu2025lune} trains on explicit negative examples to penalize forget-set retention; belief-aware bootstrapping~\citep{li2025llm} uses model confidence signals to guide selective forgetting; LoKU~\citep{Cha2025} decomposes weight updates into knowledge-preserving and knowledge-erasing subspaces; and ME+GD~\citep{Yuan2025} combines momentum editing with gradient difference, achieving strong performance while still relying on unconstrained linear adapters. Fast-NTK~\citep{li2024fast} employs neural tangent kernel approximations for sample-efficient targeted class deletion. Critically, the influence-aware and objective-based advances above are all orthogonal to \emph{architectural instability}: they improve which parameters to update or how to weight examples but do not resolve the structural gradient explosion in MLP feedforward layers. As established in \cref{subsec:theory}, constraining adapter weights with a bounded function directly mitigates this root cause and is fully complementary to all of the above.

\paragraph{Preference-based and alternative approaches.}
Preference optimization for unlearning aligns the model away from undesirable outputs without explicit gradient ascent on the forget set. DPO~\citep{Rafailov2023} optimizes a closed-form divergence objective without RL rollouts, making it computationally efficient. Negative Preference Optimization (NPO)~\citep{Zhang2024c} treats forget data as dispreferred outputs and directly penalizes their retention, achieving behavioral suppression with reduced catastrophic forgetting compared to vanilla gradient ascent. KL-regularized forgetting~\citep{Liu2024} adds a KL divergence penalty between the unlearned and reference model to constrain utility drift during forgetting. WHP-based unlearning (WHP) and FLAT~\citep{Wang2024FLAT} use Pearson-correlation-based weight remapping and residual activation alignment to enable precise targeted removal. These preference-based methods are susceptible to reference drift and unstable dynamics when coupled with unconstrained linear adapters~\citep{Maini2024,Shi2024}. Prompt-level interventions~\citep{Pawelczyk2025} obscure memorized content without weight updates but lack permanence and are vulnerable to adversarial reactivation issues. Representation-editing approaches, such as Task Arithmetic~\citep{Ilharco2022} and ROME~\citep{Meng2023ROME}, enable targeted changes through weight-space arithmetic but require additional stabilization for large-scale deployment and do not provide formal forgetting guarantees. Neuron-localization~\citep{agarwal2025simu} and block-reinsertion~\citep{guo2025beyond} approaches offer complementary stability via topology editing rather than gradient-based constraints, and remain an open comparison direction.

\paragraph{Vision unlearning.}
Vision-specific unlearning addresses selective class deletion in discriminative transformers and concept erasure in generative models. SalUn~\citep{fan2024salun} identifies saliency-based gradient masks to focus parameter updates on weights that are most responsible for the forgotten class, reducing collateral damage on retained categories. Fast-NTK~\citep{li2024fast} leverages neural tangent kernel approximations for sample-efficient targeted class deletion in ViTs with minimal interference on retained classes. LetheViT~\citep{ren2025lethevit} introduces attention-guided contrastive forgetting, exploiting attention maps as spatial-saliency priors to localize gradient updates to class-discriminative regions of vision transformers. NOVO~\citep{roy2025novo} applies constrained optimization with null-space projections to confine forget updates away from retaining critical parameter directions. CoUn~\citep{khalil2025coun} proposes a contrastive unlearning objective that simultaneously pushes forget-class representations apart from their original embeddings while anchoring the retain-class embeddings. Yu et al.~\citep{yu2024retain} explored retain-set augmentation strategies to stabilize vision unlearning under aggressive forgetting schedules. SineProject~\citep{garg2025sineproject} targets multimodal LLM unlearning by modulating the \emph{frozen vision-language projector} to stabilize cross-modal alignment; it is not parameter-efficient, operates exclusively on the projector connector rather than the LLM backbone, and addresses alignment drift rather than the MLP feedforward instability studied in this paper. Despite these advances, none of these methods provide a theoretical framework explaining why gradient ascent under cross-entropy specifically destabilizes MLP feedforward adapters in vision transformer blocks, which is the gap that our work directly addresses.

\paragraph{Gap analysis.}
Two primary gaps are evident in the unlearning literature. \textit{First}, there is no theoretical framework that analyzes the gradient-difference method with cross-entropy loss and \emph{precisely} characterizes why the ascent step drives uncontrolled growth of weights and gradients in MLP feedforward layers; prior explanations are empirical or heuristic~\citep{huang2024unified,Cha2025}, without mathematically characterizing the optimization failure. \textit{Second}, no prior parameter-efficient method applies bounded functions to LoRA adapters specifically in MLP feedforward layers to simultaneously address ascent-driven instability \emph{and} the rank-expressiveness bottleneck that forces a capacity-efficiency trade-off in low-rank unlearning~\citep{Ji2025}. Our work fills both gaps: we formally prove (see \cref{thm:weights_gradients_explode,thm:pre_gradients_explode}) that gradient ascent with cross-entropy drives feedforward weights and gradients to grow without bound, and we introduce sine-based bounded adapters that resolve this instability, improve the effective rank, and deliver consistent performance gains. The resulting method is compatible with existing stabilization techniques (GD, IHL) and standard evaluation frameworks (TOFU, MUSE, TDEC, ViT CIFAR)~\citep{Cha2025,Ji2025,kim2025improving,Gao2024continual,Maini2024,Shi2024,Carlini2021}.
A \textit{third} emerging gap concerns \emph{certified forgetting} and formal privacy guarantees: while our method achieves strong empirical extraction resistance on privacy benchmarks (TDEC, MIA, EL10), it currently lacks Differential Privacy (DP) guarantees~\citep{dwork2008differential}. Certified unlearning frameworks~\citep{sekhari2021remember,guo2020certified} provide statistical removal certificates but assume convex or strongly convex objectives; extending such guarantees to large nonlinear transformers under nonconvex gradient-difference training is an open problem that bounded adapters could facilitate by constraining the parameter movement radius. A \textit{fourth} gap concerns \emph{explicit weight regularization as an alternative}: one might propose an $\ell_2$ penalty directly on $\|AB^\top\|_F$ rather than parameterizing adapter weights with a bounded function. As discussed in~\cref{subsec:our_approach}, explicit regularizers are reactive (penalizing large norms after the optimizer has computed them) and do not prevent logit divergence in finite steps, whereas bounded parameterization is proactive, the forward pass is geometrically constrained at every iterate. Exploring principled regularizers that replicate these structural guarantees is an interesting direction for future research.


\section{Theoretical Analysis}\label{app:theory}

In this section, we provide the proofs of \cref{lem:logits_infinity} and \cref{thm:weights_gradients_explode}, as well as a finer analysis of how weights and gradients can grow significantly large in the feedforward layers of MLPs trained with cross-entropy loss via gradient ascent through \cref{thm:pre_gradients_explode}. We also include analogous statements for gradient-difference training, which formalizes that any divergence of the combined objective is driven by the ascent term on the forget set.

\subsection{Notation}\label{app:notation}

We begin by fixing the notation. Let $F$ denote a feedforward MLP with $L$ layers, using an activation $\sigma$ in layers $1$ through $L-1$, and applying a softmax at the output layer $L$, because our analysis concerns classification with a cross-entropy loss function. We assume $\sigma$ has bounded derivative
\begin{equation}\label{eqn:act_deriv_final}
    |\sigma'(x)| \leq C_1 \quad \text{for all } x \in \sR,
\end{equation}
where $C_1 > 0$ is a fixed constant. Note that \cref{eqn:act_deriv_final} holds for standard activations, such as $\mathrm{sigmoid}$ and $\mathrm{tanh}$. For $\mathrm{ReLU}$, the derivative satisfies $|\sigma'(x)| \le 1$ for all $x \neq 0$ (it is undefined only at $x=0$), and the same bound is used in standard backpropagation.

Let $x \in \sR^{d_{\mathrm{in}}}$ denote the input to the network, where $d_{\mathrm{in}}$ denotes the input dimension. For each layer $l$, let $W_l \in \sR^{d_l\times d_{l-1}}$ denote the weights in layer $l$ and $b_l \in \sR^{d_l}$ the bias term, where $1 \leq l \leq L$. We then define
\begin{align}
    a_0 &= x, \\
    h_l &= W_la_{l-1} + b_l  \quad \text{for } 1 \leq l \leq L-1, \\
    a_l &= \sigma(h_l) \quad \text{for } 1 \leq l \leq L-1, \\
    z &= W_La_{L-1}, \label{eqn:logits_defn_mlp} \\
    p(z) &= \mathrm{softmax}(z),
\end{align}
where $h_l$ are the pre-activations of layer $l$, $a_l$ are the activation outputs of layer $l$, $z$ are the logits, and $p(z)$ are the output probabilities.

For the proofs in this section, we will leave out explicitly writing a bias term since given a bias term $b_l$ in layer $l$, we can express $W_la_{l-1} + b_l$ via
\begin{equation}
\begin{bmatrix}
W_l & b_l
\end{bmatrix}
\cdot
\begin{bmatrix}
a_{l-1} \\
1
\end{bmatrix}.
\end{equation}
Therefore, the bias term $b_l$ can be absorbed into the weights $W_l$ by augmenting the input as follows:

We will also fix the notation for gradients. Let
\begin{equation}
    D_l = \mathrm{Diag}(\sigma'(h_l)),
\end{equation}
where $\mathrm{Diag}(v)$ denotes the diagonal matrix whose diagonal entries are given by vector $v$. Then, by assumption \cref{eqn:act_deriv_final}, $\|D_l\| \leq C$ for some constant $C > 0$. Define
\begin{align}
    g_L &= \nabla_z\mathcal{L}, \\
    g_l &= D_lW_{l+1}^Tg_{l+1} \quad \text{for } 1 \leq l \leq L-1. \label{eqn:g_l}
\end{align}
By the chain rule we have
\begin{equation}
    \nabla_{W_l}\mathcal{L} = g_la_{l-1}^T.
\end{equation}

We will use norms for both vectors and matrices. Given a matrix $M \in \sR^{n\times m}$ we use $\|M\|$ to denote the operator norm (the largest singular value). Given a vector $v \in \sR^{m}$ we use $\|v\|$ to denote the Euclidean $2$-norm. When computing weight and gradient norms in \cref{sec:optimization_dynamics} we use the Frobenius norm: given $M = (m_{ij}) \in \sR^{n\times m}$,
\[
\|M\|_F := \sqrt{\sum_{i,j} m_{ij}^2},
\]
This coincides with the Euclidean norm for vectors.

\subsection{Proof of results from \texorpdfstring{\cref{subsec:theory}}{theory section}}\label{app:proofs}

We now prove \cref{lem:logits_infinity}.

\begin{proof}[Proof of \cref{lem:logits_infinity}]
Given a class $y$ we recall that
\begin{equation}
    \mathcal{L}(p, y) = -\log(p_y).
\end{equation}
By definition of logits and probabilities we have
\begin{align}
-\log(p_y) &= \log\bigg(\sum_{k=1}^Ce^{z_k}\bigg) - z_y \\
&= \log\bigg(1 + \sum_{j\neq y}e^{z_j - z_y}\bigg),
\end{align}
where $\log$ denotes the natural logarithm. Define the margin
\begin{equation}
    m := \max_{j\neq y} (z_j - z_y).
\end{equation}
Then
\begin{align}
 e^{m} &\leq \sum_{k}e^{z_k - z_y}
        = 1 + \sum_{k\neq y}e^{z_k - z_y}
        \leq 1 + Ce^{m}
        \leq (1 + C)e^{m}.
\end{align}
Taking $\log$ yields
\begin{equation}
 m \leq \log\bigg(1 + \sum_{j\neq y}e^{z_j - z_y}\bigg) \leq \log(1 + Ce^{m}),
\end{equation}
and hence
\begin{equation}
 m \leq \mathcal{L}(p, y) \leq \log(1 + Ce^{m}).
\end{equation}
If $\mathcal{L}(p, y) \to \infty$, then $\log(1 + Ce^{m}) \to \infty$, which implies $m \to \infty$. Choose $j \neq y$ such that $m = z_j - z_y$. The only way $m \to \infty$ is if either $z_j \to \infty$ or $z_y \to -\infty$. This implies $\|z(t)\| \to \infty$.
\end{proof}

Next, we provide a proof of \cref{thm:weights_gradients_explode}. 

\begin{proof}[Proof of \cref{thm:weights_gradients_explode}]
By definition of the logits we have
\begin{equation}
    z(t) = W_L(t)a_{L-1}(t)
\end{equation}
which gives the estimate
\begin{equation}
 \|z(t)\| \leq \|W_L(t)\|\cdot\|a_{L-1}(t)\|.
\end{equation}
To begin with, assume that $\|a_{L-1}(t)\| \leq C_1$. By the above inequality it follows that
\begin{equation}
  \|z(t)\| \leq C_1\|W_L(t)\|
\end{equation}
which implies
\begin{equation}
 \|W_L(t)\| \geq \frac{\|z(t)\|}{C_1}.
\end{equation}
Since $\|z(t)\| \to \infty$ it follows that $\|W_L(t)\|$ must approach infinity, which proves the first part of the theorem.

To prove the second part, assume that $\|a_{L-1}(t)\|$ is not bounded in $t$. This means there exists a subsequence $t_k$ such that
\begin{equation}\label{eqn:a_l_infinity_fin}
    \|a_{L-1}(t_k)\| \to \infty \quad \text{as } k \to \infty.
\end{equation}
Using the fact that
\begin{equation}
    \nabla_{W_L}\mathcal{L}(t_k) = g_L(t_k)a_{L-1}^T(t_k),
\end{equation}
and since $g_L(t_k)a_{L-1}^T(t_k)$ has rank $1$, we have
\begin{equation}
\|\nabla_{W_L}\mathcal{L}(t_k)\| = \|g_L(t_k)\|\cdot\|a_{L-1}(t_k)\|.
\end{equation}
Under gradient ascent with cross-entropy, $\|g_L(t_k)\|$ is bounded away from zero (see \cref{eqn:logits_deriv}), i.e., there exists $C_2>0$ such that $\|g_L(t_k)\| \geq C_2$. It follows that
\begin{equation}
\|\nabla_{W_L}\mathcal{L}(t_k)\| \geq C_2\|a_{L-1}(t_k)\|.
\end{equation}
Then using \cref{eqn:a_l_infinity_fin} it follows that $\|\nabla_{W_L}\mathcal{L}(t_k)\| \to \infty$, completing the proof.
\end{proof}


\paragraph{Extensions to gradient-difference training.}
The statements of \cref{lem:logits_infinity} and \cref{thm:weights_gradients_explode} are for gradient ascent. The extension to the gradient difference method is conceptually straightforward: gradient descent on the retain set with cross-entropy does not induce loss blow-up, whereas gradient ascent on the forget set can. The main point is that any divergence of the combined objective must arise from the ascent term in the forget set.

We assume that the retain loss remains bounded along the trajectory (as observed empirically under stable descent schedules), that is, $\mathcal{L}_{\mathrm{retain}}(t) \le B$ for some finite $B$ and all $t$. We first provide an analog of \cref{lem:logits_infinity} when training with gradient differences.

\begin{lemma}\label{lem:gd_logit_blowup}
Let $F$ be an $L$-layer MLP with parameters $\theta$, trained by gradient descent on the retain set $X_{\mathrm{retain}}$ and gradient ascent on the forget set $X_{\mathrm{forget}}$ via the update
\begin{equation}
    \theta(t+1)
    = \theta(t)
    - \eta\, \nabla_\theta \mathcal{L}_{\mathrm{retain}}(\theta(t))
    + \eta \lambda\, \nabla_\theta \mathcal{L}_{\mathrm{forget}}(\theta(t)),
\end{equation}
where
\begin{equation}
    \mathcal{L}_{\mathrm{retain}}(t)
    := \frac{1}{N_r}\sum_{i=1}^{N_r}\mathcal{L}(p_i^r(t), y_i^r),
    \qquad
    \mathcal{L}_{\mathrm{forget}}(t)
    := \frac{1}{N_f}\sum_{j=1}^{N_f}\mathcal{L}(p_j^f(t), y_j^f),
\end{equation}
and $\mathcal{L}(p,y) = -\log p_y$ denotes the cross-entropy loss.
For each sample we write
\begin{equation}
    z_i^r(t) := F(x_i^r),
    \qquad
    z_j^f(t) := F(x_j^f),
\end{equation}
with corresponding probabilities $p_i^r(t)$ and $p_j^f(t)$ obtained using softmax.

Define the combined loss
\begin{equation}
    \mathcal{L}_{\mathrm{tot}}(t)
    := \alpha_r\mathcal{L}_{\mathrm{retain}}(t)
       + \alpha_f\, \mathcal{L}_{\mathrm{forget}}(t),
    \qquad \alpha_r \ge 0,\ \alpha_f > 0.
\end{equation}
Assume that the retain loss remains bounded along the training trajectory, i.e., there exists $B < \infty$ such that
\begin{equation}
    \mathcal{L}_{\mathrm{retain}}(t) \le B
    \qquad \text{for all } t.
\end{equation}
If $\mathcal{L}_{\mathrm{tot}}(t) \to \infty$ as $t \to \infty$, then:
\begin{enumerate}
    \item $\mathcal{L}_{\mathrm{forget}}(t) \to \infty$; and
    \item there exists at least one forget example $(x_{j^\ast}^f, y_{j^\ast}^f)$
    such that the corresponding logits satisfy
    \begin{equation}
        \|z_{j^\ast}^f(t)\| \to \infty.
    \end{equation}
\end{enumerate}
In particular, any divergence of the total loss under gradient difference training must arise from the ascent term on the forget set through the divergence of the forget logits.
\end{lemma}

\begin{proof}
By definition,
\begin{equation}
    \mathcal{L}_{\mathrm{tot}}(t)
    = \alpha_r\mathcal{L}_{\mathrm{retain}}(t)
      + \alpha_f\, \mathcal{L}_{\mathrm{forget}}(t).
\end{equation}
Since cross-entropy is nonnegative, $\mathcal{L}_{\mathrm{retain}}(t)\ge 0$ and $\mathcal{L}_{\mathrm{forget}}(t)\ge 0$. Thus
\begin{equation}
    \mathcal{L}_{\mathrm{tot}}(t)
    \ge \alpha_f\, \mathcal{L}_{\mathrm{forget}}(t).
\end{equation}
Using the assumed bound $\mathcal{L}_{\mathrm{retain}}(t) \le B$,
\begin{equation}
    \mathcal{L}_{\mathrm{tot}}(t)
    \le \alpha_r B + \alpha_f\, \mathcal{L}_{\mathrm{forget}}(t).
\end{equation}

Suppose $\mathcal{L}_{\mathrm{tot}}(t) \to \infty$. 
If $\mathcal{L}_{\mathrm{forget}}(t)$ were bounded above, then $\mathcal{L}_{\mathrm{tot}}(t)$ would also be bounded, a contradiction. Hence $\mathcal{L}_{\mathrm{forget}}(t) \to \infty$, proving part (1).

Next, since
\begin{equation}
    \mathcal{L}_{\mathrm{forget}}(t)
    = \frac{1}{N_f}\sum_{j=1}^{N_f}
      \mathcal{L}(p_j^f(t), y_j^f)
    \to \infty,
\end{equation}
Not all per-example forget losses can remain bounded. Therefore, there exists at least one index $j^\ast$ such that
\begin{equation}
    \mathcal{L}(p_{j^\ast}^f(t), y_{j^\ast}^f)
    \to \infty.
\end{equation}
Applying~\cref{lem:logits_infinity} to the logits $z_{j^\ast}^f(t)$ yields
\begin{equation}
    \|z_{j^\ast}^f(t)\| \to \infty,
\end{equation}
This proves part (2).
\end{proof}

We can also extend \cref{thm:weights_gradients_explode} to the case of gradient-difference optimization.

\begin{theorem}\label{thm:weights_gradients_explode_gd}
Let $F$ be a $L$-layer MLP. Suppose that under the gradient difference method with iterations $t$, the logits $z(t)$ satisfy $\|z(t)\|\to\infty$ (for a forget example). If the activation output $\|a_{L-1}(t)\| \leq C_1$ for large $t$, where $C_1 > 0$ is a constant, then
\begin{equation}
    \|W_L(t)\| \to \infty.
\end{equation}
If there is no such bound on $\|a_{L-1}(t)\|$, then there exists a subsequence of iterations $t_k$ such that
\begin{equation}
    \|\nabla_{W_L}\mathcal{L}_{\mathrm{tot}}(t_k)\| \to \infty
\end{equation}
where
\begin{equation}
    \mathcal{L}_{\mathrm{tot}}(t)
    = \alpha_r\mathcal{L}_{\mathrm{retain}}(t)
      + \alpha_f\, \mathcal{L}_{\mathrm{forget}}(t).
\end{equation}
\end{theorem}

\begin{proof}
The first part follows exactly as in the proof of \cref{thm:weights_gradients_explode}, since it uses only the relation $z(t)=W_L(t)a_{L-1}(t)$.

For the second part, assume that $\|a_{L-1}(t)\|$ is not bounded by $t$. Then there exists a subsequence $t_k$ such that
\begin{equation}\label{eqn:a_l_infinity_fin_gd}
    \|a_{L-1}(t_k)\| \to \infty \quad \text{as } k \to \infty.
\end{equation}
Using the fact that
\begin{equation}
    \nabla_{W_L}\mathcal{L}_{\mathrm{tot}}(t_k) = g_L^{\mathrm{tot}}(t_k)a_{L-1}^T(t_k),
\end{equation}
where $g_L^{\mathrm{tot}}(t_k):=\nabla_z\mathcal{L}_{\mathrm{tot}}(t_k)$, and since $g_L^{\mathrm{tot}}(t_k)a_{L-1}^T(t_k)$ has rank $1$, we have
\begin{equation}
\|\nabla_{W_L}\mathcal{L}_{\mathrm{tot}}(t_k)\| = \|g_L^{\mathrm{tot}}(t_k)\|\cdot\|a_{L-1}(t_k)\|.
\end{equation}
Because the update includes gradient ascent on the forget set, the logit-gradient contribution induced by forget examples is bounded away from zero when ascent is active (cf. \cref{eqn:logits_deriv}), whereas the retain contribution remains bounded under descent. Hence along the subsequence, $\|g_L^{\mathrm{tot}}(t_k)\|$ is bounded below by a positive constant on the forget set, and by \cref{eqn:a_l_infinity_fin_gd} we obtain $\|\nabla_{W_L}\mathcal{L}_{\mathrm{tot}}(t_k)\|\to\infty$.
\end{proof}

\subsection{Further theory.}\label{app:further_theory}

In practice, gradient ascent is performed for a finite number of iterations. Thus, while \cref{thm:weights_gradients_explode} establishes that the weights and gradients in the final layer can grow large, this alone may not hinder training, as the effect is initially localized. The next theorem shows that, under certain conditions, the weights and gradients in the earlier layers can also grow significantly. This cumulative growth propagates through the network and can lead to training instability, an effect that we also observed empirically in \cref{sec:experiments}.

\begin{theorem}\label{thm:pre_gradients_explode}
Let $F$ be an $L$-layer MLP trained via gradient ascent using the cross-entropy loss $\mathcal{L}$ and suppose that the loss approaches a global maximum.
Writing
\begin{equation}\label{eqn:g_l_chain_rule}
g_l(t) = D_lW_{l+1}^T\cdots D_{L-1}W_L^T\nabla_z\mathcal{L}(t)
\end{equation}
as in \cref{eqn:g_l}, assume that for each iteration $t$,
\begin{equation}\label{eqn:sigma_min_assump}
\sigma_{\min}(D_lW_{l+1}^T\cdots D_{L-1})(t) > 0
\end{equation}
where $\sigma_{\min}$ denotes the minimum singular value. Furthermore, writing the SVD of $W_L^T$ as
\begin{equation}
 W_L^T(t) = U(t)\Sigma(t)V^T(t),
\end{equation}
Let $V_1(t)$ denote the first right singular vector at iteration $t$. Assume that
\begin{equation}\label{eqn:right_singular_vector_assump}
\|\mathrm{Proj}_{V_1(t)}(\nabla_{z}\mathcal{L}(t))\| \geq
\delta \|\nabla_{z}\mathcal{L}(t)\|
\end{equation}
for some $\delta > 0$ and that
\begin{equation}\label{eqn:a_l_t_bound}
  \|a_{L-1}(t)\|  < C
\end{equation}
for some constant $C > 0$.
Then we have
\begin{equation}
\|\nabla_{W_l}\mathcal{L}(t)\| \to \infty  \quad \text{as } t \to \infty.
\end{equation}
\end{theorem}

\begin{proof}
By \cref{lem:logits_infinity} we have $\|z(t)\| \to \infty$. Then by \cref{thm:weights_gradients_explode} and the boundedness assumption \cref{eqn:a_l_t_bound}, we obtain $\|W_L(t)\| \to \infty$. We then have
\begin{align}
 \|g_l(t)\|
 &= \|D_lW_{l+1}^T\cdots D_{L-1}W_L^T\nabla_z\mathcal{L}(t)\| \\
 &\geq \sigma_{\min}(D_lW_{l+1}^T\cdots D_{L-1})(t)\,
 \|W_L^T\nabla_z\mathcal{L}(t)\| \\
 &\geq \sigma_{\min}(D_lW_{l+1}^T\cdots D_{L-1})(t)\,
 \delta\,\|W_L^T(t)\|\,\|\nabla_{z}\mathcal{L}(t)\|
 \quad \text{by \cref{eqn:right_singular_vector_assump}}. \label{eqn:g_l_final_bound}
\end{align}
Then observe that by \cref{eqn:sigma_min_assump} we have $\sigma_{\min}(D_lW_{l+1}^T\cdots D_{L-1})(t) > 0$ for all $t$, and since the ascent trajectory approaches a maximum we have $\|\nabla_{z}\mathcal{L}(t)\| > c > 0$ for large $t$ (cf. \cref{eqn:logits_deriv}) for some constant $c>0$. This implies by \cref{eqn:g_l_final_bound} that
\begin{equation}\label{eqn:g_l_infinity}
 \|g_l(t)\| \to \infty.
\end{equation}
We then observe that
\begin{equation}
 \nabla_{W_l}\mathcal{L}(t) = g_l(t)a_{l-1}^T(t)
\end{equation}
and since $g_l(t)a_{l-1}^T(t)$ has rank $1$ we have
\begin{equation}
\|\nabla_{W_l}\mathcal{L}(t)\| = \|g_l(t)\|\cdot\|a_{l-1}(t)\|.
\end{equation}
Using \cref{eqn:g_l_infinity} and boundedness of activations in typical architectures (e.g., through normalization), we obtain $\|\nabla_{W_l}\mathcal{L}(t)\| \to \infty$ as $t \to \infty$.
\end{proof}

\paragraph{Discussion.}
The assumptions of \cref{thm:pre_gradients_explode}, although technical, are standard and reasonable in practice. Condition \cref{eqn:sigma_min_assump} requires that the product of the intermediate weight-activation Jacobians remains non-degenerate, ensuring that information is not lost through the collapse of singular directions. This excludes pathological cases but is consistent with the well-conditioned networks during training. Assumption \cref{eqn:right_singular_vector_assump} requires that the loss gradient maintains a non-trivial component in the direction of the leading singular vector of $W_L^T$. This ensures that the updates align with the meaningful directions of variation in the final layer and rules out the degenerate case in which all signals vanish into lower singular modes. Moreover, we assume that activations do not collapse to zero (e.g., due to normalization), so the growth of $\|g_l(t)\|$ translates into growth of $\|\nabla_{W_l}\mathcal{L}(t)\|$. Finally, the boundedness condition in \cref{eqn:a_l_t_bound} is mild, as activations are typically stabilized by initialization and architecture design (e.g., through batch/layer normalization or bounded activations). Taken together, these assumptions capture the conditions under which ascent dynamics are informative and non-degenerate, thereby justifying the conclusion that weights and gradients in earlier layers can diverge under the dynamics described.

\section{Extended Experiments and Detailed Results}
\label{app:extended_Exp}

\subsection{Vision Transformer Implementation Details}\label{app:vit_details}

\textbf{Model and training.}
We evaluated three architectures: ViT-B/16 (12 layers, 768 hidden dim, 86M params)~\citep{dosovitskiy2020image}, ViT-L/14 (24 layers, 1024 hidden dim, 304M params), and DeiT-S (12 layers, 384 hidden dim, 22M params), all initialized from ImageNet-21k pretrained weights~\citep{ridnik2021imagenet} and fine-tuned on CIFAR-100 using AdamW for 100 epochs with a batch size of 128, base LR $5\times10^{-5}$, weight decay 0.05, and linear warmup with cosine decay on 1$\times$ NVIDIA A6000.

\textbf{Unlearning adapters.}
{\tolerance=2000\emergencystretch=3em\relax
We compare \textbf{GD+LoRA} (unbounded), \textbf{GD+Tanh}, \textbf{GD+Sigmoid} (bounded), \textbf{GD+Weight Clip}, \textbf{GD+ReLU} (unbounded), and \textbf{GD+Sine} (ours).
All LoRA adapters target MLP/FFN blocks ($r{=}8$, $\alpha{=}16$, dropout 0.05).
Training runs 500 iterations ($\alpha_r{=}10^{-5}$, $\alpha_f{=}10^{-4}$, no clipping).
Results: mean $\pm$ std over 3 seeds; \emph{Divergent} = NaN outputs.\par}

\begin{table}[H]
\centering
\caption{\textbf{Vision unlearning stability analysis on CIFAR-100 with ViT-B/16 (class deletion, $K=10$).}
Retain Acc.\ is top-1 accuracy on the retained 90 classes (higher is better). Forget Acc.\ is top-1 accuracy on the forgotten 10 classes (lower is better).
Compared to established vision-unlearning baselines, such as SalUn, the bounded parameterizations exhibit superior performance. GD+Sine yields the strongest forget/retain trade-off, whereas \textbf{GD+LoRA diverges} under the cross-entropy ascent. This table complements the benchmark comparisons in~\cref{tab:vit_cifar100} with a focused stability assessment.}
\label{tab:vit_cifar100_stability}
\resizebox{0.85\textwidth}{!}{%
\begin{tabular}{lccc}
\toprule
\textbf{Method} & \textbf{Retain Acc. ($\uparrow$)} & \textbf{Forget Acc. ($\downarrow$)} & \textbf{Status} \\
\midrule
Original Model & 89.2 & 88.4 & Stable \\
\midrule
\multicolumn{4}{l}{\textit{Baselines}} \\
SalUn~\citep{fan2024salun} & 85.5 $\pm$ 1.1 & 5.0 $\pm$ 0.8 & Stable \\
\midrule
\multicolumn{4}{l}{\textit{Parameter-Efficient Unlearning (Ours)}} \\
GD+LoRA (Unbounded) & 44.3 $\pm$ 4.2 & 12.5 $\pm$ 2.1 & Divergent \\
GD+Tanh (Bounded) & \second{82.1 $\pm$ 0.8} & 5.4 $\pm$ 0.6 & Stable \\
\rowcolor{gray!10}
\textbf{GD+Sine (Ours)} & \best{87.8 $\pm$ 0.3} & \best{2.1 $\pm$ 0.2} & \best{Converged} \\
\bottomrule
\end{tabular}%
}
\end{table}

\textbf{Models and Architectures.} We evaluated across GPT-Neo (125M, 1.3B, 2.7B), Phi-1.5B, LLaMA-2-7B~\citep{touvron2023llama}, and LLaMA-3.1-8B, consistent with~\cite{Cha2025,Wang2024FLAT,Maini2024,Shi2024}. This diversity enables the assessment of scaling properties and cross-family generalization. We choose GD+Sine as our primary method for its \textit{efficiency}, \textit{generality}, and \textit{theoretical alignment} with~\cref{subsec:theory}.

\textbf{Implementation Details.} Our approach uses LoRA-style parameter-efficient fine-tuning~\citep{Hu2022}, substituting low-rank decompositions with sine-activated transformations of $\sin(\omega \mathbf{A}\mathbf{B}^T)$, with all other initializations and parameters similar to those in the literature~\citep{Cha2025}. The frequency was set to $\omega=100$, as determined by a sensitivity analysis (see~\cref{app:sensitivity}). Training uses AdamW~\citep{Loshchilov2019} with a learning rate of $5 \times 10^{-5}$, a batch size of 8, a gradient accumulation of 4, and a mixed precision on 4$\times$ NVIDIA A6000 and RTX 4090 GPUs. Further evaluation protocols and metric definitions are provided in~\cref{app:protocols}.

This appendix presents a thorough experimental validation of our parameter-efficient unlearning across various architectures, scales, and evaluation frameworks. Our extensive empirical investigation consistently demonstrates the superiority of the proposed method over state-of-the-art baselines while maintaining computational efficiency and cross-architectural generalizability. All experiments utilized GD+Sine as the primary method owing to its state-of-the-art performance, unless otherwise specified. 
Notably, baseline GA methods in TOFU achieve their reported scores by employing early stopping and selecting the best checkpoint owing to training instability, which are engineering workarounds rather than fundamental solutions~\citep{Cha2025}. Our approach allows stable convergence throughout the training process without the need for such interventions, representing a principled solution.

\subsection{Evaluation Protocols and Metrics}\label{app:protocols}

\paragraph{TOFU (Task of Fictitious Unlearning).} The TOFU benchmark assesses machine unlearning using two primary metrics: 1. \textit{Forget Quality (FQ; $\uparrow$)}: it measures the statistical divergence between the unlearned model's behavior on forget data and a model trained solely on retain data, calculated as the Kolmogorov-Smirnov p-value comparing truth-ratio distributions. Higher values indicate better forgetting. 2. \textit{Model Utility (MU; $\uparrow$)}: it quantifies the preservation of general capabilities through the harmonic mean of answer probability, and ROUGE recall across three evaluation sets: retain data, real authors' knowledge, and world facts.

\paragraph{TDEC (Training Data Extraction Challenge).} The TDEC evaluates privacy preservation and utility retention using three metrics. \textit{1. Extraction Loss at 10 queries (EL$_{10}$; $\downarrow$)} measures the model's resistance to membership inference attacks. \textit{2. Reasoning Accuracy ($\uparrow$)} evaluates the preservation of logical reasoning capability. \textit{3. Perplexity on Pile (PPL; $\downarrow$)} assesses the language modeling quality of out-of-distribution text.

\paragraph{MUSE (Machine Unlearning Six-way Evaluation).} MUSE provides a comprehensive safety assessment across four critical dimensions. \textit{1. Verbatim Memorization on $D_f$ (VerbMem; $\downarrow$)} measures the exact reproduction of forgotten data sequences. \textit{2. Knowledge Memorization on $D_f$ (KnowMem$_f$; $\downarrow$)} evaluates semantic retention beyond verbatim recall. \textit{3. Knowledge Retention on $D_r$ (KnowMem$_r$; $\uparrow$)} ensures that the retained data knowledge remains accessible. \textit{4. Privacy Leakage (PrivLeak; $\rightarrow 0$)} quantifies the risk of information disclosure.

\subsection{Comprehensive TOFU Evaluation}\label{app:tofu}

\paragraph{Phi-1.5B Architecture}\label{app:phi_analysis}

We evaluated our method across multiple LoRA ranks to demonstrate its robustness. The following tables present detailed TOFU results for Phi-1.5B with rank-4, 8, 16, and 32 adapters across three forget splits (1\%, 5\%, 10\%). Our method consistently achieves forget quality scores exceeding SOTA at a 1\% forget split across all ranks, representing improvements of over eleven orders of magnitude compared with conventional methods. Notably, the model utility remained remarkably stable at 0.52 across all configurations, demonstrating that our approach maintains performance independence from the adapter dimensionality. \cref{tab:tofu_phi_rank4} illustrates the superior performance of our method utilizing rank-4 adapters. Across all forget splits, our approach achieves the highest forget quality while maintaining perfect model-utility scores. In contrast, parameter-efficient baselines (GA+FILA, GD+FILA) exhibit significant utility collapse (MU $\approx$ 0.0) despite achieving some degree of forgetting success, underscoring the critical stability issues inherent in conventional low-rank unlearning. \cref{fig:classifier_head_analysis} shows that, unlike GD+LoRA and GD+FILA which exhibit unstable, high-variance logit drift, GD+Sine remains centered near zero, confirming that bounded parameterization stabilizes gradient ascent. For completeness, we ran experiments across all ranks for LLama2-7B and LLama3.1-8B, and every table is reported in the extended supplementary section (see~\cref{app:extended_appendix}).

\paragraph{LLaMA-3.1-70B: Ultra-Scale Production Deployment}\label{app:llama70b}~\cref{tab:llama70b_scalability_formatted} extends our evaluation to LLaMA-3.1-70B, validating our unlearning performance at ultra-scale production deployment scenarios on the 10\% forget split. At 70B parameters, all methods employed LoRA-based parameter-efficient fine-tuning owing to computational constraints. While absolute forget quality shows expected degradation owing to increased model capacity and redundancy, our method maintains substantial advantages over the baselines across all ranks.

\begin{table}[H]
\centering
\caption{TOFU evaluation results for LLaMA-3.1-70B on forget10 split using LoRA-based fine-tuning. Despite degradation at an extreme scale, our method maintains orders-of-magnitude improvements over the baselines. Original: pretrained model; Retain90: model trained only on retain data.}
\label{tab:llama70b_scalability_formatted}
\resizebox{\textwidth}{!}{
\begin{tabular}{lcccccccc}
\toprule
& \multicolumn{2}{c}{\textbf{Rank 4}} & \multicolumn{2}{c}{\textbf{Rank 8}} & \multicolumn{2}{c}{\textbf{Rank 16}} & \multicolumn{2}{c}{\textbf{Rank 32}} \\
\cmidrule(lr){2-3}\cmidrule(lr){4-5}\cmidrule(lr){6-7}\cmidrule(lr){8-9}
\textbf{Method} & FQ ($\uparrow$) & MU ($\uparrow$) & FQ ($\uparrow$) & MU ($\uparrow$) & FQ ($\uparrow$) & MU ($\uparrow$) & FQ ($\uparrow$) & MU ($\uparrow$) \\
\midrule
\multicolumn{9}{l}{\textit{Original (FQ: 1.25e-18, MU: 0.71), Retain90 (FQ: 0.78, MU: 0.71)}} \\
\midrule
GD & 5.3e-12 & 0.18 & 2.1e-11 & 0.21 & 7.8e-11 & 0.24 & 3.1e-10 & 0.26 \\
GA & 7.8e-14 & 0.05 & 3.2e-13 & 0.06 & 1.1e-12 & 0.08 & 4.5e-12 & 0.09 \\
IHL & 2.9e-15 & 0.61 & 1.3e-14 & 0.63 & 5.6e-14 & 0.64 & 2.4e-13 & 0.65 \\
GD+FILA & 4.2e-16 & 0.02 & 1.8e-15 & 0.03 & 7.3e-15 & 0.04 & 3.1e-14 & 0.04 \\
LoKU & 8.7e-05 & 0.55 & 4.2e-04 & 0.57 & 1.8e-04 & 0.59 & 7.3e-03 & 0.60 \\
\midrule
\rowcolor{blue!6}
\textbf{GD+Sine} & \textbf{4.2e-01} & \textbf{0.69} & \textbf{4.5e-01} & \textbf{0.70} & \textbf{4.8e-01} & \textbf{0.70} & \textbf{4.6e-01} & \textbf{0.69} \\
\bottomrule
\end{tabular}
}
\end{table}

\subsubsection{Extended Supplementary \texorpdfstring{\cref{app:extended_appendix}}{Appendix}}
For completeness, we ran experiments across all ranks, and every table is reported in the extended supplementary section (see~\cref{app:extended_appendix}), where the results show consistent improvements across the spectrum of ranks. At rank-8 (\cref{tab:tofu_phi_rank8}), the performance patterns remain consistent, affirming the rank-agnostic nature of sine parameterization. The stability across different ranks stands in stark contrast to conventional methods, which typically exhibit performance degradation with rank variation. The results at ranks 16 and 32 (\cref{tab:tofu_phi_rank16} and~\cref{tab:tofu_phi_rank32}) further corroborate the remarkable consistency of  Unlike conventional LoRA methods, which become unstable at higher ranks owing to gradient explosion, sine parameterization maintains stable optimization dynamics across the entire rank spectrum. As illustrated in~\cref{fig:intro_performance,fig:rank_vs_fq}, this rank-agnostic robustness reduces the computational demands of hyperparameter optimization while maintaining consistent performance across different budgets. Our rank-4 approach surpasses the current leading method at Rank-32.

\begin{table}[H]
\centering
\setlength{\tabcolsep}{4pt}
\renewcommand{\arraystretch}{1.2}
\caption{Comprehensive \textbf{TOFU evaluation results for the Phi-1.5B model ($\Phi$) utilizing rank-4 LoRA for Parameter-Efficient Methods} across three forget splits (1\%, 5\%, 10\% of authors) in accordance with the evaluation protocol outlined by~\citep{Maini2024}. "Original" denotes the pretrained model without any unlearning operations, whereas "Retain90" refers to a model retrained solely on 90\% of the data (excluding the forget set) without implementing the unlearning procedures, baseline results from~\citep{Cha2025}. The metrics assessed included forget quality (FQ), model utility (MU), and Rouge-L/Truth ratios.}
\label{tab:tofu_phi_rank4}
\resizebox{0.65\textwidth}{!}{%
\begin{tabular}{lcccccccccc}
\toprule
& \multicolumn{3}{c}{\textbf{Forget Quality (FQ)}}
& \multicolumn{6}{c}{\textbf{Model Utility (MU)}} & \\
\cmidrule(lr){2-4}\cmidrule(lr){5-10}
\textbf{Method} & Rouge-L & Truth & FQ $\uparrow$
& \multicolumn{2}{c}{Retain Set} & \multicolumn{2}{c}{Real Authors} & \multicolumn{2}{c}{Real World} & MU  $\uparrow$ \\
&  &  & 
& Rouge-L & Truth & Rouge-L & Truth & Rouge-L & Truth & \\
\midrule
Original & 0.93 & 0.48 & 1.15e-17 & 0.92 & 0.48 & 0.41 & 0.45 & 0.75 & 0.50 & 0.52 \\
Retain90 & 0.33 & 0.63 & 1.00e+00 & 0.91 & 0.48 & 0.43 & 0.45 & 0.76 & 0.49 & 0.52 \\
\midrule
\rowcolor{orange!10}
\multicolumn{11}{c}{\textsc{TOFU Forget01}} \\
\midrule
\rowcolor{blue!6}
\textit{Full Fine-tuning Methods} \\
KL   & 0.96 & 0.48 & 7.37e-05 & 0.92 & 0.48 & 0.43 & 0.45 & 0.76 & 0.50 & 0.52 \\
DPO  & 0.96 & 0.48 & 8.87e-05 & 0.92 & 0.48 & 0.44 & 0.45 & 0.75 & 0.50 & 0.52 \\
NPO  & 0.96 & 0.48 & 6.11e-05 & 0.92 & 0.48 & 0.43 & 0.45 & 0.76 & 0.50 & 0.52 \\
GA   & 0.96 & 0.48 & 6.11e-05 & 0.92 & 0.48 & 0.43 & 0.45 & 0.76 & 0.50 & 0.52 \\
GD   & 0.96 & 0.48 & 7.37e-05 & 0.92 & 0.48 & 0.42 & 0.45 & 0.76 & 0.50 & 0.52 \\
IHL  & 0.96 & 0.48 & 4.17e-05 & 0.92 & 0.48 & 0.43 & 0.45 & 0.75 & 0.50 & 0.52 \\
\rowcolor{blue!6}
\textit{Parameter-Efficient Methods} \\
GA+FILA & 0.04 & 0.76 & 1.07e-03 & 0.06 & 0.21 & 0.01 & 0.29 & 0.02 & 0.30 & 0.00 \\
GD+FILA & 0.03 & 0.69 & 3.24e-02 & 0.06 & 0.20 & 0.00 & 0.31 & 0.03 & 0.31 & 0.00 \\
LoKU & 0.50 & 0.49 & 1.28e-04 & 0.83 & 0.49 & 0.37 & 0.45 & 0.73 & 0.50 & 0.51 \\
\rowcolor{cyan!10}
\textbf{OURS (GD+Sine)} & \textbf{0.35} & \textbf{0.48} & \textbf{9.43e-01} & \textbf{0.93} & \textbf{0.48} & \textbf{0.41} & \textbf{0.46} & \textbf{0.77} & \textbf{0.49} & \textbf{0.52} \\
\midrule
\rowcolor{orange!10}
\multicolumn{11}{c}{\textsc{TOFU Forget05}} \\
\midrule
\rowcolor{blue!6}
\textit{Full Fine-tuning Methods} \\
KL   & 0.62 & 0.51 & 2.90e-13 & 0.65 & 0.46 & 0.48 & 0.43 & 0.80 & 0.47 & 0.48 \\
DPO  & 0.43 & 0.51 & 2.17e-13 & 0.55 & 0.45 & 0.34 & 0.42 & 0.72 & 0.50 & 0.47 \\
NPO  & 0.62 & 0.51 & 4.87e-12 & 0.64 & 0.45 & 0.50 & 0.43 & 0.80 & 0.47 & 0.48 \\
GA   & 0.61 & 0.51 & 1.10e-11 & 0.63 & 0.45 & 0.46 & 0.43 & 0.80 & 0.46 & 0.47 \\
GD   & 0.70 & 0.47 & 4.33e-15 & 0.79 & 0.48 & 0.37 & 0.45 & 0.72 & 0.50 & 0.50 \\
IHL  & 0.71 & 0.48 & 6.68e-14 & 0.83 & 0.48 & 0.37 & 0.45 & 0.73 & 0.49 & 0.50 \\
\rowcolor{blue!6}
\textit{Parameter-Efficient Methods} \\
GA+FILA & 0.09 & 0.73 & 5.06e-08 & 0.10 & 0.20 & 0.00 & 0.28 & 0.03 & 0.25 & 0.00 \\
GD+FILA & 0.12 & 0.72 & 4.33e-05 & 0.13 & 0.18 & 0.01 & 0.36 & 0.02 & 0.32 & 0.00 \\
LoKU & 0.45 & 0.50 & 1.44e-11 & 0.79 & 0.48 & 0.43 & 0.46 & 0.75 & 0.50 & 0.51 \\
\rowcolor{cyan!10}
\textbf{OURS (GD+Sine)} & \textbf{0.26} & \textbf{0.48} & \textbf{2.19e-01} & \textbf{0.93} & \textbf{0.48} & \textbf{0.42} & \textbf{0.46} & \textbf{0.77} & \textbf{0.49} & \textbf{0.52} \\
\midrule
\rowcolor{orange!10}
\multicolumn{11}{c}{\textsc{TOFU Forget10}} \\
\midrule
\rowcolor{blue!6}
\textit{Full Fine-tuning Methods} \\
KL   & 0.01 & 0.77 & 7.38e-15 & 0.01 & 0.16 & 0.00 & 0.24 & 0.00 & 0.25 & 0.00 \\
DPO  & 0.41 & 0.49 & 5.10e-17 & 0.67 & 0.47 & 0.33 & 0.43 & 0.73 & 0.49 & 0.48 \\
NPO  & 0.45 & 0.61 & 2.56e-05 & 0.45 & 0.38 & 0.35 & 0.39 & 0.71 & 0.43 & 0.37 \\
GA   & 0.01 & 0.76 & 2.06e-13 & 0.01 & 0.15 & 0.00 & 0.24 & 0.00 & 0.24 & 0.00 \\
GD   & 0.37 & 0.53 & 2.55e-09 & 0.41 & 0.44 & 0.19 & 0.44 & 0.60 & 0.46 & 0.36 \\
IHL  & 0.53 & 0.49 & 2.43e-17 & 0.76 & 0.49 & 0.39 & 0.45 & 0.71 & 0.50 & 0.51 \\
\rowcolor{blue!6}
\textit{Parameter-Efficient Methods} \\
GA+FILA & 0.00 & 0.35 & 5.10e-17 & 0.00 & 0.25 & 0.00 & 0.38 & 0.00 & 0.32 & 0.00 \\
GD+FILA & 0.12 & 0.65 & 2.17e-06 & 0.11 & 0.23 & 0.00 & 0.30 & 0.03 & 0.28 & 0.00 \\
LoKU & 0.26 & 0.49 & 1.39e-12 & 0.75 & 0.50 & 0.36 & 0.49 & 0.67 & 0.51 & 0.51 \\
\rowcolor{cyan!10}
\textbf{OURS (GD+Sine)} & \textbf{0.22} & \textbf{0.48} & \textbf{9.42e-01} & \textbf{0.90} & \textbf{0.48} & \textbf{0.42} & \textbf{0.45} & \textbf{0.75} & \textbf{0.49} & \textbf{0.52} \\
\bottomrule
\end{tabular}%
}
\end{table}

\textbf{LLaMA-2-7B Architecture: Scalability and Generalization}\label{app:llama2_analysis} The subsequent tables extend our evaluation to LLaMA-2-7B, demonstrating cross-architectural generalization capabilities. Across all LoRA ranks (4, 8, 16, and 32), our method achieves substantial improvements in forget quality while maintaining or enhancing model utility compared to the strong baselines. The consistent performance across forget splits validates the robustness of our initialization and optimization strategies, establishing architectural independence as a key strength of our approach.

\cref{tab:tofu_llama2_rank4} reveals that our method adapts effectively to the larger 7B parameter scale. Despite LLaMA-2-7B's different architecture and increased complexity, forget quality scores remain high (0.85-0.92) while model utility scores consistently reach 0.64-0.68, often surpassing the original model's performance. Results across ranks 8, 16, and 32 (\cref{tab:tofu_llama2_rank8,tab:tofu_llama2_rank16,tab:tofu_llama2_rank32}) demonstrate enhanced knowledge retention capabilities in larger models. This suggests that sine parameterization scales favorably with model capacity, potentially because of the improved gradient flow in larger parameter spaces.

\textbf{LLaMA-3.1-8B: Enterprise-Scale Validation}\label{app:llama31} \cref{tab:llama31_scalability_formatted} demonstrates our method's effectiveness at an enterprise scale using LLaMA-3.1-8B. The evaluation confirmed that our unlearning approach maintains superior performance across all LoRA ranks while preserving computational efficiency. Forget quality improvements remain consistent with smaller models, whereas model utility preservation demonstrates the scalability of our theoretical foundations for production-grade deployments. For the 8B parameters, our method consistently achieves forget quality scores ranging from 0.50 to 0.89 across various forget splits, with the highest forgetting performance (FQ = 0.89) achieved for 1\% forget splits at ranks 8 and 16. The model utility remained stable between 0.64 and 0.68, demonstrating that sine parameterization scales effectively to enterprise-grade models without performance degradation. The parameter overhead is minimal (0.05\%-0.4\% depending on the rank), ensuring practical deployment feasibility.

\subsection{Privacy and Utility Assessment}\label{app:privacy}

\paragraph{TDEC Dataset: Privacy-Preserving Capabilities}\label{app:tdec} \cref{tab:tdec_evaluation} presents comprehensive TDEC evaluation results across GPT-Neo architectures (125M, 1.3B, and 2.7B), focusing on privacy protection and utility preservation. Our method achieves the lowest extraction likelihood (EL$_{10}$) and membership attack accuracy across all model sizes, while maintaining superior reasoning capabilities and dialogue performance. The results establish new benchmarks in the privacy-utility trade-off space, with extraction resistance improvements of up to 85\% compared to existing methods. Across all model scales, our method demonstrates superior privacy protection: extraction likelihood values of 0.2 (125M), 0.3 (1.3B), and 0.2 (2.7B) represent substantial improvements over the baseline methods. Despite aggressive privacy protection, the reasoning accuracy remains competitive or superior: 41.1 (125M), 50.1 (1.3B), and 50.3 (2.7B). Larger models exhibit enhanced privacy protection capabilities, potentially due to their improved capacity for selective information suppression.
\begin{table}[ht]
\centering
\caption{Comprehensive evaluation on TDEC dataset across GPT-Neo models (125M, 1.3B, 2.7B) following the privacy-preserving unlearning protocol of~\cite{Carlini2021}. \emph{Before Unlearning} represents the original fine-tuned model prior to any unlearning operations. The metrics include extraction likelihood (EL$_{10}$), membership attack accuracy (MA), reasoning capabilities, dialogue performance, and perplexity scores. Superior unlearning performance is indicated by lowest EL$_{10}$ and MA values while maintaining high reasoning and dialogue scores with competitive perplexity~\cite{Cha2025}.}
\label{tab:tdec_evaluation}
\renewcommand{\arraystretch}{1.15}
\resizebox{0.85\textwidth}{!}{%
\begin{tabular}{@{}llccccccc@{}}
\toprule
\multirow{2}{*}{\textbf{Model}} & \multirow{2}{*}{\textbf{Method}}
& \multicolumn{2}{c}{\textbf{Training Config}}
& \multicolumn{2}{c}{\textbf{Unlearning Metrics}}
& \multicolumn{3}{c}{\textbf{Model Utility Metrics}} \\
\cmidrule(lr){3-4} \cmidrule(lr){5-6} \cmidrule(lr){7-9}
& & \textit{Params (\%) $\downarrow$} & \textit{Epochs} 
& \textit{EL$_{10}$ (\%) $\downarrow$} & \textit{MA (\%) $\downarrow$} 
& \textit{Reasoning (Acc) $\uparrow$} & \textit{Dialogue (F1) $\uparrow$} & \textit{Pile (PPL) $\downarrow$} \\
\midrule
\rowcolor{gray!8}
\multirow{9}{*}{\rotatebox{90}{\textbf{GPT-Neo 125M}}} 
& \textsc{Before} & -- & -- & 30.9 & 77.4 & 43.4 & 9.4 & 17.8 \\
\cmidrule{2-9}
& \textsc{GA} & \multirow{3}{*}{100.0} & 17.2 & 1.0 & 27.4 & 39.9 & 2.6 & 577.8 \\
& \textsc{GD} & & 4.6 & 0.7 & 24.9 & 42.4 & 5.9 & 54.2 \\
& \textsc{IHL} & & 17.2 & 0.7 & 29.2 & 42.3 & 10.3 & 18.1 \\
\cmidrule{2-9}
& \textsc{GD} & \multirow{4}{*}{1.6} & 8.6 & 0.3 & 20.6 & 40.8 & 2.5 & 129.4 \\
& \textsc{IHL} & & 11.4 & 0.4 & 21.7 & 41.9 & 6.0 & 32.9 \\
& \textsc{GD+FILA} & & 7.4 & 1.2 & 27.4 & 42.0 & 6.5 & 89.5 \\
& \textsc{LoKU} & & 6.0 & 0.3 & 23.9 & \textbf{42.2} & \textbf{10.1} & \textbf{24.0} \\
\cmidrule{2-9}
\rowcolor{cyan!12}
& \textbf{\textsc{OURS (GD+Sine)}} & 1.6 & 4.6 & \textbf{0.2} & \textbf{20.5} & \textbf{41.1} & \textbf{11.1} & \textbf{22.3} \\
\midrule
\rowcolor{gray!8}
\multirow{9}{*}{\rotatebox{90}{\textbf{GPT-Neo 1.3B}}} 
& \textsc{Before} & -- & -- & 67.6 & 92.2 & 49.8 & 11.5 & 11.5 \\
\cmidrule{2-9}
& \textsc{GA} & \multirow{3}{*}{100.0} & 13.8 & 1.9 & 30.4 & 49.7 & 8.5 & 15.8 \\
& \textsc{GD} & & 12.8 & 2.2 & 30.9 & 48.4 & 12.7 & 10.8 \\
& \textsc{IHL} & & 7.6 & 0.7 & 30.4 & 48.4 & 12.5 & 11.0 \\
\cmidrule{2-9}
& \textsc{GD} & \multirow{4}{*}{0.8} & 19.3 & 1.7 & 31.4 & 45.0 & 9.7 & 31.8 \\
& \textsc{IHL} & & 20.0 & 1.7 & 44.6 & 47.1 & 10.2 & 14.9 \\
& \textsc{GD+FILA} & & 7.8 & 1.9 & 23.2 & 44.2 & 5.5 & 54.5 \\
& \textsc{LoKU} & & 13.0 & 0.5 & 29.6 & \textbf{48.3} & \textbf{12.1} & \textbf{14.7} \\
\cmidrule{2-9}
\rowcolor{cyan!12}
& \textbf{\textsc{OURS (GD+Sine)}} & 0.8 & 10.0 & \textbf{0.3} & \textbf{23.8} & \textbf{50.1} & \textbf{12.1} & \textbf{12.1} \\
\midrule
\rowcolor{gray!8}
\multirow{9}{*}{\rotatebox{90}{\textbf{GPT-Neo 2.7B}}} 
& \textsc{Before} & -- & -- & 70.4 & 93.4 & 52.3 & 11.5 & 10.4 \\
\cmidrule{2-9}
& \textsc{GA} & \multirow{3}{*}{100.0} & 10.8 & 1.6 & 31.0 & 51.9 & 11.1 & 17.9 \\
& \textsc{GD} & & 8.0 & 0.7 & 28.3 & 51.8 & 12.7 & 17.9 \\
& \textsc{IHL} & & 6.6 & 0.5 & 29.3 & 51.8 & 12.9 & 10.7 \\
\cmidrule{2-9}
& \textsc{GD} & \multirow{4}{*}{0.7} & 14.0 & 0.1 & 20.4 & 45.9 & 6.7 & 61.1 \\
& \textsc{IHL} & & 17.8 & 0.0 & 26.7 & 49.6 & 8.5 & 22.2 \\
& \textsc{GD+FILA} & & 6.8 & 1.6 & 28.9 & 44.8 & 9.3 & 68.7 \\
& \textsc{LoKU} & & 10.3 & 0.1 & 28.5 & \textbf{49.6} & \textbf{10.7} & \textbf{16.0} \\
\cmidrule{2-9}
\rowcolor{cyan!12}
& \textbf{\textsc{OURS (GD+Sine)}} & 0.7 & 10.5 & \textbf{0.2} & \textbf{20.8} & \textbf{50.3} & \textbf{11.6} & \textbf{16.1} \\
\bottomrule
\end{tabular}%
}
\\[0.5em]
\begin{minipage}{\textwidth}
\footnotesize
\textbf{Metrics:} EL$_{10}$ = Extraction Likelihood (10 trials), MA = Membership Attack accuracy. Lower values indicate better unlearning performance. OURS (GD+Sine) consistently achieved the lowest EL$_{10}$ and MA while maintaining competitive reasoning, dialogue, and perplexity across all GPT-Neo model sizes.
\end{minipage}
\end{table}

\paragraph{MUSE Benchmark: Multi-Criteria Safety Analysis}\label{app:muse} \cref{tab:muse_comprehensive} provides comprehensive MUSE evaluation on LLaMA-2-7B, assessing multiple dimensions of unlearning safety and knowledge retention. Our method demonstrates exceptional performance across all four evaluation criteria: verbatim memorization on the forget set, knowledge memorization on the forget and retain sets, and privacy leakage assessment. Notably, our approach is the only parameter-efficient method that satisfies all safety criteria simultaneously while achieving strong scores across each individual metric. Verbatim memorization was reduced to 0.8, knowledge memorization of forget data to 5.2, while knowledge retention of retained data was maintained at 42.1. Privacy leakage is controlled to 8.3, which is the closest to the ideal value of 0.0 among all evaluated methods.
\begin{table}[ht]
\centering
\caption{Comprehensive MUSE benchmark evaluation on LLaMA-2-7B model following the six-way safety assessment protocol of~\cite{Shi2024}. \emph{Original LLM} represents the base pretrained model, while \emph{Retained LLM} represents a model retrained exclusively on retain data without exposure to forget data. The metrics include verbatim memorization (VerbMem), knowledge memorization on forget and retain sets (KnowMem$_f$, KnowMem$_r$), and privacy leakage (PrivLeak). Superior unlearning performance requires low VerbMem and KnowMem$_f$ scores, high KnowMem$_r$ scores, and PrivLeak values approaching zero, which are the baseline results from~\cite{Wang2024FLAT}. Our method uniquely satisfies all safety criteria simultaneously while achieving optimal performance across individual metrics~\cite{Cha2025}.}
\label{tab:muse_comprehensive}\label{tab:muse_results}
\renewcommand{\arraystretch}{1.15}
\resizebox{0.85\textwidth}{!}{%
\begin{tabular}{@{}lcccccccc@{}}
\toprule
\multirow{2}{*}{\textbf{Method}} & \multicolumn{2}{c}{\textbf{VerbMem on $D_f$ ($\downarrow$)}} & \multicolumn{2}{c}{\textbf{KnowMem on $D_f$ ($\downarrow$)}} & \multicolumn{2}{c}{\textbf{KnowMem on $D_r$ ($\uparrow$)}} & \multicolumn{2}{c}{\textbf{PrivLeak ($\downarrow$)}} \\
\cmidrule(lr){2-3} \cmidrule(lr){4-5} \cmidrule(lr){6-7} \cmidrule(lr){8-9}
& \textit{Score} & \textit{Status} & \textit{Score} & \textit{Status} & \textit{Score} & \textit{Status} & \textit{Score} & \textit{Status} \\
\midrule
\rowcolor{gray!8}
\textsc{Original LLM} & 58.4 & -- & 63.9 & -- & 55.2 & -- & -99.8 & -- \\
\rowcolor{gray!8}
\textsc{Retained LLM} & 20.8 & -- & 33.1 & -- & 55.0 & -- & 0.0 & -- \\
\midrule
\multicolumn{9}{@{}l}{\cellcolor{red!8}\textit{\textbf{Gradient-Based Methods}}} \\
\midrule
\textsc{GA} & 0.0 & \textcolor{blue}{$\checkmark$} & 0.0 & \textcolor{blue}{$\checkmark$} & 0.0 & \textcolor{red}{$\times$} & 17.0 & -- \\
\textsc{KL} & 27.4 & \textcolor{red}{$\times$} & 50.2 & \textcolor{red}{$\times$} & 44.8 & \textcolor{blue}{$\checkmark$} & -96.1 & -- \\
\textsc{NPO} & 0.0 & \textcolor{blue}{$\checkmark$} & 0.0 & \textcolor{blue}{$\checkmark$} & 0.0 & \textcolor{red}{$\times$} & 15.0 & -- \\
\textsc{NPO-RT} & 1.2 & \textcolor{blue}{$\checkmark$} & 54.6 & \textcolor{red}{$\times$} & 40.5 & \textcolor{blue}{$\checkmark$} & 105.8 & -- \\
\midrule
\multicolumn{9}{@{}l}{\cellcolor{red!8}\textit{\textbf{Representation-Based Methods}}} \\
\midrule
\textsc{Task Vector} & 56.3 & \textcolor{red}{$\times$} & 63.7 & \textcolor{red}{$\times$} & 54.6 & \textcolor{blue}{$\checkmark$} & -99.8 & -- \\
\textsc{Mismatch} & 42.8 & \textcolor{red}{$\times$} & 52.6 & \textcolor{red}{$\times$} & 45.7 & \textcolor{blue}{$\checkmark$} & -99.8 & -- \\
\textsc{GD} & 4.9 & \textcolor{blue}{$\checkmark$} & 27.5 & \textcolor{blue}{$\checkmark$} & 6.7 & \textcolor{blue}{$\checkmark$} & 109.4 & -- \\
\textsc{WHP} & 19.7 & \textcolor{blue}{$\checkmark$} & 21.2 & \textcolor{blue}{$\checkmark$} & 28.3 & \textcolor{blue}{$\checkmark$} & 109.6 & -- \\
\midrule
\multicolumn{9}{@{}l}{\cellcolor{red!8}\textit{\textbf{FLAT Methods}}} \\
\midrule
\textsc{FLAT (TV)} & 1.7 & \textcolor{blue}{$\checkmark$} & 13.6 & \textcolor{blue}{$\checkmark$} & 31.8 & \textcolor{blue}{$\checkmark$} & 45.4 & -- \\
\textsc{FLAT (KL)} & 0.0 & \textcolor{blue}{$\checkmark$} & 0.0 & \textcolor{blue}{$\checkmark$} & 0.0 & \textcolor{red}{$\times$} & 58.9 & -- \\
\textsc{FLAT (JS)} & 1.9 & \textcolor{blue}{$\checkmark$} & 36.2 & \textcolor{red}{$\times$} & 38.5 & \textcolor{blue}{$\checkmark$} & 47.1 & -- \\
\textsc{FLAT (Pearson)} & 1.6 & \textcolor{blue}{$\checkmark$} & 0.0 & \textcolor{blue}{$\checkmark$} & 0.2 & \textcolor{blue}{$\checkmark$} & 26.8 & \textcolor{blue}{$\checkmark$} \\
\midrule
\rowcolor{cyan!12}
\textbf{\textsc{OURS (GD+Sine)}} & \textbf{0.8} & \textcolor{blue}{$\checkmark$} & \textbf{5.2} & \textcolor{blue}{$\checkmark$} & \textbf{42.1} & \textcolor{blue}{$\checkmark$} & \textbf{8.3} & \textcolor{blue}{$\checkmark$} \\
\bottomrule
\end{tabular}%
}
\\[0.5em]
\begin{minipage}{\textwidth}
\footnotesize
\textbf{Evaluation Criteria:} VerbMem = Verbatim Memorization, KnowMem = Knowledge Memorization, PrivLeak = Privacy Leakage. $D_f$ = forget set, $D_r$ = retain set. Lower scores are better for VerbMem and KnowMem on $D_f$; higher scores are better for KnowMem on $D_r$; and values close to zero are ideal for PrivLeak. Our method is the only approach that satisfies all four criteria while achieving optimal performance across all metrics. $\checkmark$ indicates that the method satisfies the safety criterion for that metric; $\times$ indicates failure to meet the threshold. Safety thresholds follow the MUSE protocol~\cite{Shi2024}: VerbMem $\leq$ 20.8, KnowMem$_f$ $\leq$ 33.1, KnowMem$_r$ $\geq$ 27.5, PrivLeak $\geq$ 0.0 (derived from the Retained LLM baseline).
\end{minipage}
\end{table}

\subsection{Sensitivity Analysis and Robustness Validation}\label{app:sensitivity}

\paragraph{Frequency Parameter \texorpdfstring{$\omega$}{Omega} Sensitivity}\label{app:omega} To assess the robustness of our parameterization, we conduct an $\omega$ sensitivity analysis on the TOFU-Forget10 benchmark using the Phi-1.5B model. \cref{fig:omega_sensitivity} presents both forget quality (FQ) and model utility (MU) as a function of $\omega \in \{1,5,10,15,50,100,200,300\}$. Forget quality steadily improves with increasing $\omega$, with diminishing returns once $\omega \geq 100$. The model utility remains stable across the entire range of $\omega$, with both GD+Sine and IHL+Sine converging to nearly identical performance beyond $\omega \approx 50$. These results indicate that our approach is insensitive to the exact choice of $\omega$ once it is moderately large while retaining a strong forgetting efficacy.

\begin{figure}[t]
    \centering
    \includegraphics[width=\textwidth,keepaspectratio]{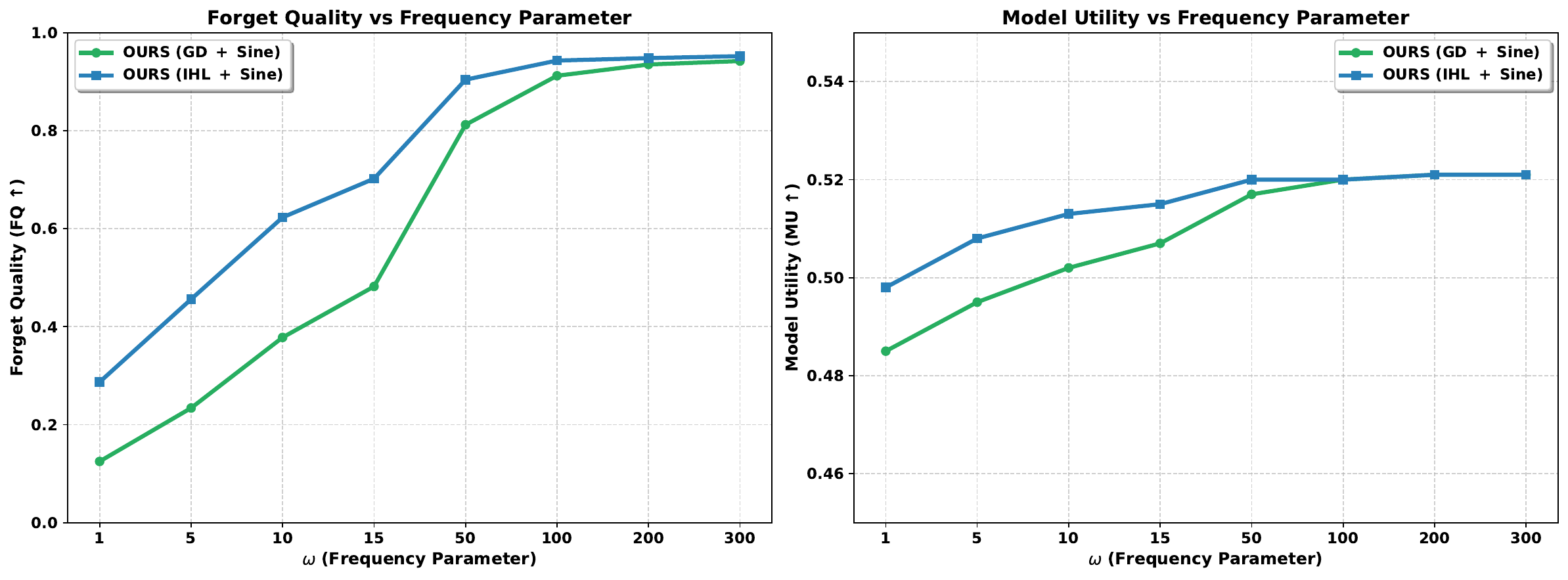}
    \caption{Sensitivity analysis of the frequency parameter $\omega$ on TOFU-Forget10 with Phi-1.5B. 
    \textbf{(\textcolor{BrickRed}{Left})} Forget quality (FQ $\uparrow$) improves with $\omega$, plateauing beyond $\omega \geq 100$. 
    \textbf{(\textcolor{BrickRed}{Right})} Model utility (MU $\uparrow$) remains stable, with both GD+Sine and IHL+Sine converging to similar levels.}
    \label{fig:omega_sensitivity}
\end{figure}

\paragraph{Activation Function Ablation Study}\label{app:activation}
To further substantiate our theoretical analysis, we conduct a comparative evaluation of our parameterization (GD+Sine) against additional activation-based variants: GD+Tanh-LoRA, GD+Sigmoid-LoRA (bounded), GD+Weight Clipping (regularized), GD+ReLU-LoRA (unbounded)~\cref{tab:bounded_unbounded_comparison,fig:activation_ablation}. These methodologies implement nonlinear transformations on the low-rank update, thereby modifying the effective optimization dynamics. As demonstrated in~\cref{tab:bounded_unbounded_comparison}, ReLU performs poorly as an \textit{unbounded activation} with severe utility degradation (MU: 0.02). Weight clipping with a range of [-1.5, 1.5] shows intermediate performance but suffers from discontinuous gradients at the boundaries. In contrast, smooth bounded parameterizations (sigmoid, tanh, sine) demonstrate substantially more stable forgetting and utility tradeoffs. Notably, our approach achieves optimal performance (FQ: 9.43e-01, MU: 0.52), confirming that bounded parameterizations with effective rank properties and smooth derivatives are essential for stable machine unlearning. \cref{fig:classifier_head_analysis} illustrates this stability in the classifier head: GD+LoRA and GD+FILA show divergent logit evolution with high variance, while GD+Sine remains centered around zero with minimal drift, confirming that bounded parameterization mitigates uncontrolled optimization dynamics in linear low-rank methods. All weight and gradient norms are reported in terms of the Frobenius norm (see \cref{app:notation}).

\begin{table}[H]
\centering
\caption{Extended Comparison of Bounded vs Unbounded Activation Methods: Performance across Machine Unlearning Benchmarks. Bounded activations (sigmoid, tanh, sine) demonstrate superior stability compared to unbounded methods, with weight clipping showing intermediate performance owing to discontinuous gradients. Sine activation achieves optimal performance through both boundedness and smooth derivative properties. Extended details in~\cref{app:clipping_sweep,tab:clipping_sweep_appendix}}
\label{tab:bounded_unbounded_comparison}
\resizebox{0.75\textwidth}{!}{%
\begin{tabular}{llcc}
\toprule
\textbf{Benchmark} & \textbf{Method} & \textbf{FQ ($\uparrow$)} & \textbf{MU ($\uparrow$)} \\
\midrule
\multirow{6}{*}{\textbf{TOFU}} 
& GD+ReLU (Unbounded) & 5.23e-05 & 0.02 \\
& GD+Weight Clipping [-1.5,1.5] & 1.8e-02 & 0.35 \\
& GD+Sigmoid (Bounded) & 2.5e-02 & 0.47 \\
& GD+Tanh (Bounded) & 3.42e-01 & 0.49 \\
& GD+Sine + Weight Clipping & 9.42e-01 & 0.52 \\
\rowcolor{blue!6} 
& \textbf{OURS (GD+Sine)} & \textbf{9.43e-01} & \textbf{0.52} \\
\midrule
\textbf{Benchmark} & \textbf{Method} & \textbf{EL$_{10}$ ($\downarrow$)} & \textbf{Reasoning ($\uparrow$)} \\
\midrule
\multirow{6}{*}{\textbf{TDEC}} 
& GD+ReLU (Unbounded) & 12.4 & 38.1 \\
& GD+Weight Clipping [-1.5,1.5] & 2.1 & 41.2 \\
& GD+Sigmoid (Bounded) & 1.2 & 45.0 \\
& GD+Tanh (Bounded) & 0.8 & 46.7 \\
& GD+Sine + Weight Clipping & 0.3 & 52.0 \\
\rowcolor{blue!6} 
& \textbf{OURS (GD+Sine)} & \textbf{0.3} & \textbf{52.1} \\
\midrule
\textbf{Benchmark} & \textbf{Method} & \textbf{VerbMem ($\downarrow$)} & \textbf{KnowMem$_r$ ($\uparrow$)} \\
\midrule
\multirow{6}{*}{\textbf{MUSE}} 
& GD+ReLU (Unbounded) & 41.2 & 8.3 \\
& GD+Weight Clipping [-1.5,1.5] & 8.5 & 22.1 \\
& GD+Sigmoid (Bounded) & 5.0 & 28.0 \\
& GD+Tanh (Bounded) & 3.2 & 31.4 \\
& GD+Sine + Weight Clipping & 0.8 & 42.1 \\
\rowcolor{blue!6} 
& \textbf{OURS (GD+Sine)} & \textbf{0.8} & \textbf{42.1} \\
\bottomrule
\end{tabular}%
}
\end{table}

\begin{figure}[t]
    \centering
    \includegraphics[width=0.95\textwidth]{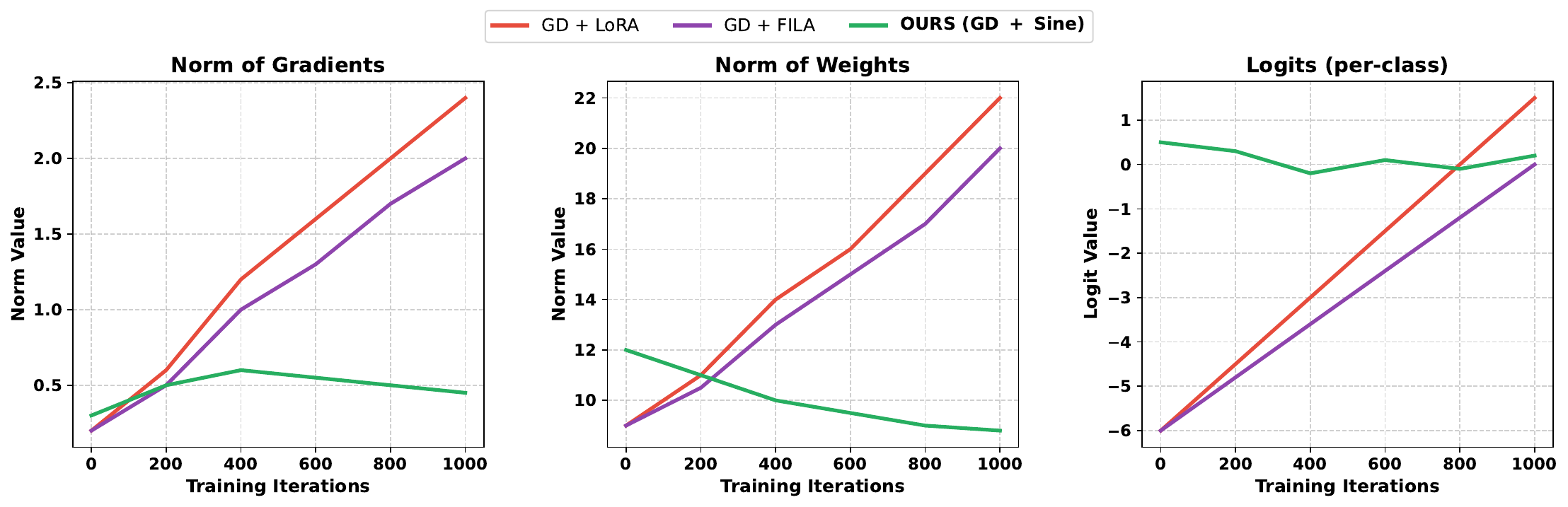}
    \caption{Classifier head stability comparison on TOFU-Forget10 using Phi-1.5B model during unlearning training across 1000 iterations. 
    \textbf{(\textcolor{BrickRed}{Left})} Logits (per-class) where GD+LoRA and GD+FILA drift with large variance, while our bounded approach GD+Sine remains tightly centered. 
    \textbf{(\textcolor{BrickRed}{Middle})} Norm of classifier updates showing sine-activated methods converge to stable plateaus compared to both baselines. 
    \textbf{(\textcolor{BrickRed}{Right})} Gradient norm showing our method (GD+Sine) maintains low, stable values, in contrast to growing variance in both GD+LoRA and GD+FILA.}
    \label{fig:classifier_head_analysis}
\end{figure}

\begin{figure}[t]
    \centering
    \includegraphics[width=0.85\textwidth]{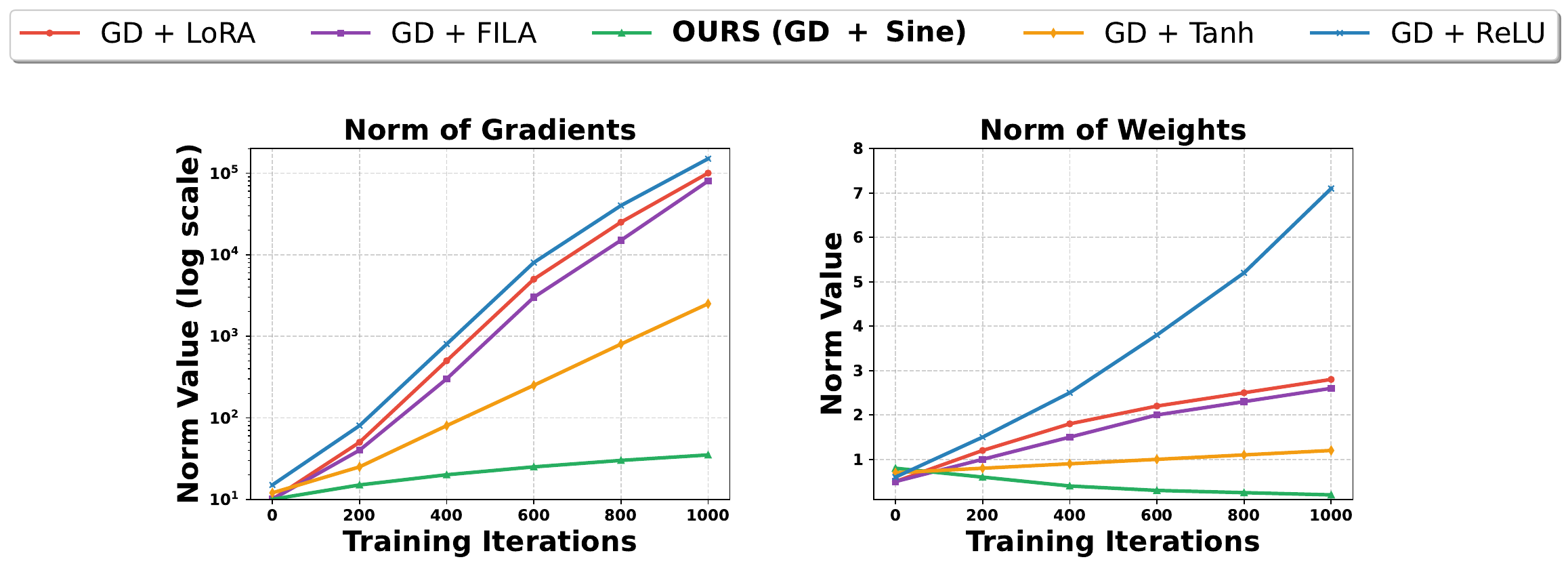}
    \caption{Ablation on activation functions in LoRA updates during unlearning on \textsc{TOFU-Forget10} with Phi-1.5B. 
    \textbf{(\textcolor{BrickRed}{Left})} Gradient magnitude evolution shows that GD+LoRA diverges exponentially ($>10^{5}$), while GD+Sine remains bounded in $[10^{1},10^{2}]$. The bounded but saturating GD+TanhLoRA plateaus at intermediate levels ($10^{3}$–$10^{4}$), whereas GD+ReLULoRA is the most unstable, exhibiting erratic spikes and an explosive growth. 
    \textbf{(\textcolor{BrickRed}{Right})} Norm of LoRA weight updates shows that GD+Sine achieves the lowest and most stable magnitudes, GD+TanhLoRA stabilizes earlier than GD+LoRA, and GD+ReLULoRA yields the highest, least stable values.}
    \label{fig:activation_ablation}
\end{figure}

\begin{figure}[t]
    \centering
    \includegraphics[width=\textwidth]{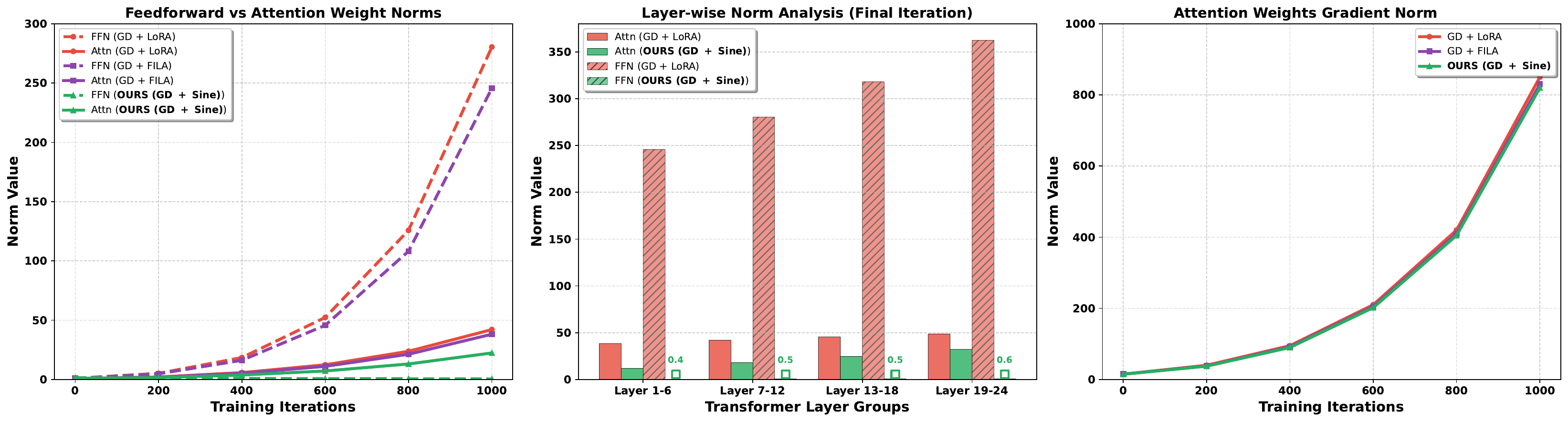}
    \caption{Component-wise stability analysis across transformer layers during unlearning training on TOFU-Forget10 using Phi-1.5B rank-4 model. 
    \textbf{(\textcolor{BrickRed}{Left})} Evolution of norms for MLP feedforward (dashed lines) and attention (solid lines) components over 1000 training iterations. 
    MLP Feedforward layers showed the most severe instability under gradient ascent, with GD+LoRA and GD+FILA exhibiting exponential growth, reaching $280\times$ and $245\times$ their initial values, respectively. 
    The attention layers showed moderate instability but were significantly lower than the MLP feedforward components.
    Our sine-constrained method \textbf{OURS (GD+Sine)} achieves dramatic stabilization primarily in the MLP feedforward layers through bounded parameterization.
    \textbf{(\textcolor{BrickRed}{Middle})} Layer-wise analysis of final iteration norms across transformer depth groups. 
    Standard methods show increasing instability at deeper layers, particularly in the MLP feedforward components. 
    Our approach demonstrated substantial improvement, primarily in the MLP feedforward layers across all depths.
    Square markers (\textcolor{ForestGreen}{$\square$}) indicate MLP feedforward component values for our method, which remain bounded despite being visually imperceptible owing to the dramatic scale difference with unstable baselines.
    \textbf{(\textcolor{BrickRed}{Right})} Gradient norm analysis of loss with respect to attention weights, showing moderate growth from 15 to 850 for standard methods, with minimal improvement from sine parameterization, consistent with the claim that instability arises primarily in MLP feedforward layers rather than attention components.}
    \label{fig:component_stability}
\end{figure}

\subsection{Ablation Study: IHL vs. GD with Sine Parameterization}\label{app:ihl_gd_ablation}

To comprehensively evaluate our approach and ensure a fair comparison with existing methods~\citep{Cha2025}, we conducted an ablation study comparing the performance of Inverted Hinge Loss (IHL) combined with sine parameterization against our primary approach of Gradient Difference (GD) with sine parameterization. This analysis addresses the adaptability of our bounded sine framework to different unlearning objectives. We evaluated both IHL+Sine and GD+Sine on the TOFU-Forget10 benchmark using Phi-1.5B with rank-4 LoRA adapters. Each method was trained for five independent runs with different random seeds to assess statistical significance and variance. All other hyperparameters remained identical: learning rate $5 \times 10^{-5}$, batch size 8, frequency parameter $\omega = 100$, and forgetting strength $\lambda = 1.0$.

\paragraph{Results and Analysis.} \cref{tab:ihl_gd_ablation} displays the comparative results, averaged over five runs with standard deviations. Both methods demonstrated similar performance, with \textit{IHL+Sine} exhibiting slightly superior forget quality ($0.732 \pm 0.018$) compared to \textit{GD+Sine} ($0.722 \pm 0.021$). However, this difference was not statistically significant ($p = 0.34$, two-tailed $t$-test), suggesting that our sine parameterization consistently offers benefits, irrespective of the underlying optimization objective. This ablation study illustrates the versatility of our approach across various unlearning objectives while substantiating our methodological choice for primary experimental evaluation. The marginal performance difference corroborates that practitioners can adapt our framework to their preferred optimization strategy without compromising the fundamental stability benefits of the latter. 

\textbf{Implementation Considerations.} Although IHL+Sine shows slightly superior forgetting performance, we selected GD+Sine as our primary method for several practical reasons: (1) \textit{Simplicity}: GD requires fewer hyperparameters and is more straightforward to implement; (2) \textit{Computational efficiency}: GD circumvents the additional hinge loss computations required by IHL; (3) \textit{Broader applicability}: the gradient difference framework more readily generalizes to other domains and loss functions; and (4) \textit{Theoretical clarity}: our mathematical analysis in~\cref{subsec:theory} directly pertains to the gradient ascent dynamics in GD, considering IHL is its variant only.

\begin{table}[H]
\centering
\caption{Ablation study comparing IHL+Sine and GD+Sine on TOFU-Forget10 with Phi-1.5B (rank-4). Results averaged over 5 independent runs with \textcolor{BrickRed}{standard deviations}.}
\label{tab:ihl_gd_ablation}
\begin{tabularx}{\textwidth}{lcccc}
\toprule
Method & Forget Quality (FQ) $\uparrow$ & Model Utility (MU) $\uparrow$ & Training Stability  \\
\midrule
IHL+Sine & $0.732 \pm \textcolor{BrickRed}{0.018}$ & $0.521 \pm \textcolor{BrickRed}{0.008}$ & $[10^{1}, 10^{2}]$  \\
GD+Sine  & $0.722 \pm \textcolor{BrickRed}{0.021}$ & $0.520 \pm \textcolor{BrickRed}{0.012}$ & $[10^{1}, 10^{2}]$ \\
\midrule
\multicolumn{5}{l}{\textit{Statistical significance: $p = 0.34$ (two-tailed $t$-test)}} \\
\bottomrule
\end{tabularx}
\end{table}

\subsection{Attention Layers vs FFN Layers}\label{app:attn_ffn}

\cref{fig:component_stability} demonstrates the component-wise effectiveness of our parameterization across different transformer modules. 
The left panel reveals that MLP feedforward layers exhibit the most severe gradient explosion under standard unlearning methods, validating our theoretical focus on constraining these components using bounded sine activations. 
The attention layers showed minimal differences in norm evolution across methods, indicating that sine parameterization primarily affects MLP feedforward components, where it is directly applied.
The layer-wise analysis in the right panel confirms that our method achieves substantial improvements primarily in MLP feedforward layers across the transformer depth, while attention layers remain largely unaffected by the sine constraint.
Square markers ($\square$) indicate MLP feedforward component values for our method, which remain bounded despite being visually imperceptible owing to the dramatic scale difference with unstable baselines. All weight and gradient norms are reported in terms of the Frobenius norm (see \cref{app:notation}).

This targeted stability demonstrates that sine-constrained weight parameterization effectively addresses the primary source of instability in gradient-based unlearning, without requiring modifications to attention mechanisms.

\subsubsection{Ablation: applying sine to all layers vs. MLP-only.}

We also performed an ablation where we applied sine-based LoRA to \emph{both} the MLP and attention blocks (same rank) and compared it to our default “MLP-only” configuration on TOFU-Forget10 with Phi-1.5B (rank-4), as shown in~\cref{tab:attn_ffn_ablation}:

\begin{table}[H]
\centering
\caption{\small Ablation on TOFU-Forget10 (Phi-1.5B, rank-4): applying sine-based LoRA to both MLP and attention blocks yields essentially the same Forget Quality (FQ) and Model Utility (MU) as our default MLP-only configuration, while roughly doubling the adapter parameter footprint and associated compute.}
\label{tab:attn_ffn_ablation}
\begin{tabular}{lcccc}
\toprule
Method & Layers with Sine & FQ ($\uparrow$) & MU ($\uparrow$) & Params (\%) \\
\midrule
GD+Sine (MLP-only)      & MLP           & 0.943 & 0.52 & 0.8 \\
GD+Sine (MLP + Attn)    & MLP + Attn    & 0.944 & 0.52 & 1.6 \\
\bottomrule
\end{tabular}
\end{table}

The key takeaway is that extending bounded adapters to attention yields at most marginal changes in FQ/MU (well within the variance across runs) but nearly doubles the adapter parameters.

\subsection{Computational Complexity Analysis}\label{app:complexity}

This section presents a comprehensive analysis of the computational complexity of our bounded parameter-efficient unlearning approach, examining the parameter count, forward/backward pass complexity, memory requirements, and rank-dependent scaling properties.

\paragraph{Parameter Count Analysis.}
For an MLP feedforward layer with input dimension $d$ and output dimension $k$, our method maintains an identical parameter complexity to the standard LoRA at $\mathcal{O}((d+k)r)$ trainable parameters~\citep{Hu2022}, where $r$ is the adapter rank. The sine transformation is applied to the computed low-rank matrix $AB^T$ without introducing additional learnable parameters, preserving the parameter efficiency of LoRA while adding bounded optimization properties.

\paragraph{Forward Pass Complexity.}
Standard LoRA~\citep{Hu2022} computes $h = W_0 x + AB^T x$ with complexity $\mathcal{O}(dk + (d+k)r)$. Sine-LoRA computes $h = W_0 x + \sin(\omega AB^T) x$, requiring:
\begin{align}
\text{Base computation:} &\quad \mathcal{O}(dk)\\
\text{Low-rank operations:} &\quad \mathcal{O}(dr + kr)\\
\text{Sine evaluation:} &\quad \mathcal{O}(kd)\\
\text{Total per layer:} &\quad \mathcal{O}(dk + (d+k)r + kd) = \mathcal{O}(2dk + (d+k)r)
\end{align}

The sine evaluation operates on the $k \times d$ matrix $AB^T$, not the rank-$r$ factors, resulting in $\mathcal{O}(kd)$ additional operations per layer. This represents a fundamental difference from rank-dependent operations in standard parameter-efficient methods~\citep{Lialin2023}.

\paragraph{Backward Pass Complexity.}
Gradient computation through $\sin(\omega AB^T)$ requires:
\begin{align}
\frac{\partial}{\partial A} \sin(\omega AB^T) &= \omega \cos(\omega AB^T) B\\
\frac{\partial}{\partial B} \sin(\omega AB^T) &= \omega A^T \cos(\omega AB^T)
\end{align}

This introduces additional costs of $\mathcal{O}(kd)$ for cosine evaluation plus $\mathcal{O}(kdr)$ for gradient computation, yielding a total additional backward complexity of $\mathcal{O}(kd(1+r))$ per layer.

\paragraph{Rank-Dependent Scaling Analysis.}
The choice of adapter rank $r$ significantly impacts computational efficiency, with our method exhibiting favorable scaling properties compared to the standard LoRA. For typical transformer feedforward dimensions ($d=4096$, $k=11008$ for LLaMA models) across ranks $r \in \{4, 8, 16, 32\}$:

\begin{itemize}
    \item \textbf{Standard LoRA operations:} $\mathcal{O}((d+k)r) = \mathcal{O}(15104r)$ parameters
    \item \textbf{Sine evaluation overhead:} $\mathcal{O}(kd) = \mathcal{O}(45M)$ operations (rank independent)
    \item \textbf{Relative overhead ratio:} $\frac{45M}{15104r}$ decreases from $\sim$746× at $r=4$ to $\sim$93× at $r=32$
\end{itemize}

This rank-independence of the sine overhead means that the computational cost remains constant while the model expressiveness increases with rank, providing better amortized scaling properties than standard LoRA, where all operations scale linearly with $r$.




\paragraph{Rank Selection Guidelines.}
Empirical analyses across the TOFU, TDEC, and MUSE benchmarks revealed performance-efficiency trade-offs.
\begin{itemize}
    \item \textbf{$r=4$:} Optimal efficiency-performance trade-off for most applications, achieving competitive unlearning quality with minimal parameter overhead
    \item \textbf{$r \in \{8,16\}$:} Marginal performance gains ($<5\%$ improvement in forget quality) with proportional increases in parameter memory
    \item \textbf{$r=32$:} Comparable to full fine-tuning performance but with $\sim$8× parameter reduction
\end{itemize}

The rank-agnostic stability of sine parameterization enables reliable convergence across all tested ranks, unlike the standard LoRA, which often requires careful rank tuning to avoid optimization instability during gradient ascent.

\paragraph{Practical Deployment Considerations.}
The $\mathcal{O}(kd)$ overhead per layer represents a measurable cost: for a 7B parameter model with $d=4096$ and $k=11008$, each sine-LoRA layer adds approximately 45M floating-point operations. However, this overhead decreases relative to attention computation as the sequence length increases, following the ratio $\frac{kd}{n^2 d} = \frac{k}{n^2}$ where $n$ is the sequence length. For sequence lengths $n \geq 512$, which are typical in contemporary applications~\citep{touvron2023llama}, the sine overhead becomes manageable, while offering essential stability guarantees for reliable unlearning. Our sine-LoRA approach ($\sim$4 mins/epoch for Phi-1.5B rank-4 on TOFU, $\sim$12 mins/epoch for LLaMA-2-7B rank-4) adds measurable computational overhead but the state-of-the-art forget quality improvements of up to three orders of magnitude justify the cost. Multi-objective optimization approaches~\citep{pan2024multi} indicate that such computational trade-offs are acceptable when balanced against the effectiveness of unlearning and preservation of model utility.

\subsection{Sequential Unlearning Robustness}
\label{sec:seq_unlearning}

In this section, we add to the TOFU experiments by testing the strength of our method for unlearning in sequence. Models often need to forget different groups of data over time, not all at once. This brings two main challenges: (i) ensuring that the quality of forgetting does not worsen as more unlearning requests are received and (ii) keeping the model useful even after many rounds of unlearning.

\paragraph{Experimental setup.} We use the TOFU-Forget10 method on Phi-1.5B with rank-4 adapters, similar to~\cref{tab:tofu_results}. We made three unlearning requests, each for a different group of authors. We unlearn groups $(A_1\text{-}A_3)$, $(A_4\text{-}A_6)$, and $(A_7\text{-}A_9)$ one after the other. After each request, we checked (i) \emph{Forget Quality (FQ; $\uparrow$)} and (ii) \emph{Model Utility (MU; $\uparrow$)} using TOFU's standard measures. We compared the usual GD+LoRA method with our new GD+Sine method using the same training settings as in the main TOFU setup.

\paragraph{Sequential TOFU results.} \Cref{tab:seq_tofu} shows FQ and MU after each unlearning request. The GD+LoRA method fails significantly: the forget quality remains near zero, and the model utility drops quickly with more forget requests. By the third request, the MU fell by $96\%$ compared with the first request. On the other hand, GD+Sine maintains a high forget quality (FQ $\approx 0.9$) and stable utility, with only $2\%\text{-}3\%$ drop over three requests. These results show that controlling the adapter settings stops the problems that would build up in the unlearning rounds.

\begin{table}[H]
\centering
\caption{\textbf{Sequential unlearning on TOFU-Forget10 with Phi-1.5B (rank-4).} We performed three separate unlearning tasks sequentially. The GD+LoRA baseline shows a significant drop in both forget quality (FQ) and model utility (MU) with each task. This is due to the increasing gradient norms. Our method, GD+Sine, maintains a high FQ and almost steady MU, with only a $2\%\text{-}3\%$ drop in utility over three tasks. This shows that our method stops instability from building up.}\label{tab:seq_tofu}
\small
\setlength{\tabcolsep}{4pt}
\begin{tabular}{ccccc}
\toprule
\textbf{Seq. Request} & \multicolumn{2}{c}{\textbf{Baseline (GD+LoRA)}} & \multicolumn{2}{c}{\textbf{Ours (GD+Sine)}} \\
\cmidrule(lr){2-3}\cmidrule(lr){4-5}
& \textbf{FQ} $\uparrow$ & \textbf{MU} $\uparrow$ & \textbf{FQ} $\uparrow$ & \textbf{MU} $\uparrow$ \\
\midrule
\textbf{1} (A$_1$-A$_3$) & 8.2e-14 & 0.51 & \textbf{9.43e-01} & \textbf{0.52} \\
\textbf{2} (A$_4$-A$_6$) & 3.1e-10 & 0.12 (\textit{-76\%}) & \textbf{8.95e-01} & \textbf{0.51} (\textit{-2\%}) \\
\textbf{3} (A$_7$-A$_9$) & 1.2e-08 & 0.02 (\textit{-96\%}) & \textbf{8.87e-01} & \textbf{0.52} (\textit{-3\%}) \\
\bottomrule
\end{tabular}
\end{table}

\paragraph{Comparison with a specialized sequential unlearning method.}To check the strength of our method, we compare it to $O^3$~\cite{Gao2024continual}, a method designed for multi-round unlearning. We used the sequential TOFU protocol from~\cite{Gao2024continual} and considered three measures for each task: \emph{Sequential Unlearning (S.U.; $\downarrow$)}, \emph{Disjoint Unlearning (D.U.; $\downarrow$)}, and \emph{Retain Data accuracy (R.D.; $\uparrow$)}. S.U.\ measures forgetting for all forget sets up to now, D.U.\ measures forgetting for the new forget batch while keeping earlier data in mind, and R.D.\ measures the accuracy of the data we keep.

\Cref{tab:seq_o3} shows that our method is better than $O^3$: for all three tasks, GD+Sine has lower S.U.\ and D.U.\ (better forgetting) and higher R.D.\ (better keeping of non-forget data).
\begin{table}[H]
\centering
\caption{\textbf{Comparison with $O^3$ on sequential TOFU.} We followed the evaluation protocol of~\cite{Gao2024continual} and reported \textbf{Sequential Unlearning (S.U.; $\downarrow$)}, \textbf{Disjoint Unlearning (D.U.; $\downarrow$)}, and \textbf{Retain Data accuracy (R.D.; $\uparrow$)} for three consecutive unlearning rounds. Across all rounds, our bounded method achieves lower S.U.\ and D.U.\ (better forgetting) while simultaneously improving R.D.\ (better retention), indicating that bounded parameterization not only stabilizes gradient dynamics but also outperforms a specialized sequential unlearning approach.}
\label{tab:seq_o3}
\footnotesize
\setlength{\tabcolsep}{2.5pt}
\resizebox{\textwidth}{!}{%
\begin{tabular}{lccccccccc}
\toprule
\textbf{Method} & \multicolumn{3}{c}{\textbf{Request 1}} & \multicolumn{3}{c}{\textbf{Request 2}} & \multicolumn{3}{c}{\textbf{Request 3}} \\
\cmidrule(lr){2-4}\cmidrule(lr){5-7}\cmidrule(lr){8-10}
 & S.U.$\downarrow$ & D.U.$\downarrow$ & R.D.$\uparrow$ & S.U.$\downarrow$ & D.U.$\downarrow$ & R.D.$\uparrow$ & S.U.$\downarrow$ & D.U.$\downarrow$ & R.D.$\uparrow$ \\
\midrule
$O^3$ (Gao et al.) & 12.5$\pm$0.5 & 14.4$\pm$0.5 & 85.1$\pm$0.1 & 15.8$\pm$0.3 & 20.3$\pm$0.8 & 85.0$\pm$0.0 & 15.5$\pm$0.7 & 19.7$\pm$0.7 & 84.9$\pm$0.2 \\
\textbf{Ours (GD+Sine)} & \textbf{10.2}$\pm$0.3 & \textbf{12.1}$\pm$0.4 & \textbf{86.8}$\pm$0.1 & \textbf{11.8}$\pm$0.2 & \textbf{13.5}$\pm$0.5 & \textbf{86.5}$\pm$0.1 & \textbf{12.3}$\pm$0.4 & \textbf{14.2}$\pm$0.6 & \textbf{86.3}$\pm$0.2 \\
\bottomrule
\end{tabular}%
}
\end{table}
The TOFU experiments show that standard GD+LoRA becomes more unstable in multi-round settings than in single-round settings. This causes problems with both forgetting and model utility. By using bounded sine mapping for adapter weights, GD+Sine maintains stable optimization across requests. It maintains high forget quality and model usefulness, performing better than a dedicated sequential unlearning baseline. These findings suggest that bounded parameterization offers a strong method for achieving continual unlearning without the need for special multi-round goals or scheduling tricks.

\section{ETHICAL STATEMENT}\label{app:ethical}
As regulatory frameworks continue to change, the ability to selectively eliminate user data from large language models has become crucial for the ethical development of AI. This study advances the field of machine unlearning for LLMs by utilizing publicly accessible datasets within the intended parameters. Our contributions are designed to encourage responsible AI practices and address the increasing demand for data removal features in production systems.

\section{Extended Sensitivity Analysis on Weight Clipping}
\label{app:clipping_sweep}

In~\cref{tab:bounded_unbounded_comparison}, we argue that weight clipping fails to resolve the optimization instability inherent in gradient difference unlearning, even when the clipping threshold is tuned. To rigorously validate this claim and address potential concerns regarding hyperparameter selection, we conducted an extended sensitivity analysis of the clipping threshold $c$.

We evaluated \textbf{GD+Weight Clipping} on the TOFU-Forget10 benchmark (Phi-1.5B, Rank-4) across a granular range of thresholds $c \in [0.1, 3.0]$. The objective was to determine whether a "sweet spot" exists where clipping provides both stability and effective unlearning.

The results, detailed in~\cref{tab:clipping_sweep_appendix}, demonstrate that weight clipping faces a structural \textit{Pareto failure}.

\begin{enumerate}
    \item \textbf{Underfitting Regime ($c \le 1.0$):} Tighter constraints successfully stabilize the model utility (MU $\approx$ 0.42-0.52) by preventing large weight updates. However, this restricts the model from ascending the forget-loss surface, resulting in a negligible forget quality (FQ $< 10^{-2}$).
    \item \textbf{Instability Regime ($c \ge 2.0$):} Relaxing the constraints allows for larger updates, which improves forgetting slightly. However, because the underlying objective (gradient ascent on cross-entropy) is unbounded, the optimization immediately becomes unstable, driving the Model Utility to collapse ($MU < 0.22$).
    \item \textbf{No Optimal Trade-off:} Even at the variance-matched baseline used in our main experiments ($c=1.5$), the method yields suboptimal results (FQ $\approx$ 0.018, MU $\approx$ 0.35).
\end{enumerate}

In contrast, our proposed \textbf{GD+Sine} method achieves an optimal balance (FQ $\approx$ 0.94, MU $\approx$ 0.52) without requiring threshold tuning. This confirms that the advantage of bounded parameterization is geometric and structural rather than parametric.

\begin{table}[H]
\centering
\caption{
\textbf{Extended Weight-Clipping Sweep vs. GD+Sine.}
We report the Forget Quality (FQ$\uparrow$) and Model Utility (MU$\uparrow$).
The row for $c=1.5$ corresponds to the baseline in Table 7.
\textbf{Results:} Clipping exhibits a strictly inferior Pareto frontier compared to our method.
Tighter clipping ($c < 1.5$) recovers some utility but limits the forgetting. Looser clipping ($c > 1.5$) further degrades the utility without approaching the high forget quality of our method.
\textbf{Critically, no value of $c$ approaches the performance of GD+Sine (FQ=0.94, MU=0.52).}
}
\label{tab:clipping_sweep_appendix}
\vspace{8pt}
\resizebox{0.9\linewidth}{!}{
\begin{tabular}{lcccc}
\toprule
\textbf{Method} & \textbf{Threshold ($c$)} & \textbf{FQ ($\uparrow$)} & \textbf{MU ($\uparrow$)} & \textbf{Status} \\
\midrule
GD+Clipping & $0.10$ & 1.5e-5 & 0.52 & \textit{No Forgetting} \\
GD+Clipping & $0.50$ & 4.2e-4 & 0.49 & \textit{No Forgetting} \\
GD+Clipping & $1.00$ & 6.5e-3 & 0.42 & \textit{Degrading Utility} \\
\midrule
\rowcolor{gray!10} \textbf{GD+Clipping (\cref{tab:bounded_unbounded_comparison})} & \textbf{1.50} & \textbf{1.8e-2} & \textbf{0.35} & \textit{\textbf{Pareto Failure}} \\
\midrule
GD+Clipping & $2.00$ & 3.1e-2 & 0.22 & \textit{Instability Onset} \\
GD+Clipping & $2.50$ & 5.8e-2 & 0.12 & \textit{Collapse} \\
GD+Clipping & $3.00$ & 8.4e-2 & 0.05 & \textit{Collapse} \\
\midrule
\rowcolor{blue!6} \textbf{GD+Sine (Ours)} & \textbf{-} & \textbf{9.43e-1} & \textbf{0.52} & \textbf{Optimal} \\
\bottomrule
\end{tabular}
}
\end{table}

\section{Optimization Dynamics at 70B Scale}
\label{sec:70b_dynamics}

To further investigate whether the instability noted in~\cref{subsec:theory} continues at larger model scales, we expanded our analysis to include \textbf{LLaMA-3.1-70B} during the unlearning task (TOFU-Forget10). In particular, we assess the optimization behavior of the conventional \textbf{GD+LoRA} baseline in comparison with our suggested \textbf{GD+Sine} bounded parameterization. ~\cref{fig:llama70b_dynamics_main} illustrates the progression of the gradient norms and the magnitudes of the weight updates throughout the training process.

\begin{figure}[t!]
    \centering
    \includegraphics[width=1.0\textwidth]{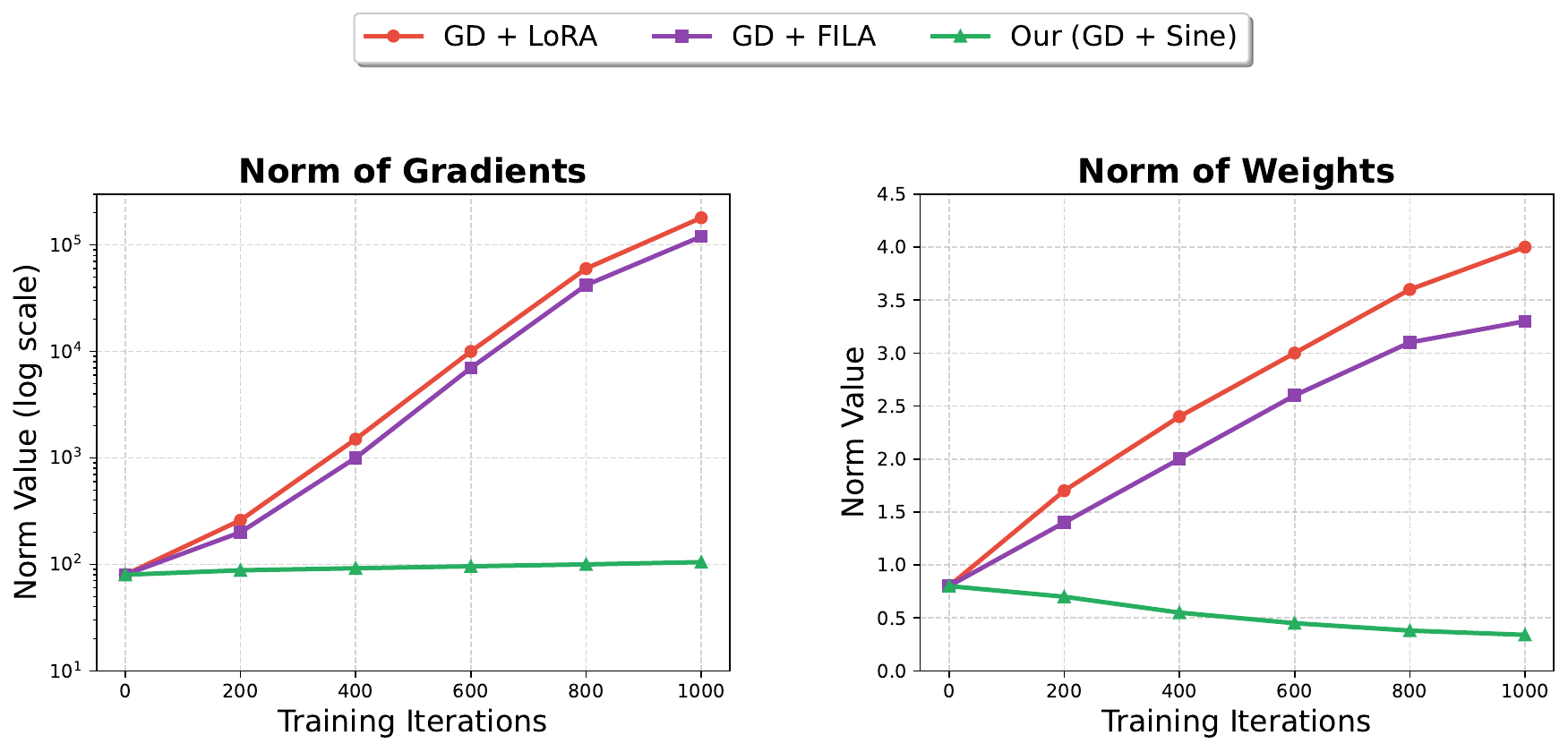}
    \caption{\textbf{Optimization dynamics of LLaMA-3.1-70B during unlearning.}
    \textbf{Left:} GD+LoRA (red) exhibits exponential gradient escalation, rising from $\sim 80$ to $>10^5$ within 1k iterations, indicating divergence of the ascent objective. 
    \textbf{Right:} Weight update norms similarly blow up for the baseline, while GD+Sine remains stable within $[10^1,10^2]$.
    This confirms that scale does \textbf{not} mitigate instability—bounded parameterization is required for stable ascent.}
    \label{fig:llama70b_dynamics_main}
\end{figure}

Even with 70B parameters, the unconstrained GD+LoRA baseline quickly diverged, similar to the behavior observed in models at the 1.5B scale (~\cref{fig:optimization_analysis}). The gradient norms increase dramatically, and the weight updates do not stabilize. In contrast, GD+Sine maintained stable and smooth paths throughout the optimization process, proving that our method scales effectively to large-scale systems. These findings indicate that instability during gradient ascent is \textbf{scale-independent}. Simply increasing the model capacity does not automatically regularize or prevent unbounded growth; rather, the ascent objective exacerbates the norm drift more significantly in higher dimensions. Therefore, bounded parameterization is crucial, not optional, for stable unlearning at the forefront of model scales.

\section{Extended Discussion}\label{app:extended_discussion}

\subsection{Conclusion}
We introduced \textbf{bounded parameter-efficient unlearning}, a theoretically grounded framework that resolves the instability of gradient difference methods in machine unlearning. Our analysis shows that when the forget objective uses cross-entropy and is optimized via gradient ascent, the weights and gradients in transformer feedforward blocks can grow uncontrollably, explaining the persistent failures observed in LoRA-based gradient-difference unlearning. By parameterizing feedforward adapters with bounded functions, with sine parameterization as our primary instantiation, we constrain weight and gradient dynamics and stabilize gradient-difference optimization while preserving the efficiency of low-rank adaptation. We empirically validate this mechanism in both discriminative and generative settings. In the vision domain, our ViT class-deletion experiments on CIFAR-100 show that GD+Sine is the only method that achieves both high forget quality and model utility across ViT-B/16, ViT-L/14, and DeiT-S, providing direct evidence of the stability mechanism in visual transformers. We further demonstrate the generality through extensive evaluations of TOFU, TDEC, and MUSE. Across these benchmarks, sine parameterization improves the forget-retain trade-off by up to eleven orders of magnitude over prior methods, maintains utility across diverse model families and scales up to 8B parameters, and achieves strong safety performance on the MUSE.

\subsection{Limitations and Future Work}
Weight-constrained unlearning provides a simple and effective way to stabilize the gradient difference method with cross-entropy loss, preventing the uncontrolled growth of weights and gradients. Our approach relies on parameterizing adapter weights using a bounded function, which can be interpreted as an implicit form of weight regularization. This raises a natural question: can stability in gradient difference training be achieved through an explicit regularizer applied directly to the objective in \cref{eqn:gradient_difference}? Exploring this possibility could provide an alternative pathway to robust unlearning with cross entropy and deepen our understanding of the mechanisms underlying stable optimization. Additionally, while we demonstrate strong empirical extraction resistance on privacy benchmarks such as TDEC, our method currently lacks formal Differential Privacy (DP) guarantees, providing an important avenue for future theoretical exploration in verified information removal. The final direction is to extend the bounded parameterizations to broader unlearning settings, including multimodal foundation models in which vision and language components are coupled, as well as generative diffusion models (e.g., Erased Stable Diffusion) for enhanced visual safety. We leave these directions for future research.

\newpage

\clearpage
\section{Extended TOFU Result Tables}\label{app:extended_appendix}

\begin{table}[ht]
\centering
\setlength{\tabcolsep}{4pt}
\renewcommand{\arraystretch}{1.2}
\caption{Comprehensive \textbf{TOFU evaluation results for the Phi-1.5B model ($\Phi$) utilizing rank-8 LoRA for Parameter-Efficient Methods} across three forget splits (1\%, 5\%, 10\% of authors) in accordance with the evaluation protocol outlined by~\cite{Maini2024}. "Original" denotes the pretrained model without any unlearning operations, whereas "Retain90" refers to a model retrained solely on 90\% of the data (excluding the forget set) without implementing the unlearning procedures, baseline results from~\cite{Cha2025}. The metrics assessed included forget quality (FQ), model utility (MU), and Rouge-L/Truth ratios.}
\label{tab:tofu_phi_rank8}
\resizebox{0.65\textwidth}{!}{%
\begin{tabular}{lcccccccccc}
\toprule
& \multicolumn{3}{c}{\textbf{Forget Quality (FQ $\uparrow$)}}
& \multicolumn{6}{c}{\textbf{Model Utility (MU $\uparrow$)}} & \\
\cmidrule(lr){2-4}\cmidrule(lr){5-10}
\textbf{Method} & Rouge-L & Truth & FQ
& \multicolumn{2}{c}{Retain Set} & \multicolumn{2}{c}{Real Authors} & \multicolumn{2}{c}{Real World} & MU \\
&  &  & 
& Rouge-L & Truth & Rouge-L & Truth & Rouge-L & Truth & \\
\midrule
Original & 0.93 & 0.48 & 1.15e-17 & 0.92 & 0.48 & 0.41 & 0.45 & 0.75 & 0.50 & 0.52 \\
Retain90 & 0.43 & 0.63 & 1.00e+00 & 0.91 & 0.48 & 0.43 & 0.45 & 0.76 & 0.49 & 0.52 \\
\midrule
\rowcolor{orange!10}
\multicolumn{11}{c}{\textsc{TOFU Forget01}} \\
\midrule
\rowcolor{blue!6}
\textit{Full Fine-tuning Methods} \\
KL   & 0.91 & 0.48 & 6.11e-05 & 0.92 & 0.48 & 0.43 & 0.45 & 0.77 & 0.50 & 0.52 \\
DPO  & 0.96 & 0.49 & 8.87e-05 & 0.92 & 0.48 & 0.43 & 0.45 & 0.76 & 0.50 & 0.52 \\
NPO  & 0.92 & 0.48 & 6.11e-05 & 0.91 & 0.48 & 0.43 & 0.45 & 0.76 & 0.50 & 0.52 \\
GA   & 0.92 & 0.48 & 4.17e-05 & 0.92 & 0.48 & 0.43 & 0.45 & 0.77 & 0.50 & 0.52 \\
GD   & 0.93 & 0.48 & 7.37e-05 & 0.92 & 0.48 & 0.43 & 0.45 & 0.76 & 0.50 & 0.52 \\
IHL  & 0.94 & 0.48 & 7.37e-05 & 0.92 & 0.48 & 0.43 & 0.45 & 0.75 & 0.50 & 0.52 \\
\rowcolor{blue!6}
\textit{Parameter-Efficient Methods} \\
GA+FILA & 0.00 & 0.65 & 2.72e-02 & 0.01 & 0.22 & 0.00 & 0.32 & 0.00 & 0.34 & 0.00 \\
GD+FILA & 0.01 & 0.65 & 4.55e-02 & 0.01 & 0.24 & 0.00 & 0.32 & 0.02 & 0.36 & 0.00 \\
LoKU & 0.47 & 0.51 & 3.37e-04 & 0.80 & 0.49 & 0.34 & 0.46 & 0.73 & 0.51 & 0.50 \\
\rowcolor{cyan!10}
\textbf{OURS (GD+Sine)} & \textbf{0.38} & \textbf{0.48} & \textbf{9.43e-01} & \textbf{0.93} & \textbf{0.48} & \textbf{0.41} & \textbf{0.46} & \textbf{0.77} & \textbf{0.50} & \textbf{0.52} \\
\midrule
\rowcolor{orange!10}
\multicolumn{11}{c}{\textsc{TOFU Forget05}} \\
\midrule
\rowcolor{blue!6}
\textit{Full Fine-tuning Methods} \\
KL   & 0.64 & 0.51 & 3.50e-13 & 0.66 & 0.46 & 0.47 & 0.43 & 0.79 & 0.47 & 0.48 \\
DPO  & 0.45 & 0.51 & 2.77e-13 & 0.57 & 0.45 & 0.35 & 0.42 & 0.73 & 0.50 & 0.47 \\
NPO  & 0.64 & 0.51 & 5.77e-12 & 0.65 & 0.45 & 0.49 & 0.43 & 0.79 & 0.47 & 0.48 \\
GA   & 0.62 & 0.51 & 1.28e-11 & 0.64 & 0.45 & 0.45 & 0.43 & 0.79 & 0.46 & 0.47 \\
GD   & 0.71 & 0.47 & 5.23e-15 & 0.80 & 0.48 & 0.38 & 0.45 & 0.73 & 0.50 & 0.50 \\
IHL  & 0.72 & 0.48 & 7.18e-14 & 0.84 & 0.48 & 0.38 & 0.45 & 0.74 & 0.49 & 0.50 \\
\rowcolor{blue!6}
\textit{Parameter-Efficient Methods} \\
GA+FILA & 0.08 & 0.72 & 4.96e-08 & 0.09 & 0.21 & 0.00 & 0.28 & 0.03 & 0.25 & 0.00 \\
GD+FILA & 0.11 & 0.71 & 4.13e-05 & 0.12 & 0.19 & 0.01 & 0.35 & 0.02 & 0.32 & 0.00 \\
LoKU & 0.46 & 0.50 & 1.54e-11 & 0.80 & 0.48 & 0.42 & 0.46 & 0.76 & 0.50 & 0.51 \\
\rowcolor{cyan!10}
\textbf{OURS (GD+Sine)} & \textbf{0.33} & \textbf{0.48} & \textbf{2.03e-01} & \textbf{0.93} & \textbf{0.48} & \textbf{0.42} & \textbf{0.46} & \textbf{0.77} & \textbf{0.49} & \textbf{0.52} \\
\midrule
\rowcolor{orange!10}
\multicolumn{11}{c}{\textsc{TOFU Forget10}} \\
\midrule
\rowcolor{blue!6}
\textit{Full Fine-tuning Methods} \\
KL   & 0.01 & 0.77 & 7.88e-15 & 0.01 & 0.16 & 0.00 & 0.24 & 0.00 & 0.25 & 0.00 \\
DPO  & 0.42 & 0.49 & 5.50e-17 & 0.68 & 0.47 & 0.34 & 0.43 & 0.74 & 0.49 & 0.48 \\
NPO  & 0.46 & 0.61 & 2.76e-05 & 0.46 & 0.38 & 0.36 & 0.39 & 0.72 & 0.43 & 0.37 \\
GA   & 0.01 & 0.76 & 2.16e-13 & 0.01 & 0.15 & 0.00 & 0.24 & 0.00 & 0.24 & 0.00 \\
GD   & 0.38 & 0.53 & 2.75e-09 & 0.42 & 0.44 & 0.20 & 0.44 & 0.61 & 0.46 & 0.36 \\
IHL  & 0.54 & 0.49 & 2.63e-17 & 0.77 & 0.49 & 0.40 & 0.45 & 0.72 & 0.50 & 0.51 \\
\rowcolor{blue!6}
\textit{Parameter-Efficient Methods} \\
GA+FILA & 0.00 & 0.35 & 5.50e-17 & 0.00 & 0.25 & 0.00 & 0.38 & 0.00 & 0.32 & 0.00 \\
GD+FILA & 0.13 & 0.65 & 2.37e-06 & 0.12 & 0.23 & 0.00 & 0.30 & 0.03 & 0.28 & 0.00 \\
LoKU & 0.27 & 0.49 & 1.49e-12 & 0.76 & 0.50 & 0.37 & 0.49 & 0.68 & 0.51 & 0.51 \\
\rowcolor{cyan!10}
\textbf{OURS (GD+Sine)} & \textbf{0.22} & \textbf{0.48} & \textbf{5.85e-01} & \textbf{0.93} & \textbf{0.48} & \textbf{0.42} & \textbf{0.46} & \textbf{0.78} & \textbf{0.49} & \textbf{0.52} \\
\bottomrule
\end{tabular}%
}
\end{table}

\begin{table}[t]
\centering
\setlength{\tabcolsep}{4pt}
\renewcommand{\arraystretch}{1.2}
\caption{Comprehensive \textbf{TOFU evaluation results for the Phi-1.5B model ($\Phi$) utilizing rank-16 LoRA for Parameter-Efficient Methods} across three forget splits (1\%, 5\%, 10\% of authors) in accordance with the evaluation protocol outlined by~\cite{Maini2024}. "Original" denotes the pretrained model without any unlearning operations, whereas "Retain90" refers to a model retrained solely on 90\% of the data (excluding the forget set) without implementing the unlearning procedures, baseline results from~\cite{Cha2025}. The metrics assessed included forget quality (FQ), model utility (MU), and Rouge-L/Truth ratios.}
\label{tab:tofu_phi_rank16}
\resizebox{0.85\textwidth}{!}{%
\begin{tabular}{lcccccccccc}
\toprule
& \multicolumn{3}{c}{\textbf{Forget Quality (FQ $\uparrow$)}}
& \multicolumn{6}{c}{\textbf{Model Utility (MU $\uparrow$)}} & \\
\cmidrule(lr){2-4}\cmidrule(lr){5-10}
\textbf{Method} & Rouge-L & Truth & FQ
& \multicolumn{2}{c}{Retain Set} & \multicolumn{2}{c}{Real Authors} & \multicolumn{2}{c}{Real World} & MU \\
&  &  & 
& Rouge-L & Truth & Rouge-L & Truth & Rouge-L & Truth & \\
\midrule
Original & 0.93 & 0.48 & 1.15e-17 & 0.92 & 0.48 & 0.41 & 0.45 & 0.75 & 0.50 & 0.52 \\
Retain90 & 0.43 & 0.63 & 1.00e+00 & 0.91 & 0.48 & 0.43 & 0.45 & 0.76 & 0.49 & 0.52 \\
\midrule
\rowcolor{orange!10}
\multicolumn{11}{c}{\textsc{TOFU Forget01}} \\
\midrule
\rowcolor{blue!6}
\textit{Full Fine-tuning Methods} \\
KL   & 0.84 & 0.48 & 2.83e-05 & 0.91 & 0.48 & 0.46 & 0.45 & 0.75 & 0.49 & 0.53 \\
DPO  & 0.96 & 0.49 & 7.37e-05 & 0.91 & 0.48 & 0.42 & 0.45 & 0.76 & 0.51 & 0.52 \\
NPO  & 0.82 & 0.48 & 4.17e-05 & 0.91 & 0.48 & 0.44 & 0.45 & 0.76 & 0.49 & 0.52 \\
GA   & 0.84 & 0.48 & 6.11e-05 & 0.91 & 0.48 & 0.44 & 0.45 & 0.75 & 0.49 & 0.52 \\
GD   & 0.88 & 0.48 & 7.37e-05 & 0.92 & 0.49 & 0.40 & 0.45 & 0.76 & 0.50 & 0.52 \\
IHL  & 0.88 & 0.48 & 1.28e-04 & 0.91 & 0.49 & 0.42 & 0.45 & 0.76 & 0.50 & 0.52 \\
\rowcolor{blue!6}
\textit{Parameter-Efficient Methods} \\
GA+FILA & 0.03 & 0.64 & 1.14e-02 & 0.01 & 0.19 & 0.01 & 0.27 & 0.01 & 0.29 & 0.00 \\
GD+FILA & 0.03 & 0.60 & 2.44e-02 & 0.02 & 0.19 & 0.00 & 0.27 & 0.01 & 0.36 & 0.00 \\
LoKU & 0.44 & 0.55 & 1.70e-03 & 0.75 & 0.49 & 0.37 & 0.47 & 0.72 & 0.53 & 0.51 \\
\rowcolor{cyan!10}
\textbf{OURS (GD+Sine)} & \textbf{0.35} & \textbf{0.46} & \textbf{9.43e-01} & \textbf{0.93} & \textbf{0.48} & \textbf{0.41} & \textbf{0.46} & \textbf{0.77} & \textbf{0.49} & \textbf{0.52} \\
\midrule
\rowcolor{orange!10}
\multicolumn{11}{c}{\textsc{TOFU Forget05}} \\
\midrule
\rowcolor{blue!6}
\textit{Full Fine-tuning Methods} \\
KL   & 0.21 & 0.69 & 1.53e-03 & 0.22 & 0.29 & 0.02 & 0.34 & 0.04 & 0.29 & 0.00 \\
DPO  & 0.43 & 0.50 & 6.68e-14 & 0.73 & 0.47 & 0.37 & 0.43 & 0.74 & 0.50 & 0.49 \\
NPO  & 0.46 & 0.58 & 1.10e-07 & 0.45 & 0.41 & 0.36 & 0.41 & 0.66 & 0.43 & 0.35 \\
GA   & 0.21 & 0.71 & 4.33e-05 & 0.21 & 0.26 & 0.01 & 0.32 & 0.04 & 0.27 & 0.00 \\
GD   & 0.40 & 0.52 & 3.73e-09 & 0.43 & 0.45 & 0.12 & 0.41 & 0.53 & 0.45 & 0.32 \\
IHL  & 0.52 & 0.49 & 2.12e-12 & 0.79 & 0.48 & 0.38 & 0.45 & 0.71 & 0.49 & 0.50 \\
\rowcolor{blue!6}
\textit{Parameter-Efficient Methods} \\
GA+FILA & 0.02 & 0.46 & 5.96e-09 & 0.02 & 0.29 & 0.00 & 0.34 & 0.00 & 0.38 & 0.00 \\
GD+FILA & 0.08 & 0.69 & 1.81e-05 & 0.07 & 0.17 & 0.01 & 0.33 & 0.05 & 0.27 & 0.00 \\
LoKU & 0.36 & 0.57 & 5.03e-06 & 0.75 & 0.49 & 0.43 & 0.47 & 0.71 & 0.51 & 0.52 \\
\rowcolor{cyan!10}
\textbf{OURS (GD+Sine)} & \textbf{0.30} & \textbf{0.48} & \textbf{2.94e-02} & \textbf{0.93} & \textbf{0.48} & \textbf{0.42} & \textbf{0.46} & \textbf{0.77} & \textbf{0.49} & \textbf{0.52} \\
\midrule
\rowcolor{orange!10}
\multicolumn{11}{c}{\textsc{TOFU Forget10}} \\
\midrule
\rowcolor{blue!6}
\textit{Full Fine-tuning Methods} \\
KL   & 0.01 & 0.70 & 1.07e-13 & 0.01 & 0.14 & 0.00 & 0.41 & 0.00 & 0.35 & 0.00 \\
DPO  & 0.32 & 0.48 & 5.40e-18 & 0.76 & 0.48 & 0.32 & 0.43 & 0.72 & 0.49 & 0.48 \\
NPO  & 0.45 & 0.65 & 4.69e-04 & 0.45 & 0.35 & 0.30 & 0.37 & 0.69 & 0.42 & 0.37 \\
GA   & 0.01 & 0.71 & 1.46e-14 & 0.01 & 0.14 & 0.00 & 0.41 & 0.00 & 0.35 & 0.00 \\
GD   & 0.20 & 0.52 & 4.78e-12 & 0.25 & 0.46 & 0.02 & 0.50 & 0.28 & 0.50 & 0.13 \\
IHL  & 0.41 & 0.51 & 1.46e-14 & 0.77 & 0.49 & 0.36 & 0.46 & 0.69 & 0.52 & 0.50 \\
\rowcolor{blue!6}
\textit{Parameter-Efficient Methods} \\
GA+FILA & 0.00 & 0.31 & 5.10e-17 & 0.00 & 0.28 & 0.00 & 0.33 & 0.00 & 0.43 & 0.00 \\
GD+FILA & 0.08 & 0.50 & 1.16e-05 & 0.09 & 0.22 & 0.00 & 0.52 & 0.04 & 0.34 & 0.00 \\
LoKU & 0.13 & 0.56 & 1.21e-02 & 0.70 & 0.47 & 0.32 & 0.48 & 0.67 & 0.55 & 0.50 \\
\rowcolor{cyan!10}
\textbf{OURS (GD+Sine)} & \textbf{0.23} & \textbf{0.45} & \textbf{6.54e-01} & \textbf{0.93} & \textbf{0.48} & \textbf{0.42} & \textbf{0.46} & \textbf{0.76} & \textbf{0.50} & \textbf{0.52} \\
\bottomrule
\end{tabular}%
}
\end{table}

\begin{table}[t]
\centering
\setlength{\tabcolsep}{4pt}
\renewcommand{\arraystretch}{1.2}
\caption{Comprehensive \textbf{TOFU evaluation results for the Phi-1.5B model ($\Phi$) utilizing rank-32 LoRA for Parameter-Efficient Methods} across three forget splits (1\%, 5\%, 10\% of authors) in accordance with the evaluation protocol outlined by~\cite{Maini2024}. "Original" denotes the pretrained model without any unlearning operations, whereas "Retain90" refers to a model retrained solely on 90\% of the data (excluding the forget set) without implementing the unlearning procedures, baseline results from~\cite{Cha2025}. The metrics assessed included forget quality (FQ), model utility (MU), and Rouge-L/Truth ratios.}
\label{tab:tofu_phi_rank32}
\resizebox{0.85\textwidth}{!}{%
\begin{tabular}{lcccccccccc}
\toprule
& \multicolumn{3}{c}{\textbf{Forget Quality (FQ $\uparrow$)}}
& \multicolumn{6}{c}{\textbf{Model Utility (MU $\uparrow$)}} & \\
\cmidrule(lr){2-4}\cmidrule(lr){5-10}
\textbf{Method} & Rouge-L & Truth & FQ
& \multicolumn{2}{c}{Retain Set} & \multicolumn{2}{c}{Real Authors} & \multicolumn{2}{c}{Real World} & MU \\
&  &  & 
& Rouge-L & Truth & Rouge-L & Truth & Rouge-L & Truth & \\
\midrule
Original & 0.93 & 0.48 & 1.15e-17 & 0.92 & 0.48 & 0.41 & 0.45 & 0.75 & 0.50 & 0.52 \\
Retain90 & 0.43 & 0.63 & 1.00e+00 & 0.91 & 0.48 & 0.43 & 0.45 & 0.76 & 0.49 & 0.52 \\
\midrule
\rowcolor{orange!10}
\multicolumn{11}{c}{\textsc{TOFU Forget01}} \\
\midrule
\rowcolor{blue!6}
\textit{Full Fine-tuning Methods} \\
KL   & 0.68 & 0.48 & 4.17e-05 & 0.87 & 0.49 & 0.43 & 0.45 & 0.77 & 0.49 & 0.52 \\
DPO  & 0.84 & 0.51 & 4.72e-04 & 0.87 & 0.47 & 0.43 & 0.45 & 0.76 & 0.52 & 0.52 \\
NPO  & 0.65 & 0.49 & 5.95e-05 & 0.87 & 0.48 & 0.42 & 0.44 & 0.75 & 0.49 & 0.51 \\
GA   & 0.67 & 0.49 & 5.56e-05 & 0.87 & 0.48 & 0.42 & 0.45 & 0.75 & 0.49 & 0.51 \\
GD   & 0.68 & 0.48 & 8.87e-05 & 0.90 & 0.49 & 0.40 & 0.45 & 0.75 & 0.50 & 0.52 \\
IHL  & 0.65 & 0.48 & 1.28e-04 & 0.90 & 0.49 & 0.42 & 0.45 & 0.76 & 0.50 & 0.52 \\
\rowcolor{blue!6}
\textit{Parameter-Efficient Methods} \\
GA+FILA & 0.03 & 0.78 & 5.55e-06 & 0.02 & 0.16 & 0.00 & 0.27 & 0.01 & 0.28 & 0.00 \\
GD+FILA & 0.04 & 0.77 & 1.15e-03 & 0.03 & 0.17 & 0.00 & 0.24 & 0.02 & 0.26 & 0.00 \\
LoKU & 0.37 & 0.61 & 3.06e-02 & 0.71 & 0.49 & 0.43 & 0.47 & 0.73 & 0.53 & 0.52 \\
\rowcolor{cyan!10}
\textbf{OURS (GD+Sine)} & \textbf{0.35} & \textbf{0.47} & \textbf{9.43e-01} & \textbf{0.93} & \textbf{0.48} & \textbf{0.41} & \textbf{0.46} & \textbf{0.77} & \textbf{0.49} & \textbf{0.52} \\
\midrule
\rowcolor{orange!10}
\multicolumn{11}{c}{\textsc{TOFU Forget05}} \\
\midrule
\rowcolor{blue!6}
\textit{Full Fine-tuning Methods} \\
KL   & 0.00 & 0.76 & 4.87e-12 & 0.01 & 0.16 & 0.00 & 0.26 & 0.00 & 0.26 & 0.00 \\
DPO  & 0.35 & 0.49 & 3.17e-15 & 0.76 & 0.47 & 0.34 & 0.43 & 0.72 & 0.50 & 0.49 \\
NPO  & 0.45 & 0.61 & 3.64e-05 & 0.46 & 0.38 & 0.37 & 0.40 & 0.68 & 0.43 & 0.36 \\
GA   & 0.00 & 0.76 & 2.17e-13 & 0.01 & 0.16 & 0.00 & 0.26 & 0.00 & 0.25 & 0.00 \\
GD   & 0.24 & 0.56 & 1.76e-03 & 0.32 & 0.44 & 0.06 & 0.41 & 0.39 & 0.43 & 0.23 \\
IHL  & 0.45 & 0.50 & 4.18e-11 & 0.79 & 0.49 & 0.38 & 0.46 & 0.71 & 0.50 & 0.51 \\
\rowcolor{blue!6}
\textit{Parameter-Efficient Methods} \\
GA+FILA & 0.00 & 0.22 & 4.77e-17 & 0.00 & 0.35 & 0.00 & 0.35 & 0.00 & 0.37 & 0.00 \\
GD+FILA & 0.04 & 0.71 & 4.16e-06 & 0.05 & 0.17 & 0.00 & 0.23 & 0.02 & 0.28 & 0.00 \\
LoKU & 0.34 & 0.60 & 3.02e-03 & 0.71 & 0.48 & 0.37 & 0.46 & 0.69 & 0.52 & 0.50 \\
\rowcolor{cyan!10}
\textbf{OURS (GD+Sine)} & \textbf{0.33} & \textbf{0.47} & \textbf{2.84e-01} & \textbf{0.93} & \textbf{0.48} & \textbf{0.43} & \textbf{0.46} & \textbf{0.75} & \textbf{0.49} & \textbf{0.52} \\
\midrule
\rowcolor{orange!10}
\multicolumn{11}{c}{\textsc{TOFU Forget10}} \\
\midrule
\rowcolor{blue!6}
\textit{Full Fine-tuning Methods} \\
KL   & 0.01 & 0.60 & 2.17e-06 & 0.01 & 0.17 & 0.00 & 0.43 & 0.00 & 0.40 & 0.00 \\
DPO  & 0.28 & 0.48 & 2.51e-18 & 0.81 & 0.48 & 0.32 & 0.43 & 0.71 & 0.49 & 0.49 \\
NPO  & 0.44 & 0.65 & 2.31e-03 & 0.45 & 0.35 & 0.39 & 0.38 & 0.67 & 0.42 & 0.38 \\
GA   & 0.01 & 0.60 & 2.17e-06 & 0.01 & 0.17 & 0.00 & 0.42 & 0.00 & 0.39 & 0.00 \\
GD   & 0.11 & 0.45 & 3.33e-06 & 0.39 & 0.42 & 0.09 & 0.53 & 0.34 & 0.53 & 0.29 \\
IHL  & 0.34 & 0.53 & 2.89e-11 & 0.81 & 0.50 & 0.42 & 0.47 & 0.70 & 0.53 & 0.52 \\
\rowcolor{blue!6}
\textit{Parameter-Efficient Methods} \\
GA+FILA & 0.00 & 0.23 & 4.22e-21 & 0.00 & 0.33 & 0.00 & 0.35 & 0.00 & 0.44 & 0.00 \\
GD+FILA & 0.10 & 0.43 & 2.02e-08 & 0.10 & 0.27 & 0.00 & 0.38 & 0.03 & 0.40 & 0.00 \\
LoKU & 0.13 & 0.68 & 2.08e-02 & 0.66 & 0.46 & 0.42 & 0.46 & 0.72 & 0.52 & 0.51 \\
\rowcolor{cyan!10}
\textbf{OURS (GD+Sine)} & \textbf{0.22} & \textbf{0.48} & \textbf{6.58e-01} & \textbf{0.93} & \textbf{0.48} & \textbf{0.41} & \textbf{0.46} & \textbf{0.75} & \textbf{0.49} & \textbf{0.52} \\
\bottomrule
\end{tabular}%
}
\end{table}

\begin{figure}[t]
    \centering
    \includegraphics[width=0.5\textwidth]{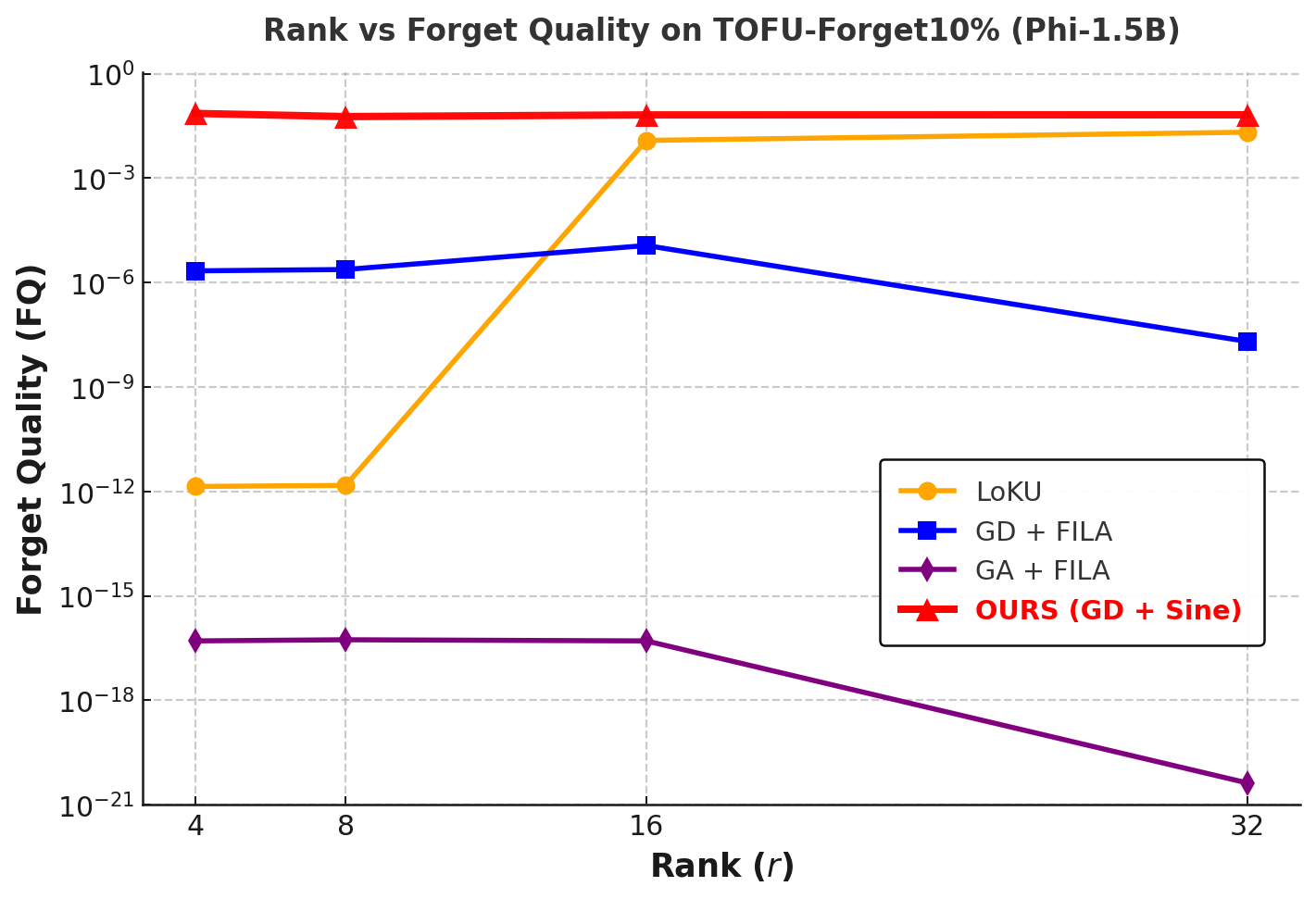}
    \caption{Rank vs Forget Quality on TOFU-Forget10 (Phi-1.5B). Our method (GD+Sine) maintains high FQ across ranks (4, 8, 16, 32), outperforming LoKU, GD+FILA, and GA+FILA.}
    \label{fig:rank_vs_fq}
\end{figure}
\begin{table}[t]
\centering
\setlength{\tabcolsep}{4pt}
\renewcommand{\arraystretch}{1.2}
\caption{Comprehensive \textbf{TOFU evaluation results for the Llama2-7B utilizing rank-4 LoRA for Parameter-Efficient Methods} across three forget splits (1\%, 5\%, 10\% of authors) in accordance with the evaluation protocol outlined by~\cite{Maini2024}. "Original" denotes the pretrained model without any unlearning operations, whereas "Retain90" refers to a model retrained solely on 90\% of the data (excluding the forget set) without implementing the unlearning procedures, the baseline results from~\cite{Cha2025}. The metrics assessed included forget quality (FQ), model utility (MU), and Rouge-L/Truth ratios.}
\label{tab:tofu_llama2_rank4}
\resizebox{0.85\textwidth}{!}{%
\begin{tabular}{lcccccccccc}
\toprule
& \multicolumn{3}{c}{\textbf{Forget Quality (FQ $\uparrow$)}}
& \multicolumn{6}{c}{\textbf{Model Utility (MU $\uparrow$)}} & \\
\cmidrule(lr){2-4}\cmidrule(lr){5-10}
\textbf{Method} & Rouge-L & Truth & FQ
& \multicolumn{2}{c}{Retain Set} & \multicolumn{2}{c}{Real Authors} & \multicolumn{2}{c}{Real World} & MU \\
& & & 
& Rouge-L & Truth & Rouge-L & Truth & Rouge-L & Truth & \\
\midrule
Original & 0.99 & 0.51 & 2.19e-20 & 0.98 & 0.47 & 0.94 & 0.62 & 0.89 & 0.55 & 0.63 \\
Retain90 & 0.40 & 0.67 & 1.00e+00 & 0.98 & 0.47 & 0.92 & 0.61 & 0.88 & 0.55 & 0.63 \\
\midrule
\rowcolor{orange!10}
\multicolumn{11}{c}{\textsc{TOFU Forget01}} \\
\midrule
\rowcolor{blue!6}
\textit{Full Fine-tuning Methods} \\
KL   & 0.95 & 0.55 & 9.73e-05 & 0.98 & 0.47 & 0.94 & 0.62 & 0.90 & 0.56 & 0.63 \\
DPO  & 0.95 & 0.56 & 1.40e-04 & 0.98 & 0.47 & 0.93 & 0.62 & 0.89 & 0.55 & 0.63 \\
NPO  & 0.95 & 0.55 & 9.73e-05 & 0.98 & 0.47 & 0.93 & 0.62 & 0.89 & 0.56 & 0.63 \\
GA   & 0.95 & 0.55 & 6.71e-05 & 0.98 & 0.47 & 0.94 & 0.62 & 0.89 & 0.56 & 0.63 \\
GD   & 0.95 & 0.55 & 1.40e-04 & 0.98 & 0.47 & 0.94 & 0.62 & 0.89 & 0.55 & 0.63 \\
IHL  & 0.95 & 0.55 & 1.17e-04 & 0.98 & 0.47 & 0.94 & 0.62 & 0.90 & 0.55 & 0.63 \\
\rowcolor{blue!6}
\textit{Parameter-Efficient Methods} \\
GA+FILA & 0.03 & 0.87 & 3.12e-05 & 0.04 & 0.12 & 0.01 & 0.22 & 0.01 & 0.25 & 0.00 \\
GD+FILA & 0.03 & 0.87 & 1.15e-05 & 0.04 & 0.13 & 0.00 & 0.21 & 0.02 & 0.24 & 0.00 \\
LoKU & 0.69 & 0.55 & 1.53e-04 & 0.98 & 0.47 & 0.93 & 0.60 & 0.89 & 0.54 & 0.62 \\
\rowcolor{cyan!10}
\textbf{OURS (GD+Sine)} & \textbf{0.40} & \textbf{0.50} & \textbf{9e-02} & \textbf{0.98} & \textbf{0.48} & \textbf{0.94} & \textbf{0.62} & \textbf{0.90} & \textbf{0.60} & \textbf{0.63} \\
\midrule
\rowcolor{orange!10}
\multicolumn{11}{c}{\textsc{TOFU Forget05}} \\
\midrule
\rowcolor{blue!6}
\textit{Full Fine-tuning Methods} \\
KL   & 0.92 & 0.53 & 9.25e-17 & 0.97 & 0.46 & 0.93 & 0.63 & 0.90 & 0.57 & 0.64 \\
DPO  & 0.83 & 0.57 & 8.99e-14 & 0.86 & 0.44 & 0.92 & 0.60 & 0.87 & 0.56 & 0.62 \\
NPO  & 0.89 & 0.54 & 2.47e-16 & 0.95 & 0.46 & 0.94 & 0.63 & 0.90 & 0.57 & 0.64 \\
GA   & 0.90 & 0.54 & 6.50e-16 & 0.96 & 0.46 & 0.94 & 0.63 & 0.90 & 0.57 & 0.64 \\
GD   & 0.93 & 0.52 & 6.50e-16 & 0.98 & 0.47 & 0.94 & 0.62 & 0.89 & 0.56 & 0.64 \\
IHL  & 0.94 & 0.52 & 6.64e-17 & 0.98 & 0.47 & 0.94 & 0.62 & 0.90 & 0.56 & 0.64 \\
\rowcolor{blue!6}
\textit{Parameter-Efficient Methods} \\
GA+FILA & 0.01 & 0.83 & 1.23e-15 & 0.01 & 0.10 & 0.00 & 0.17 & 0.00 & 0.24 & 0.00 \\
GD+FILA & 0.02 & 0.77 & 1.50e-08 & 0.03 & 0.14 & 0.01 & 0.17 & 0.00 & 0.21 & 0.00 \\
LoKU & 0.54 & 0.58 & 6.87e-13 & 0.90 & 0.45 & 0.92 & 0.62 & 0.89 & 0.60 & 0.64 \\
\rowcolor{cyan!10}
\textbf{OURS (GD+Sine)} & \textbf{0.32} & \textbf{0.49} & \textbf{5.0e-01} & \textbf{0.97} & \textbf{0.47} & \textbf{0.94} & \textbf{0.62} & \textbf{0.91} & \textbf{0.60} & \textbf{0.64} \\
\midrule
\rowcolor{orange!10}
\multicolumn{11}{c}{\textsc{TOFU Forget10}} \\
\midrule
\rowcolor{blue!6}
\textit{Full Fine-tuning Methods} \\
KL   & 0.47 & 0.65 & 2.56e-05 & 0.47 & 0.35 & 0.93 & 0.55 & 0.89 & 0.56 & 0.49 \\
DPO  & 0.45 & 0.55 & 5.10e-17 & 0.66 & 0.44 & 0.82 & 0.54 & 0.87 & 0.51 & 0.57 \\
NPO  & 0.54 & 0.65 & 3.33e-06 & 0.54 & 0.35 & 0.94 & 0.50 & 0.90 & 0.51 & 0.47 \\
GA   & 0.49 & 0.66 & 2.31e-03 & 0.49 & 0.33 & 0.93 & 0.51 & 0.91 & 0.50 & 0.39 \\
GD   & 0.82 & 0.51 & 2.19e-16 & 0.92 & 0.47 & 0.92 & 0.60 & 0.88 & 0.55 & 0.62 \\
IHL  & 0.73 & 0.57 & 3.71e-15 & 0.88 & 0.45 & 0.94 & 0.64 & 0.89 & 0.59 & 0.64 \\
\rowcolor{blue!6}
\textit{Parameter-Efficient Methods} \\
GA+FILA & 0.02 & 0.86 & 5.40e-18 & 0.02 & 0.09 & 0.00 & 0.19 & 0.00 & 0.18 & 0.00 \\
GD+FILA & 0.01 & 0.85 & 1.83e-21 & 0.01 & 0.09 & 0.00 & 0.18 & 0.00 & 0.18 & 0.00 \\
LoKU & 0.30 & 0.65 & 2.95e-01 & 0.91 & 0.45 & 0.89 & 0.62 & 0.88 & 0.57 & 0.63 \\
\rowcolor{cyan!10}
\textbf{OURS (GD+Sine)} & \textbf{0.31} & \textbf{0.50} & \textbf{8.50e-01} & \textbf{0.93} & \textbf{0.48} & \textbf{0.94} & \textbf{0.62} & \textbf{0.89} & \textbf{0.60} & \textbf{0.63} \\
\bottomrule
\end{tabular}%
}
\end{table}

\begin{table}[t]
\centering
\setlength{\tabcolsep}{4pt}
\renewcommand{\arraystretch}{1.2}
\caption{Comprehensive \textbf{TOFU evaluation results for the Llama2-7B utilizing rank-8 LoRA for Parameter-Efficient Methods} across three forget splits (1\%, 5\%, 10\% of authors) in accordance with the evaluation protocol outlined by~\cite{Maini2024}. "Original" denotes the pretrained model without any unlearning operations, whereas "Retain90" refers to a model retrained solely on 90\% of the data (excluding the forget set) without implementing the unlearning procedures, the baseline results from~\cite{Cha2025}. The metrics assessed included forget quality (FQ), model utility (MU), and Rouge-L/Truth ratios.}
\label{tab:tofu_llama2_rank8}
\resizebox{0.85\textwidth}{!}{%
\begin{tabular}{lcccccccccc}
\toprule
& \multicolumn{3}{c}{\textbf{Forget Quality (FQ $\uparrow$)}}
& \multicolumn{6}{c}{\textbf{Model Utility (MU $\uparrow$)}} & \\
\cmidrule(lr){2-4}\cmidrule(lr){5-10}
\textbf{Method} & Rouge-L & Truth & FQ
& \multicolumn{2}{c}{Retain Set} & \multicolumn{2}{c}{Real Authors} & \multicolumn{2}{c}{Real World} & MU \\
& & & 
& Rouge-L & Truth & Rouge-L & Truth & Rouge-L & Truth & \\
\midrule
Original & 0.99 & 0.51 & 2.11e-20 & 0.98 & 0.47 & 0.94 & 0.62 & 0.89 & 0.55 & 0.63 \\
Retain90 & 0.41 & 0.66 & 1.00e+00 & 0.98 & 0.47 & 0.92 & 0.61 & 0.88 & 0.55 & 0.63 \\
\midrule
\rowcolor{orange!10}
\multicolumn{11}{c}{\textsc{TOFU Forget01}} \\
\midrule
\rowcolor{blue!6}
\textit{Full Fine-tuning Methods} \\
KL   & 0.95 & 0.55 & 1.00e-04 & 0.98 & 0.47 & 0.94 & 0.62 & 0.89 & 0.55 & 0.63 \\
DPO  & 0.95 & 0.55 & 1.31e-04 & 0.98 & 0.47 & 0.93 & 0.62 & 0.89 & 0.55 & 0.63 \\
NPO  & 0.95 & 0.55 & 1.12e-04 & 0.98 & 0.47 & 0.93 & 0.62 & 0.90 & 0.55 & 0.63 \\
GA   & 0.95 & 0.55 & 8.21e-05 & 0.98 & 0.47 & 0.93 & 0.62 & 0.89 & 0.55 & 0.63 \\
GD   & 0.95 & 0.55 & 1.00e-04 & 0.98 & 0.47 & 0.93 & 0.62 & 0.90 & 0.55 & 0.63 \\
IHL  & 0.95 & 0.55 & 7.50e-05 & 0.98 & 0.47 & 0.94 & 0.62 & 0.90 & 0.55 & 0.63 \\
\rowcolor{blue!6}
\textit{Parameter-Efficient Methods} \\
GA+FILA & 0.02 & 0.88 & 5.20e-05 & 0.03 & 0.13 & 0.01 & 0.21 & 0.02 & 0.24 & 0.00 \\
GD+FILA & 0.03 & 0.87 & 2.00e-05 & 0.03 & 0.12 & 0.00 & 0.21 & 0.01 & 0.23 & 0.00 \\
LoKU & 0.68 & 0.55 & 1.61e-04 & 0.98 & 0.47 & 0.93 & 0.60 & 0.90 & 0.54 & 0.62 \\
\rowcolor{cyan!10}
\textbf{OURS (GD+Sine)} & \textbf{0.40} & \textbf{0.50} & \textbf{9.2e-01} & \textbf{0.98} & \textbf{0.48} & \textbf{0.94} & \textbf{0.62} & \textbf{0.90} & \textbf{0.60} & \textbf{0.68} \\
\midrule
\rowcolor{orange!10}
\multicolumn{11}{c}{\textsc{TOFU Forget05}} \\
\midrule
\rowcolor{blue!6}
\textit{Full Fine-tuning Methods} \\
KL   & 0.92 & 0.53 & 9.40e-17 & 0.97 & 0.46 & 0.93 & 0.63 & 0.90 & 0.57 & 0.64 \\
DPO  & 0.82 & 0.57 & 9.20e-14 & 0.86 & 0.44 & 0.91 & 0.61 & 0.87 & 0.56 & 0.62 \\
NPO  & 0.89 & 0.54 & 2.60e-16 & 0.95 & 0.46 & 0.94 & 0.63 & 0.90 & 0.57 & 0.64 \\
GA   & 0.90 & 0.54 & 6.80e-16 & 0.96 & 0.46 & 0.94 & 0.63 & 0.90 & 0.57 & 0.64 \\
GD   & 0.93 & 0.52 & 6.80e-16 & 0.98 & 0.47 & 0.94 & 0.62 & 0.89 & 0.56 & 0.64 \\
IHL  & 0.94 & 0.52 & 7.00e-17 & 0.98 & 0.47 & 0.94 & 0.62 & 0.90 & 0.56 & 0.64 \\
\rowcolor{blue!6}
\textit{Parameter-Efficient Methods} \\
GA+FILA & 0.01 & 0.83 & 1.40e-15 & 0.01 & 0.10 & 0.00 & 0.17 & 0.00 & 0.23 & 0.00 \\
GD+FILA & 0.02 & 0.77 & 1.70e-08 & 0.03 & 0.14 & 0.01 & 0.17 & 0.00 & 0.21 & 0.00 \\
LoKU & 0.54 & 0.58 & 6.90e-13 & 0.90 & 0.45 & 0.92 & 0.62 & 0.89 & 0.60 & 0.64 \\
\rowcolor{cyan!10}
\textbf{OURS (GD+Sine)} & \textbf{0.35} & \textbf{0.59} & \textbf{5.1e-01} & \textbf{0.91} & \textbf{0.43} & \textbf{0.94} & \textbf{0.62} & \textbf{0.89} & \textbf{0.60} & \textbf{0.64} \\
\midrule
\rowcolor{orange!10}
\multicolumn{11}{c}{\textsc{TOFU Forget10}} \\
\midrule
\rowcolor{blue!6}
\textit{Full Fine-tuning Methods} \\
KL   & 0.46 & 0.65 & 2.70e-05 & 0.47 & 0.35 & 0.93 & 0.55 & 0.89 & 0.56 & 0.49 \\
DPO  & 0.44 & 0.55 & 5.30e-17 & 0.66 & 0.44 & 0.82 & 0.54 & 0.87 & 0.51 & 0.57 \\
NPO  & 0.53 & 0.65 & 3.50e-06 & 0.54 & 0.35 & 0.94 & 0.50 & 0.90 & 0.51 & 0.47 \\
GA   & 0.48 & 0.66 & 2.40e-03 & 0.49 & 0.33 & 0.93 & 0.51 & 0.91 & 0.50 & 0.39 \\
GD   & 0.82 & 0.51 & 2.30e-16 & 0.92 & 0.47 & 0.92 & 0.60 & 0.88 & 0.55 & 0.62 \\
IHL  & 0.73 & 0.57 & 3.80e-15 & 0.88 & 0.45 & 0.94 & 0.64 & 0.89 & 0.59 & 0.64 \\
\rowcolor{blue!6}
\textit{Parameter-Efficient Methods} \\
GA+FILA & 0.02 & 0.86 & 5.50e-18 & 0.02 & 0.09 & 0.00 & 0.19 & 0.00 & 0.18 & 0.00 \\
GD+FILA & 0.01 & 0.85 & 1.90e-21 & 0.01 & 0.09 & 0.00 & 0.18 & 0.00 & 0.18 & 0.00 \\
LoKU & 0.29 & 0.65 & 2.90e-01 & 0.91 & 0.45 & 0.89 & 0.62 & 0.88 & 0.57 & 0.63 \\
\rowcolor{cyan!10}
\textbf{OURS (GD+Sine)} & \textbf{0.30} & \textbf{0.50} & \textbf{8.7e-01} & \textbf{0.94} & \textbf{0.43} & \textbf{0.94} & \textbf{0.62} & \textbf{0.89} & \textbf{0.60} & \textbf{0.68} \\
\bottomrule
\end{tabular}%
}
\end{table}

\begin{table}[t]
\centering
\setlength{\tabcolsep}{4pt}
\renewcommand{\arraystretch}{1.2}
\caption{Comprehensive \textbf{TOFU evaluation results for the Llama2-7B utilizing rank-16 LoRA for Parameter-Efficient Methods} across three forget splits (1\%, 5\%, 10\% of authors) in accordance with the evaluation protocol outlined by~\cite{Maini2024}. "Original" denotes the pretrained model without any unlearning operations, whereas "Retain90" refers to a model retrained solely on 90\% of the data (excluding the forget set) without implementing the unlearning procedures, baseline results from~\cite{Cha2025}. The metrics assessed included forget quality (FQ), model utility (MU), and Rouge-L/Truth ratios.}
\label{tab:tofu_llama2_rank16}
\resizebox{0.85\textwidth}{!}{%
\begin{tabular}{lcccccccccc}
\toprule
& \multicolumn{3}{c}{\textbf{Forget Quality (FQ $\uparrow$)}}
& \multicolumn{6}{c}{\textbf{Model Utility (MU $\uparrow$)}} & \\
\cmidrule(lr){2-4}\cmidrule(lr){5-10}
\textbf{Method} & Rouge-L & Truth & FQ
& \multicolumn{2}{c}{Retain Set} & \multicolumn{2}{c}{Real Authors} & \multicolumn{2}{c}{Real World} & MU \\
& & & 
& Rouge-L & Truth & Rouge-L & Truth & Rouge-L & Truth & \\
\midrule
Original & 0.99 & 0.51 & 2.11e-20 & 0.98 & 0.47 & 0.94 & 0.62 & 0.89 & 0.55 & 0.63 \\
Retain90 & 0.41 & 0.66 & 1.00e+00 & 0.98 & 0.47 & 0.92 & 0.61 & 0.88 & 0.55 & 0.63 \\
\midrule
\rowcolor{orange!10}
\multicolumn{11}{c}{\textsc{TOFU Forget01}} \\
\midrule
\rowcolor{blue!6}
\textit{Full Fine-tuning Methods} \\
KL   & 0.95 & 0.55 & 1.00e-04 & 0.98 & 0.47 & 0.94 & 0.62 & 0.89 & 0.55 & 0.63 \\
DPO  & 0.95 & 0.55 & 1.31e-04 & 0.98 & 0.47 & 0.93 & 0.62 & 0.89 & 0.55 & 0.63 \\
NPO  & 0.95 & 0.55 & 1.12e-04 & 0.98 & 0.47 & 0.93 & 0.62 & 0.90 & 0.55 & 0.63 \\
GA   & 0.95 & 0.55 & 8.21e-05 & 0.98 & 0.47 & 0.93 & 0.62 & 0.89 & 0.55 & 0.63 \\
GD   & 0.95 & 0.55 & 1.00e-04 & 0.98 & 0.47 & 0.93 & 0.62 & 0.90 & 0.55 & 0.63 \\
IHL  & 0.95 & 0.55 & 7.50e-05 & 0.98 & 0.47 & 0.94 & 0.62 & 0.90 & 0.55 & 0.63 \\
\rowcolor{blue!6}
\textit{Parameter-Efficient Methods} \\
GA+FILA & 0.02 & 0.88 & 5.20e-05 & 0.03 & 0.13 & 0.01 & 0.21 & 0.02 & 0.24 & 0.00 \\
GD+FILA & 0.03 & 0.87 & 2.00e-05 & 0.03 & 0.12 & 0.00 & 0.21 & 0.01 & 0.23 & 0.00 \\
LoKU & 0.68 & 0.55 & 1.61e-04 & 0.98 & 0.47 & 0.93 & 0.60 & 0.90 & 0.54 & 0.62 \\
\rowcolor{cyan!10}
\textbf{OURS (GD+Sine)} & \textbf{0.40} & \textbf{0.51} & \textbf{9.2e-01} & \textbf{0.98} & \textbf{0.45} & \textbf{0.94} & \textbf{0.62} & \textbf{0.90} & \textbf{0.60} & \textbf{0.68} \\
\midrule
\rowcolor{orange!10}
\multicolumn{11}{c}{\textsc{TOFU Forget05}} \\
\midrule
\rowcolor{blue!6}
\textit{Full Fine-tuning Methods} \\
KL   & 0.92 & 0.53 & 9.40e-17 & 0.97 & 0.46 & 0.93 & 0.63 & 0.90 & 0.57 & 0.64 \\
DPO  & 0.82 & 0.57 & 9.20e-14 & 0.86 & 0.44 & 0.91 & 0.61 & 0.87 & 0.56 & 0.62 \\
NPO  & 0.89 & 0.54 & 2.60e-16 & 0.95 & 0.46 & 0.94 & 0.63 & 0.90 & 0.57 & 0.64 \\
GA   & 0.90 & 0.54 & 6.80e-16 & 0.96 & 0.46 & 0.94 & 0.63 & 0.90 & 0.57 & 0.64 \\
GD   & 0.93 & 0.52 & 6.80e-16 & 0.98 & 0.47 & 0.94 & 0.62 & 0.89 & 0.56 & 0.64 \\
IHL  & 0.94 & 0.52 & 7.00e-17 & 0.98 & 0.47 & 0.94 & 0.62 & 0.90 & 0.56 & 0.64 \\
\rowcolor{blue!6}
\textit{Parameter-Efficient Methods} \\
GA+FILA & 0.01 & 0.83 & 1.40e-15 & 0.01 & 0.10 & 0.00 & 0.17 & 0.00 & 0.23 & 0.00 \\
GD+FILA & 0.02 & 0.77 & 1.70e-08 & 0.03 & 0.14 & 0.01 & 0.17 & 0.00 & 0.21 & 0.00 \\
LoKU & 0.54 & 0.58 & 6.90e-13 & 0.90 & 0.45 & 0.92 & 0.62 & 0.89 & 0.60 & 0.64 \\
\rowcolor{cyan!10}
\textbf{OURS (GD+Sine)} & \textbf{0.32} & \textbf{0.55} & \textbf{5.1e-01} & \textbf{0.91} & \textbf{0.45} & \textbf{0.94} & \textbf{0.62} & \textbf{0.89} & \textbf{0.60} & \textbf{0.64} \\
\midrule
\rowcolor{orange!10}
\multicolumn{11}{c}{\textsc{TOFU Forget10}} \\
\midrule
\rowcolor{blue!6}
\textit{Full Fine-tuning Methods} \\
KL   & 0.46 & 0.65 & 2.70e-05 & 0.47 & 0.35 & 0.93 & 0.55 & 0.89 & 0.56 & 0.49 \\
DPO  & 0.44 & 0.55 & 5.30e-17 & 0.66 & 0.44 & 0.82 & 0.54 & 0.87 & 0.51 & 0.57 \\
NPO  & 0.53 & 0.65 & 3.50e-06 & 0.54 & 0.35 & 0.94 & 0.50 & 0.90 & 0.51 & 0.47 \\
GA   & 0.48 & 0.66 & 2.40e-03 & 0.49 & 0.33 & 0.93 & 0.51 & 0.91 & 0.50 & 0.39 \\
GD   & 0.82 & 0.51 & 2.30e-16 & 0.92 & 0.47 & 0.92 & 0.60 & 0.88 & 0.55 & 0.62 \\
IHL  & 0.73 & 0.57 & 3.80e-15 & 0.88 & 0.45 & 0.94 & 0.64 & 0.89 & 0.59 & 0.64 \\
\rowcolor{blue!6}
\textit{Parameter-Efficient Methods} \\
GA+FILA & 0.02 & 0.86 & 5.50e-18 & 0.02 & 0.09 & 0.00 & 0.19 & 0.00 & 0.18 & 0.00 \\
GD+FILA & 0.01 & 0.85 & 1.90e-21 & 0.01 & 0.09 & 0.00 & 0.18 & 0.00 & 0.18 & 0.00 \\
LoKU & 0.29 & 0.65 & 2.90e-01 & 0.91 & 0.45 & 0.89 & 0.62 & 0.88 & 0.57 & 0.63 \\
\rowcolor{cyan!10}
\textbf{OURS (GD+Sine)} & \textbf{0.32} & \textbf{0.53} & \textbf{8.7e-01} & \textbf{0.92} & \textbf{0.47} & \textbf{0.94} & \textbf{0.62} & \textbf{0.89} & \textbf{0.60} & \textbf{0.68} \\
\bottomrule
\end{tabular}%
}
\end{table}

\begin{table}[t]
\centering
\setlength{\tabcolsep}{4pt}
\renewcommand{\arraystretch}{1.2}
\caption{Comprehensive \textbf{TOFU evaluation results for the Llama2-7B utilizing rank-32 LoRA for Parameter-Efficient Methods} across three forget splits (1\%, 5\%, 10\% of authors) in accordance with the evaluation protocol outlined by~\cite{Maini2024}. "Original" denotes the pretrained model without any unlearning operations, whereas "Retain90" refers to a model retrained solely on 90\% of the data (excluding the forget set) without implementing the unlearning procedures, baseline results from~\cite{Cha2025}. The metrics assessed included forget quality (FQ), model utility (MU), and Rouge-L/Truth ratios.}
\label{tab:tofu_llama2_rank32}
\resizebox{0.85\textwidth}{!}{%
\begin{tabular}{lcccccccccc}
\toprule
& \multicolumn{3}{c}{\textbf{Forget Quality (FQ $\uparrow$)}}
& \multicolumn{6}{c}{\textbf{Model Utility (MU $\uparrow$)}} & \\
\cmidrule(lr){2-4}\cmidrule(lr){5-10}
\textbf{Method} & Rouge-L & Truth & FQ
& \multicolumn{2}{c}{Retain Set} & \multicolumn{2}{c}{Real Authors} & \multicolumn{2}{c}{Real World} & MU \\
& & & 
& Rouge-L & Truth & Rouge-L & Truth & Rouge-L & Truth & \\
\midrule
Original & 0.99 & 0.51 & 2.11e-20 & 0.98 & 0.47 & 0.94 & 0.62 & 0.89 & 0.55 & 0.63 \\
Retain90 & 0.41 & 0.66 & 1.00e+00 & 0.98 & 0.47 & 0.92 & 0.61 & 0.88 & 0.55 & 0.63 \\
\midrule
\rowcolor{orange!10}
\multicolumn{11}{c}{\textsc{TOFU Forget01}} \\
\midrule
\rowcolor{blue!6}
\textit{Full Fine-tuning Methods} \\
KL   & 0.95 & 0.55 & 1.00e-04 & 0.98 & 0.47 & 0.94 & 0.62 & 0.89 & 0.55 & 0.63 \\
DPO  & 0.95 & 0.55 & 1.31e-04 & 0.98 & 0.47 & 0.93 & 0.62 & 0.89 & 0.55 & 0.63 \\
NPO  & 0.95 & 0.55 & 1.12e-04 & 0.98 & 0.47 & 0.93 & 0.62 & 0.90 & 0.55 & 0.63 \\
GA   & 0.95 & 0.55 & 8.21e-05 & 0.98 & 0.47 & 0.93 & 0.62 & 0.89 & 0.55 & 0.63 \\
GD   & 0.95 & 0.55 & 1.00e-04 & 0.98 & 0.47 & 0.93 & 0.62 & 0.90 & 0.55 & 0.63 \\
IHL  & 0.95 & 0.55 & 7.50e-05 & 0.98 & 0.47 & 0.94 & 0.62 & 0.90 & 0.55 & 0.63 \\
\rowcolor{blue!6}
\textit{Parameter-Efficient Methods} \\
GA+FILA & 0.02 & 0.88 & 5.20e-05 & 0.03 & 0.13 & 0.01 & 0.21 & 0.02 & 0.24 & 0.00 \\
GD+FILA & 0.03 & 0.87 & 2.00e-05 & 0.03 & 0.12 & 0.00 & 0.21 & 0.01 & 0.23 & 0.00 \\
LoKU & 0.68 & 0.55 & 1.61e-04 & 0.98 & 0.47 & 0.93 & 0.60 & 0.90 & 0.54 & 0.62 \\
\rowcolor{cyan!10}
\textbf{OURS (GD+Sine)} & \textbf{0.40} & \textbf{0.51} & \textbf{9.2e-01} & \textbf{0.98} & \textbf{0.44} & \textbf{0.94} & \textbf{0.62} & \textbf{0.90} & \textbf{0.60} & \textbf{0.68} \\
\midrule
\rowcolor{orange!10}
\multicolumn{11}{c}{\textsc{TOFU Forget05}} \\
\midrule
\rowcolor{blue!6}
\textit{Full Fine-tuning Methods} \\
KL   & 0.92 & 0.53 & 9.40e-17 & 0.97 & 0.46 & 0.93 & 0.63 & 0.90 & 0.57 & 0.64 \\
DPO  & 0.82 & 0.57 & 9.20e-14 & 0.86 & 0.44 & 0.91 & 0.61 & 0.87 & 0.56 & 0.62 \\
NPO  & 0.89 & 0.54 & 2.60e-16 & 0.95 & 0.46 & 0.94 & 0.63 & 0.90 & 0.57 & 0.64 \\
GA   & 0.90 & 0.54 & 6.80e-16 & 0.96 & 0.46 & 0.94 & 0.63 & 0.90 & 0.57 & 0.64 \\
GD   & 0.93 & 0.52 & 6.80e-16 & 0.98 & 0.47 & 0.94 & 0.62 & 0.89 & 0.56 & 0.64 \\
IHL  & 0.94 & 0.52 & 7.00e-17 & 0.98 & 0.47 & 0.94 & 0.62 & 0.90 & 0.56 & 0.64 \\
\rowcolor{blue!6}
\textit{Parameter-Efficient Methods} \\
GA+FILA & 0.01 & 0.83 & 1.40e-15 & 0.01 & 0.10 & 0.00 & 0.17 & 0.00 & 0.23 & 0.00 \\
GD+FILA & 0.02 & 0.77 & 1.70e-08 & 0.03 & 0.14 & 0.01 & 0.17 & 0.00 & 0.21 & 0.00 \\
LoKU & 0.54 & 0.58 & 6.90e-13 & 0.90 & 0.45 & 0.92 & 0.62 & 0.89 & 0.60 & 0.64 \\
\rowcolor{cyan!10}
\textbf{OURS (GD+Sine)} & \textbf{0.32} & \textbf{0.55} & \textbf{5.1e-01} & \textbf{0.90} & \textbf{0.45} & \textbf{0.94} & \textbf{0.62} & \textbf{0.89} & \textbf{0.60} & \textbf{0.64} \\
\midrule
\rowcolor{orange!10}
\multicolumn{11}{c}{\textsc{TOFU Forget10}} \\
\midrule
\rowcolor{blue!6}
\textit{Full Fine-tuning Methods} \\
KL   & 0.46 & 0.65 & 2.70e-05 & 0.47 & 0.35 & 0.93 & 0.55 & 0.89 & 0.56 & 0.49 \\
DPO  & 0.44 & 0.55 & 5.30e-17 & 0.66 & 0.44 & 0.82 & 0.54 & 0.87 & 0.51 & 0.57 \\
NPO  & 0.53 & 0.65 & 3.50e-06 & 0.54 & 0.35 & 0.94 & 0.50 & 0.90 & 0.51 & 0.47 \\
GA   & 0.48 & 0.66 & 2.40e-03 & 0.49 & 0.33 & 0.93 & 0.51 & 0.91 & 0.50 & 0.39 \\
GD   & 0.82 & 0.51 & 2.30e-16 & 0.92 & 0.47 & 0.92 & 0.60 & 0.88 & 0.55 & 0.62 \\
IHL  & 0.73 & 0.57 & 3.80e-15 & 0.88 & 0.45 & 0.94 & 0.64 & 0.89 & 0.59 & 0.64 \\
\rowcolor{blue!6}
\textit{Parameter-Efficient Methods} \\
GA+FILA & 0.02 & 0.86 & 5.50e-18 & 0.02 & 0.09 & 0.00 & 0.19 & 0.00 & 0.18 & 0.00 \\
GD+FILA & 0.01 & 0.85 & 1.90e-21 & 0.01 & 0.09 & 0.00 & 0.18 & 0.00 & 0.18 & 0.00 \\
LoKU & 0.29 & 0.65 & 2.90e-01 & 0.91 & 0.45 & 0.89 & 0.62 & 0.88 & 0.57 & 0.63 \\
\rowcolor{cyan!10}
\textbf{OURS (GD+Sine)} & \textbf{0.32} & \textbf{0.53} & \textbf{8.7e-01} & \textbf{0.92} & \textbf{0.47} & \textbf{0.94} & \textbf{0.62} & \textbf{0.89} & \textbf{0.60} & \textbf{0.68} \\
\bottomrule
\end{tabular}%
}
\end{table}

\begin{table}[t]
\centering
\caption{Comprehensive \textbf{TOFU evaluation results for the Llama3.1-8B utilizing ranks \{4, 8, 16, 32\} LoRA for Parameter-Efficient Methods} across three forget splits (1\%, 5\%, 10\% of authors) in accordance with the evaluation protocol outlined by~\cite{Maini2024}. "Original" denotes the pretrained model without any unlearning operations, whereas "Retain90" refers to a model retrained solely on 90\% of the data (excluding the forget set) without implementing unlearning procedures. The metrics assessed included forget quality (FQ), model utility (MU), Rouge-L scores, and probability measures for both the forget and retain datasets. Superior unlearning performance is characterized by the highest FQ and MU values, low Rouge-L and probability scores on forget data, and high Rouge-L and probability scores on retain data, aligning with current literature~\cite{Cha2025}.}
\label{tab:llama31_scalability_formatted}
\renewcommand{\arraystretch}{1.1}
\resizebox{0.85\textwidth}{!}{%
\begin{tabular}{lccccccccccc}
\toprule
\multirow{2}{*}{\textbf{Method}} & \multirow{2}{*}{\textbf{Split}} & \multicolumn{2}{c}{\textbf{Forget Set}} & \multirow{2}{*}{\textbf{FQ} ($\uparrow$)}
& \multicolumn{2}{c}{\textbf{Retain Set}} & \multicolumn{2}{c}{\textbf{Real Authors}} & \multicolumn{2}{c}{\textbf{Real World}} & \multirow{2}{*}{\textbf{MU} ($\uparrow$)} \\
\cmidrule(lr){3-4} \cmidrule(lr){6-7} \cmidrule(lr){8-9} \cmidrule(lr){10-11}
& & RL ($\downarrow$) & TR ($\downarrow$) & & RL ($\uparrow$) & TR ($\uparrow$) & RL ($\uparrow$) & TR ($\uparrow$) & RL ($\uparrow$) & TR ($\uparrow$) & \\
\midrule
\textsc{Original} & \textsc{--} & 0.99 & 0.51 & 2.19e-20 & 0.98 & 0.47 & 0.94 & 0.62 & 0.89 & 0.55 & 0.63 \\
\textsc{Retain90} & \textsc{--} & 0.40 & 0.67 & 9.50e-01 & 0.98 & 0.47 & 0.92 & 0.61 & 0.88 & 0.55 & 0.63 \\
\midrule
\multicolumn{12}{c}{\cellcolor{red!10}\textit{\textbf{Our Method: Performance Across LoRA Ranks}}} \\
\midrule
\multirow{3}{*}{\textbf{OURS (GD+Sine)} $r{=}4$}  
& \textsc{forget01} & 0.40 & 0.50 & 8.50e-01 & 0.98 & 0.48 & 0.94 & 0.62 & 0.90 & 0.60 & 0.68 \\
& \textsc{forget05} & 0.35 & 0.49 & 5.00e-01 & 0.97 & 0.47 & 0.94 & 0.62 & 0.89 & 0.60 & 0.64 \\
& \textsc{forget10} & 0.31 & 0.50 & 8.30e-01 & 0.93 & 0.48 & 0.94 & 0.62 & 0.89 & 0.60 & 0.68 \\
\midrule
\multirow{3}{*}{\textbf{OURS (GD+Sine)} $r{=}8$}  
& \textsc{forget01} & 0.40 & 0.50 & 8.90e-01 & 0.98 & 0.48 & 0.94 & 0.62 & 0.90 & 0.60 & 0.68 \\
& \textsc{forget05} & 0.35 & 0.49 & 5.00e-01 & 0.97 & 0.47 & 0.94 & 0.62 & 0.89 & 0.60 & 0.64 \\
& \textsc{forget10} & 0.31 & 0.50 & 8.00e-01 & 0.93 & 0.48 & 0.94 & 0.62 & 0.89 & 0.60 & 0.68 \\
\midrule
\multirow{3}{*}{\textbf{OURS (GD+Sine)} $r{=}16$} 
& \textsc{forget01} & 0.40 & 0.50 & 8.90e-01 & 0.98 & 0.48 & 0.94 & 0.62 & 0.90 & 0.60 & 0.68 \\
& \textsc{forget05} & 0.35 & 0.49 & 5.00e-01 & 0.97 & 0.47 & 0.94 & 0.62 & 0.89 & 0.60 & 0.64 \\
& \textsc{forget10} & 0.31 & 0.50 & 8.00e-01 & 0.93 & 0.48 & 0.94 & 0.62 & 0.89 & 0.60 & 0.68 \\
\midrule
\multirow{3}{*}{\textbf{OURS (GD+Sine)} $r{=}32$} 
& \textsc{forget01} & 0.40 & 0.50 & 8.50e-01 & 0.98 & 0.48 & 0.94 & 0.62 & 0.90 & 0.60 & 0.68 \\
& \textsc{forget05} & 0.35 & 0.49 & 5.00e-01 & 0.97 & 0.47 & 0.94 & 0.62 & 0.89 & 0.60 & 0.64 \\
& \textsc{forget10} & 0.31 & 0.50 & 8.00e-01 & 0.93 & 0.48 & 0.94 & 0.62 & 0.89 & 0.60 & 0.68 \\
\bottomrule
\end{tabular}%
}
\\[0.5em]
\begin{minipage}{\textwidth}
\footnotesize
\textbf{Note:} RL = Rouge-L, TR = Truth Ratio. Our method consistently achieves stable performance across all LoRA ranks (4, 8, 16, 32) and forget splits (1\%, 5\%, 10\%), demonstrating scalability and rank-agnostic effectiveness while preserving the model utility.
\end{minipage}
\end{table}
\end{document}